\documentclass{article}

    \PassOptionsToPackage{numbers, compress}{natbib}

    \usepackage[final]{neurips_2022}

\usepackage[utf8]{inputenc} 
\usepackage[T1]{fontenc}    
\usepackage{hyperref}       
\usepackage{url}            
\usepackage{booktabs}       
\usepackage{amsfonts}       
\usepackage{nicefrac}       
\usepackage{microtype}      
\usepackage{xcolor}         

\usepackage{amsfonts}
\usepackage{amsmath, amssymb}
\usepackage{amsthm}
\usepackage{relsize}
\usepackage{enumitem}
\usepackage{upgreek}
\usepackage{caption}
\usepackage{subcaption}
\usepackage{microtype}
\usepackage{graphicx}
\usepackage{booktabs} 
\usepackage{algorithm}
\usepackage{algorithmic}
\usepackage{wrapfig}
\usepackage{MnSymbol}

\newtheorem{theorem}[]{Theorem}
\newtheorem{lemma}[]{Lemma}
\newtheorem{proposition}[]{Proposition}
\newtheorem{corollary}[]{Corollary}
\newtheorem{definition}[]{Definition}
\newtheorem{assumption}[]{Assumption}

\newtheorem{remark}[]{Remark}
\newtheorem{example}[]{Example}

\title{SHAQ: Incorporating Shapley Value Theory into Multi-Agent Q-Learning}

\author{%
  Jianhong Wang \\
  Imperial College London, UK \\
  \texttt{jianhong.wang16@imperial.ac.uk}
  \And
  Yuan Zhang \\
  University of Freiburg, Germany \\
  \texttt{yzhang@cs.uni-freiburg.de}
  \And
  Yunjie Gu\thanks{Correspondence to Yunjie Gu who is also an honorary lecturer at Imperial College London.} \\
  University of Bath, UK \\
  \texttt{yg934@bath.ac.uk}
  \And
  Tae-Kyun Kim \\
  KAIST, South Korea \\
  \texttt{kimtaekyun@kaist.ac.kr}
}

\begin{document}

\maketitle

\begin{abstract}
    Value factorisation is a useful technique for multi-agent reinforcement learning (MARL) in global reward game, however, its underlying mechanism is not yet fully understood. This paper studies a theoretical framework for value factorisation with interpretability via Shapley value theory. We generalise Shapley value to Markov convex game called \textit{Markov Shapley value} (MSV) and apply it as a value factorisation method in global reward game, which is obtained by the equivalence between the two games. Based on the properties of MSV, we derive \textit{Shapley-Bellman optimality equation} (SBOE) to evaluate the optimal MSV, which corresponds to an optimal joint deterministic policy. Furthermore, we propose \textit{Shapley-Bellman operator} (SBO) that is proved to solve SBOE. With a stochastic approximation and some transformations, a new MARL algorithm called \textit{Shapley Q-learning} (SHAQ) is established, the implementation of which is guided by the theoretical results of SBO and MSV. We also discuss the relationship between SHAQ and relevant value factorisation methods. In the experiments, SHAQ exhibits not only superior performances on all tasks but also the interpretability that agrees with the theoretical analysis. The implementation of this paper is placed on \url{https://github.com/hsvgbkhgbv/shapley-q-learning}.
\end{abstract}

\section{Introduction}
\label{sec:introduction}
    Cooperative games are a critical research area in multi-agent reinforcement learning (MARL). Many real-life tasks can be modeled as cooperative games, e.g. the coordination of autonomous vehicles \citep{keviczky2007decentralized}, autonomous distributed logistics \citep{schuldt2012multiagent} and distributed voltage control in power networks \citep{wang2021multi}. In this paper, we consider global reward game (a.k.a. team reward game), an important subclass of cooperative games, wherein agents aim to jointly maximize cumulative global rewards over time. There are two categories of methods to solve this problem: (i) each agent identically maximizes cumulative global rewards, i.e. learning with a shared value function \citep{sukhbaatar2016learning,omidshafiei2018learning,kim2018learning}; and (ii) each agent individually maximizes distributed values, i.e. learning with (implicit) credit assignments (e.g. marginal contribution and value factorisation) \citep{foerster2018counterfactual,SunehagLGCZJLSL18,RashidSWFFW18,SonKKHY19,zhou2020learning}.
    
    By the view of non-cooperative game theory, global reward game are equivalent to Markov game \citep{shapley1953stochastic} with global reward (a.k.a. team reward). Its aim is to learn a stationary joint policy to reach a Markov equilibrium so that no agent tends to unilaterally change its policy to maximize cumulative global rewards. Standing by this view, learning with value factorisation cannot be directly explained \citep{Wang_2020}. In this paper, to clearly interpret the value factorisation, we take the perspective of cooperative game theory \citep{chalkiadakis2011computational}, wherein agents are partitioned into coalitions and a payoff distribution scheme is found to distribute optimal values to coalitions. The corresponding solution is called Markov core, whereby no agent has an incentive to deviate. When all agents are partitioned into one coalition (called grand coalition), the payoff distribution scheme naturally plays the role of value factorisation.
    
    \citet{Wang_2020} extended convex game (i.e. a game model in cooperative game theory) \citep{chalkiadakis2011computational} to dynamic scenarios, which we name as Markov convex game in this paper. We construct the analytic form of Shapley value for Markov convex game, and prove that it reaches the Markov core under the grand coalition, named as Markov Shapley value. The optimal Markov Shapley value implies not only the optimal global value but also that no agent has incentives to deviate from the grand coalition. Additionally, Markov Shapley value enjoys the following properties: (i) identifiability of dummy agents; (ii) efficiency; (iii) reflecting the contribution; and (iv) symmetry. These properties aid the interpretation and validity of value factorisation in the global reward game, and such transparency and reliability are critical to industrial applications \citep{wang2021multi}.
    
    Based on the efficiency property, we derive Shapley-Bellman optimality equation that is an extension of Bellman optimality equation \citep{bellman1952theory,sutton2018reinforcement}. Moreover, we propose Shapley-Bellman operator and prove its convergence to the Shapley-Bellman optimality equation and its optimal joint deterministic policy. With a stochastic approximation of Shapley-Bellman operator and some transformations, we derive an algorithm called Shapley Q-learning (SHAQ). SHAQ learns to approximate the optimal Markov Shapley Q-value (an equivalent form of the optimal Markov Shapley value). Moreover, we enable SHAQ decentralised in order to fit the decentralised execution framework and this decentralisation still remains the convergence condition of Shapley-Bellman operator. 
    
    The proposed method, SHAQ, is evaluated on two global reward games such as Predator-Prey \citep{bohmer2020deep} and multi-agent StarCraft benchmark tasks \citep{samvelyan2019starcraft}. In the experiments, SHAQ shows not only generally good performances on solving all tasks but also the interpretability that is deficient in the state-of-the-art baselines.

\section{Markov Convex Game}
\label{sec:markov_convex_game}
    We now formally define Markov convex game (MCG) that can be  described as a tuple $\left\langle \mathcal{N}, \mathcal{S}, \mathcal{A}, T, \Lambda, \pi, R_{t}, \gamma \right\rangle$. $\mathcal{N}$ is the set of all agents. $\mathcal{S}$ is the set of states and $\mathcal{A} = \mathlarger{\mathlarger{\times}}_{i \in \mathcal{N}} \mathcal{A}_{i}$ is the joint action set of all agents wherein $\mathcal{A}_{i}$ is each agent's action set. $T( \mathbf{s}, \mathbf{a}, \mathbf{s}' ) = Pr(\mathbf{s}' | \mathbf{s}, \mathbf{a})$ is defined as the transition probability between the successive states. $\mathcal{CS} = \left\{ \mathcal{C}_{1}, ..., \mathcal{C}_{n} \right\}$ is a \textit{coalition structure}, where $\mathcal{C}_{i} \ \mathlarger{\mathlarger{\mathlarger{\subseteq}}} \ \mathcal{N}$ called a \textit{coalition} is a subset of all agents. $\Lambda$ is a collection of coalition structures. $\emptyset$ and $\mathcal{N}$ are two special cases of coalitions i.e. the \textit{empty coalition} and the \textit{grand coalition} respectively. Conventionally, it is assumed that $\mathcal{C}_{m} \bigcap \mathcal{C}_{k} = \emptyset, \forall \mathcal{C}_{m}, \mathcal{C}_{k} \ \mathlarger{\mathlarger{\subseteq}} \ \mathcal{N}$. $\pi = \mathlarger{\mathlarger{\times}}_{\scriptscriptstyle i \in \mathcal{N}} \pi_{i}$ is the joint policy of all agents. For any coalition $\mathcal{C}$, it is equipped with a \textit{coalition policy} $\pi_{\scriptscriptstyle\mathcal{C}}(\mathbf{a}_{\scriptscriptstyle\mathcal{C}} | \mathbf{s}) = \mathlarger{\mathlarger{\times}}_{\scriptscriptstyle i \in \mathcal{C}} \pi_{i}(a_{i}|\mathbf{s})$ defined over the \textit{coalition action set} $\mathcal{A}_{\scriptscriptstyle\mathcal{C}} = \mathlarger{\mathlarger{\times}}_{\scriptscriptstyle i \in \mathcal{C}} \mathcal{A}_{i}$. Therefore, $\pi$ can be seen as the \textit{grand coalition policy}. $R_{t}: \mathcal{S} \times \mathcal{A}_{\scriptscriptstyle\mathcal{C}} \rightarrow [0, \infty)$ (i.e., a characteristic function) is the \textit{coalition reward} at time step $t$. Accordingly, $R_{t}(\mathbf{s}, \mathbf{a})$ is the \textit{grand coalition reward} (i.e., equivalent to the global reward) at time step $t$ that is written as $\mathit{R}(\mathbf{s}, \mathbf{a})$ or $\mathit{R}$ for conciseness in the rest of paper. $\gamma \in (0, 1)$ is the discounted factor. The infinite long-term discounted cumulative coalition rewards is defined as $V^{\pi_{\scriptscriptstyle\mathcal{C}}}(\mathbf{s}) = \mathbb{E}_{\pi_{\scriptscriptstyle\mathcal{C}}} \big[ \sum_{t=1}^{\infty} \gamma^{t-1} R_{t}(\mathbf{s}, \mathbf{a}_{\scriptscriptstyle\mathcal{C}}) \ | \ \mathbf{S}_{t}=\mathbf{s} \big] \in [0, \infty)$, called a \textit{coalition value}. Moreover, the empty coalition value $V^{\pi_{\emptyset}}(\mathbf{s}) = 0$ and $V^{\pi}(\mathbf{s})$ denotes the grand coalition value (i.e. also called the global value since the equivalence proof from \cite{Wang_2020}). The solution of MCG is to find a tuple $\left\langle \mathcal{CS}, \left( \max_{\pi_{i}} x_{i}(\mathbf{s}) \right)_{i \in \mathcal{N}} \right\rangle$, where $\left( \max_{\pi_{i}} x_{i}(\mathbf{s}) \right)_{i \in \mathcal{N}}$ indicates the \textit{payoff distributions} (i.e. credit assignments) under the optimal joint policy given a coalition structure. Under the assumption $\mathcal{C}_{m} \bigcap \mathcal{C}_{k} = \emptyset, \forall \mathcal{C}_{m}, \mathcal{C}_{k} \ \mathlarger{\mathlarger{\subseteq}} \ \mathcal{N}$, the condition for MCG is as follows:
        \begin{equation}
            \begin{split}
                \max_{\pi_{\mathcal{C}_{\cup}}} V^{\pi_{\mathcal{C}_{\cup}}}(\mathbf{s}) \geq
                \max_{\pi_{\mathcal{C}_{m}}} V^{\pi_{\mathcal{C}_{m}}}(\mathbf{s})
                + \max_{\pi_{\mathcal{C}_{k}}} V^{\pi_{\mathcal{C}_{k}}}(\mathbf{s}), \quad
                \forall \mathcal{C}_{m}, \mathcal{C}_{k} \ \mathlarger{\mathlarger{\subseteq}} \ \mathcal{N}, \mathcal{C}_{\cup}=\mathcal{C}_{m} \ \mathlarger{\mathlarger{\cup}} \ \mathcal{C}_{k}.
            \end{split}
        \label{eq:mcg_assumption}
        \end{equation}
        
        In MCG with the grand coalition i.e., $\mathcal{CS} = \{ \mathcal{N} \}$, \textit{Markov core}, a solution concept describing stability, is defined as a set of payoff distribution schemes by which no agent has incentives to deviate from the grand coalition to gain more profits. Mathematically, Markov core can be expressed as:
        \begin{equation}
            \texttt{MarkovCore} = \Big\{ \big( \max_{\pi_{i}} x_{i}(\mathbf{s}) \big)_{i \in \mathcal{N}} \Big\mid \max_{\pi_{\mathcal{C}}} x(\mathbf{s}|\mathcal{C}) \geq
            \max_{\pi_{\mathcal{C}}} V^{\pi_{\mathcal{C}}}(\mathbf{s}),
            \forall \mathcal{C} \ \mathlarger{\mathlarger{\subseteq}} \ \mathcal{N}, \mathbf{s} \in \mathcal{S} \ \Big\},
        \label{eq:epsilon_core}
        \end{equation}
        where $\max_{\pi_{\mathcal{C}}} x(\mathbf{s}|\mathcal{C}) = \sum_{i \in \mathcal{C}} \max_{\pi_{i}} x_{i}(\mathbf{s})$. It aims to find a payoff distribution scheme $\left( x_{i}(\mathbf{s}) \right)_{i \in \mathcal{N}}$ that can finally converge to Markov core under the optimal joint policy.
        
        To assist the application on Q-learning, we similarly define \textit{coalition Q-value} as $Q^{\pi_{\mathcal{C}}}(\mathbf{s}, \mathbf{a}_{\scriptscriptstyle\mathcal{C}}) \in [0, +\infty)$ for all coalitions $\mathcal{C} \ \mathlarger{\mathlarger{\mathlarger{\subset}}} \ \mathcal{N}$. Following the above convention, the grand coalition Q-value (or the global Q-value) can be written as $Q^{\pi}(\mathbf{s}, \mathbf{a})$. Moreover, the optimal coalition Q-value of $\mathcal{C}$ w.r.t. the optimal joint policy of $\mathcal{D} \ \mathlarger{\mathlarger{\subseteq}} \ \mathcal{C}$ (i.e., $\pi_{\scriptscriptstyle\mathcal{D}}^{*}$) and the suboptimal joint policy of $\mathcal{C} \backslash \mathcal{D}$ (i.e., $\pi_{\scriptscriptstyle\mathcal{C} \backslash \mathcal{D}}$) is defined as $Q^{\pi_{\mathcal{D}}^{*}}(\mathbf{s}, \mathbf{a}_{\scriptscriptstyle\mathcal{C}})$. Therefore, the optimal coalition Q-value of $\mathcal{C}$ w.r.t. the optimal joint policy of $\mathcal{C}$ is defined as $Q^{\pi_{\mathcal{C}}^{*}}(\mathbf{s}, \mathbf{a}_{\scriptscriptstyle\mathcal{C}})$. Accordingly, the optimal global coalition Q-value w.r.t. the optimal joint policy of the grand coalition is denoted as $Q^{\pi^{*}}(\mathbf{s}, \mathbf{a})$.
        
\section{Markov Shapley Value}
\label{sec:generalised_shapley_value_for_mcg}
    By the view of cooperative game theory, the grand coalition is progressively formed by a permutation of agents. Accordingly, marginal contribution is an implementation of the credit reflecting an agent's contribution. The formal definition is shown in Definition \ref{def:marginal_contribution}.
    \begin{definition}
    \label{def:marginal_contribution}
        In Markov convex game, with a permutation of agents $\langle j_{1}, j_{2}, ..., j_{\scriptscriptstyle|\mathcal{N}|} \rangle, \forall j_{n} \in \mathcal{N}$ forming the grand coalition $\mathcal{N}$, where $n \in \{1, ..., |\mathcal{N}|\}, j_{a} \neq j_{b} \text{ if } a \neq b$, the marginal contribution of an agent $\mathit{i}$ is defined as the following equation such that
        \begin{equation}
            \Phi_{i}(\mathbf{s} | \mathcal{C}_{i}) = \max_{\pi_{\mathcal{C}_{i}}} V^{\pi_{\mathcal{C}_{i} \cup \{i\}}}(\mathbf{s}) - \max_{\pi_{\mathcal{C}_{i}}} V^{\pi_{\mathcal{C}_{i}}}(\mathbf{s}),
        \label{eq:marginal_contribution_v}
        \end{equation}
        where $\mathcal{C}_{i} = \{ j_{1}, ..., j_{n-1} \} \text{ for } j_{n}=i$ is an arbitrary intermediate coalition where agent $\mathit{i}$ would join during the process of grand coalition formation.
    \end{definition}
    
    \begin{proposition}
        \label{prop:optimal_action_coalition_marginal_contribution}
        Agent $i$'s action marginal contribution can be derived as follows:
        \begin{equation}
            \begin{split}
                \Upphi_{i}(\mathbf{s}, a_{i} | \mathcal{C}_{i}) 
                = \max_{\mathbf{a}_{\scriptscriptstyle \mathcal{C}_{i}}} Q^{\pi_{\mathcal{C}_{i}}^{*}}(\mathbf{s}, \mathbf{a}_{\scriptscriptstyle\mathcal{C}_{i} \cup \{i\}})
                - \max_{\mathbf{a}_{\mathcal{C}_{i}}} Q^{\pi_{\mathcal{C}_{i}}^*}(\mathbf{s}, \mathbf{a}_{\scriptscriptstyle\mathcal{C}_{i}}).
            \end{split}
        \label{eq:marginal_contribution_q}
        \end{equation}
    \end{proposition}
    
    As Proposition \ref{prop:optimal_action_coalition_marginal_contribution} shows, an agent's action marginal contribution (analogous to Q-value) can be derived according to Eq.\ref{eq:marginal_contribution_q}. It is usually more useful for solving MARL problems. 
    
    It is apparent that marginal contribution only considers one permutation to form the grand coalition. By the viewpoint from \citet{shapley1953value}, the fairness is achieved through considering how much the agent $i$ increases the optimal values (i.e. marginal contributions) of the coalitions in all possible permutations when it joins in, i.e., $\max_{\pi_{\mathcal{C}_{i}}} V^{\pi_{\mathcal{C}_{i} \cup \{i\}}}(\mathbf{s}) - \max_{\pi_{\mathcal{C}_{i}}} V^{\pi_{\mathcal{C}_{i}}}(\mathbf{s}), \forall \mathcal{C}_{i} \ \mathlarger{\mathlarger{\mathlarger{\subseteq}}} \ \mathcal{N} \backslash \{i\}$. Therefore, we construct Shapley value under Markov dynamics based on the marginal contributions shown in Definition \ref{def:shapley_value}, named as \textit{Markov Shapley value} (MSV).
    \begin{definition}
        Markov Shapley value is represented as 
        \begin{equation}
            V^{\phi}_{i}(\mathbf{s}) = \sum_{\mathcal{C}_{i} \ \mathlarger{\mathlarger{\subseteq}} \ \mathcal{N} \backslash \{i\} } \frac{|\mathcal{C}_{i}|!(|\mathcal{N}|-|\mathcal{C}_{i}|-1)!}{|\mathcal{N}|!} \cdot \Phi_{i}(\mathbf{s} | \mathcal{C}_{i}).
        \label{eq:shapley_value}
        \end{equation}
        With the deterministic policy, Markov Shapley value can be equivalently represented as
        \begin{equation}
            Q^{\phi}_{i}(\mathbf{s}, a_{i}) = \sum_{\mathcal{C}_{i} \ \mathlarger{\mathlarger{\subseteq}} \ \mathcal{N} \backslash \{i\} } \frac{|\mathcal{C}_{i}|!(|\mathcal{N}|-|\mathcal{C}_{i}|-1)!}{|\mathcal{N}|!} \cdot \Upphi_{i}(\mathbf{s}, a_{i} | \mathcal{C}_{i}).
        \label{eq:shapley_q_value}
        \end{equation}
        where $\Phi_{i}(\mathbf{s} | \mathcal{C}_{i})$ is defined in Eq.\ref{eq:marginal_contribution_v} and $\Upphi_{i}(\mathbf{s}, a_{i} | \mathcal{C}_{i})$ is defined in Eq.\ref{eq:marginal_contribution_q}.
    \label{def:shapley_value}
    \end{definition}
    For convenience, we name Eq.\ref{eq:shapley_q_value} as \textit{Markov Shapley Q-value} (MSQ). Briefly, MSV calculates the weighted average of marginal contributions. Since a coalition may repeatedly appear among all permutations (i.e. $|\mathcal{N}|!$ permutations), the ratio between the occurrence frequency $|\mathcal{C}_{i}|!(|\mathcal{N}|-|\mathcal{C}_{i}|-1)!$ and the total frequency $|\mathcal{N}|!$ is used as a weight to describe the importance of the corresponding marginal contribution. Besides, the sum of all weights is equal to 1, so each weight can be interpreted as a probability distribution. Consequently, MSV can be seen as the expectation of marginal contributions, denoted as $\mathbb{E}_{\mathcal{C}_{i} \sim Pr(\mathcal{C}_{i} | \mathcal{N} \backslash \{i\})}\left[ \Phi_{i}(\mathbf{s} | \mathcal{C}_{i}) \right]$. Note that $Pr(\mathcal{C}_{i} | \mathcal{N} \backslash \{i\})$ is a bell-shaped probability distribution. By the above relationship, Remark \ref{rmk:coalition_generation} is directly obtained.
    \begin{remark}
    \label{rmk:coalition_generation}
        Uniformly sampling different permutations is equivalent to directly sampling from $Pr(\mathcal{C}_{i} | \mathcal{N} \backslash \{i\})$, since the coalition generation is from the permutation to form the grand coalition.
    \end{remark}
    
    \begin{proposition}
    \label{prop:shapley_value_properties}
        Markov Shapley value possesses properties as follows: (i) identifiability of dummy agents: $V_{i}^{\phi}(\mathbf{s}) = 0$; (ii) efficiency: $\max_{\pi} V^{\pi}(\mathbf{s}) = \sum_{i \in \mathcal{N}} \max_{\pi_{i}} V_{i}^{\phi}(\mathbf{s})$; (iii) reflecting the contribution; and (iv) symmetry.
    \end{proposition}
    
    Proposition \ref{prop:shapley_value_properties} shows four properties of MSV. The most important property is Property (ii) that aids the formulation of Shapley-Bellman optimality equation. Property (iii) shows that MSV is a fundamental index to quantitatively describe each agent's contribution. Property (i) and (iii) play important roles in interpretation for value factorisation (or credit assignment). Property (iv) indicates that if two agents are symmetric, then their optimal MSVs should be equal, \textit{but the reverse does not necessarily hold}. All these properties that define the fairness are inherited from the original Shapley value \citep{shapley1953value}.
    
\section{Shapley Q-Learning}
\label{sec:shapley_q_learning}

    \subsection{Definition and Formulation}
    \label{subsec:definition_and_formulation}
        \textbf{Shapley-Bellman Optimality Equation.} Based on the Bellman optimality equation \citep{bellman1952theory} and the following conditions (the interpretability of which are left to Section \ref{subsec:validity_of_shapley_q-learning}): 
        \begin{itemize}
            \item[\textbf{C.1.}] Efficiency of MSV (i.e. the result from Proposition \ref{prop:shapley_value_properties});
            \item[\textbf{C.2.}] $Q^{\phi^{*}}_{i}(\mathbf{s}, a_{i}) = w_{i}(\mathbf{s}, a_{i}) \ Q^{\pi^{*}}(\mathbf{s}, \mathbf{a}) - b_{i}(\mathbf{s})$, where $w_{i}(\mathbf{s}, a_{i}) > 0$ and $b_{i}(\mathbf{s}) \geq 0$ are bounded and $\sum_{i \in \mathcal{N}} w_{i}(\mathbf{s}, a_{i})^{-1} b_{i}(\mathbf{s}) = 0$,
        \end{itemize}
        we derive \textit{Shapley-Bellman optimality equation} (SBOE) for evaluating the optimal MSQ (an equivalent form to optimal MSV) such that
        \begin{equation}
        \label{eq:shapley_q_optimality_equation}
            \mathbf{Q}^{\phi^{*}}(\mathbf{s}, \mathbf{a}) = \mathbf{w}(\mathbf{s}, \mathbf{a}) \sum_{\mathbf{s}' \in \mathcal{S}} Pr(\mathbf{s}' | \mathbf{s}, \mathbf{a}) \Big[
            R
            + \ 
            \gamma \sum_{i \in \mathcal{N}} \max_{a_{i}} Q_{i}^{\phi^{*}}(\mathbf{s}', a_{i}) \Big] - \mathbf{b}(\mathbf{s}),
        \end{equation}
        where $\mathbf{w}(\mathbf{s}, \mathbf{a}) = [w_{i}(\mathbf{s}, a_{i})]^{\top} \in \mathbb{R}^{\scriptscriptstyle|\mathcal{N}|}_{+}$; $\mathbf{b}(\mathbf{s}) = [b_{i}(\mathbf{s})]^{\top} \in \mathbb{R}^{\scriptscriptstyle|\mathcal{N}|}_{\geq 0}$; $\mathbf{Q}^{\phi^{*}}(\mathbf{s}, \mathbf{a}) = [Q^{\phi^{*}}_{i}(\mathbf{s}, a_{i})]^{\top} \in \mathbb{R}^{\scriptscriptstyle|\mathcal{N}|}_{\geq 0}$ and $Q^{\phi^{*}}_{i}(\mathbf{s}, a_{i})$ denotes the optimal MSQ. If Eq.\ref{eq:shapley_q_optimality_equation} holds, the optimal MSQ is achieved. Moreover, it reveals an implication that for any $\mathbf{s} \in \mathcal{S}$ and $\mathit{a}_{i}^{*} = \arg\max_{a_{i}} Q^{\phi^{*}}_{i}(\mathbf{s}, a_{i})$, we have a solution $w_{i}(\mathbf{s}, a_{i}^{*}) = 1 / |\mathcal{N}|$ (see Appendix \ref{subsubsec:deriviation_shapley-q_optimality_appendix}). Literally, the assigned credits would be equal and each agent would receive $Q^{\pi^{*}}(\mathbf{s}, \mathbf{a}) / |\mathcal{N}|$ if performing the optimal actions. It is apparent that the efficiency still holds under this situation, which can be interpreted as an extremely fair credit assignment such that the credit to each agent should not be discriminated if all of them perform optimally, regardless of their roles. The equal credit assignment was also revealed by \citet{wang2020towards} recently from another perspective of analysis. Nevertheless, $w_{i}(\mathbf{s}, a_{i})$ for $\mathit{a}_{i} \neq \arg\max_{a_{i}} Q^{\phi^{*}}_{i}(\mathbf{s}, a_{i})$ needs to be learned.
        
        \textbf{Shapley-Bellman Operator.} To find an optimal solution described by Eq.\ref{eq:shapley_q_optimality_equation}, we now propose an operator called \textit{Shapley-Bellman operator} (SBO), i.e., $\mathlarger{\Upsilon}: \mathlarger{\mathlarger{\times}}_{i \in \mathcal{N}} Q_{i}^{\phi}(\mathbf{s}, a_{i}) \mapsto \mathlarger{\mathlarger{\times}}_{i \in \mathcal{N}} Q_{i}^{\phi}(\mathbf{s}, a_{i})$, which is defined as follows:
        \begin{equation}
            \mathlarger{\Upsilon} \left( \mathlarger{\mathlarger{\times}}_{i \in \mathcal{N}} Q_{i}^{\phi}(\mathbf{s}, a_{i}) \right) = \mathbf{w}(\mathbf{s}, \mathbf{a}) \sum_{\mathbf{s}' \in \mathcal{S}} Pr(\mathbf{s}' | \mathbf{s}, \mathbf{a}) \Big[
            R + \ \gamma \sum_{i \in \mathcal{N}} \max_{a_{i}} Q_{i}^{\phi}(\mathbf{s}', a_{i}) \Big] - \mathbf{b}(\mathbf{s}),
        \label{eq:shapley_q_operator}
        \end{equation}
        where $w_{i}(\mathbf{s}, a_{i}) = 1 / |\mathcal{N}|$ when $\mathit{a}_{i} = \arg\max_{a_{i}} Q^{\phi}_{i}(\mathbf{s}, a_{i})$. We prove that the optimal joint deterministic policy can be achieved by recursively running SBO in Theorem \ref{thm:shapley_q_optimal}.
        
        \begin{theorem}
            \label{thm:shapley_q_optimal}
            Shapley-Bellman operator is able to converge to the optimal Markov Shapley Q-value and the corresponding optimal joint deterministic policy when $\max_{\mathbf{s}} \big\{ \sum_{i \in \mathcal{N}} \max_{a_{i}} w_{i}(\mathbf{s}, a_{i}) \big\} < \frac{1}{\gamma}$.
        \end{theorem}
        
        \textbf{Shapley Q-Learning.} For easy implementation, we conduct transformation for the stochastic approximation of SBO and derive \textit{Shapley Q-learning} (SHAQ) whose TD error is shown as follows:
        \begin{equation}
            \begin{split}
                \Delta(\mathbf{s}, \mathbf{a}, \mathbf{s}') = R \ + \ \gamma \sum_{i \in \mathcal{N}} \max_{a_{i}} Q_{i}^{\phi}(\mathbf{s}', a_{i})
                - \sum_{i \in \mathcal{N}} \delta_{i}(\mathbf{s}, a_{i}) \ Q^{\phi}_{i}(\mathbf{s}, a_{i}),
            \end{split}
        \label{eq:td_error_shapley_q_learning}
        \end{equation}
        where 
        \begin{equation}
            \delta_{i}(\mathbf{s}, a_{i}) = \begin{cases} 
                                                  1 & a_{i} = \arg\max_{a_{i}} Q^{\phi}_{i}(\mathbf{s}, a_{i}), \\
                                                  \alpha_{i}(\mathbf{s}, a_{i}) & a_{i} \neq \arg\max_{a_{i}} Q^{\phi}_{i}(\mathbf{s}, a_{i}).
                                             \end{cases}
        \label{eq:delta}
        \end{equation}
        Actually, the closed-form expression of $\delta_{i}(\mathbf{s}, a_{i})$ is written as $|\mathcal{N}|^{-1} w_{i}(\mathbf{s}, a_{i})^{-1}$. If inserting the condition that $w_{i}(\mathbf{s}, a_{i}) = 1 / |\mathcal{N}|$ when $\mathit{a}_{i} = \arg\max_{a_{i}} Q^{\phi}_{i}(\mathbf{s}, a_{i})$ as well as defining $\delta_{i}(\mathbf{s}, a_{i})$ as $\alpha_{i}(\mathbf{s}, a_{i})$ when $\mathit{a}_{i} \neq \arg\max_{a_{i}} Q^{\phi}_{i}(\mathbf{s}, a_{i})$, Eq.\ref{eq:delta} is obtained. The term $\mathbf{b}(\mathbf{s})$ is cancelled in Eq.\ref{eq:td_error_shapley_q_learning} thanks to the condition such that $\sum_{i \in \mathcal{N}} w_{i}(\mathbf{s}, a_{i})^{-1} b_{i}(\mathbf{s}) = 0$. Note that the condition to $w_{i}(\mathbf{s}, a_{i})$ in Theorem \ref{thm:shapley_q_optimal} should hold for the convergence of SHAQ in implementation (see Appendix \ref{subsubsec:derivation_of_shapley_q-learning}).
    
    \subsection{Validity and Interpretability}
    \label{subsec:validity_of_shapley_q-learning}
    
        In this section, we show the validity of SBOE and the interpretability of SHAQ, i.e., providing the reasons why SBOE is valid to be formulated and SHAQ is an interpretable value factorisation method for the global reward game.
        \begin{theorem}
        \label{thm:shapley_value_core}
            The optimal Markov Shapley value is a solution in the Markov core under Markov convex game with the grand coalition.
        \end{theorem}
        
        \begin{remark}
            For an arbitrary state $\mathbf{s} \in \mathcal{S}$, by C.2 it is not difficult to check that even if an arbitrary agent $i$ is dummy (i.e., $Q^{\phi^{*}}_{i}(\mathbf{s}, a_{i}) = 0$ for some $i \in \mathcal{N}$), $Q^{\pi^{*}}(\mathbf{s}, \mathbf{a})$ and $Q^{\phi^{*}}_{j}(\mathbf{s}, a_{j}), \forall j \neq i$ would not be zero if $b_{i}(\mathbf{s}) \neq 0$. If the extreme case happens that for an arbitrary state $\mathbf{s} \in \mathcal{S}$ all agents are dummies, since $\sum_{i \in \mathcal{N}} w_{i}(\mathbf{s}, a_{i})^{-1} b_{i}(\mathbf{s}) = 0$ we are allowed to set $b_{i}(\mathbf{s}) = 0, \forall i \in \mathcal{N}$ so that $Q^{\pi^{*}}(\mathbf{s}, \mathbf{a}) = 0$ and efficiency such that $\max_{\mathbf{a}} Q^{\pi^{*}}(\mathbf{s}, \mathbf{a}) = \sum_{i \in \mathcal{N}} \max_{a_{i}} Q_{i}^{\phi^{*}}(\mathbf{s}, a_{i})$ is still valid.
        \label{rmk:dummy_agents}
        \end{remark}
        
        First, we give a proof for showing that the optimal MSV is a solution in Markov core under the grand coalition, as Theorem \ref{thm:shapley_value_core} shows. Since a solution in Markov core implies the optimal global value (see Remark \ref{rmk:example_core} in Appendix \ref{subsubsec:insight_into_the_core_appendix}), we can conclude that \textit{the optimal MSV can lead to the optimal global value} (a.k.a. social welfare), which links Condition C.1 to Markov core. As a result, \textit{solving SBOE is equivalent to solving Markov core under the grand coalition and SHAQ is actually a learning algorithm that reliably converges to Markov core}. As per the definition in Section \ref{sec:markov_convex_game}, we can say that SHAQ leads to the result that no agents have incentives to deviate from the grand coalition, which provides an interpretation of value factorisation for global reward game. Condition C.2 is a condition that \textit{maintains the validity of the relationship between the optimal MSQ and the optimal global Q-value even if there exist dummy agents} (see Remark \ref{rmk:dummy_agents}), so that the definition of SBOE is valid for MCG and MSQ in almost every case, which preserves the completeness of the theory.

    \subsection{Implementations}
    \label{sec:implementation_of_shapley_q-learning}
    
        We now describe a practical implementation of SHAQ for Dec-POMDP \citep{oliehoek2012decentralized} (i.e. the global reward game but with partial observations). First, the global state is replaced by the history of each agent to guarantee the optimal deterministic joint policy \citep{oliehoek2012decentralized}. Accordingly, Markov Shapley Q-value is denoted as $Q_{i}^{\phi}(\tau_{i}, a_{i})$, wherein $\tau_{i}$ is a history of partial observations of agent $i$. Since the paradigm of centralised training decentralised execution (CTDE) \citep{oliehoek2008optimal} is applied, the global state (i.e. $\mathbf{s}$) for $\hat{\alpha}_{i}(\mathbf{s}, a_{i})$ can be obtained during training.
        
        \begin{proposition}
        \label{prop:shapley_value_approximate}
            Suppose any action marginal contribution can be factorised to the form such that $\Upphi_{i}(\mathbf{s}, a_{i} | \mathcal{C}_{i}) = \sigma(\mathbf{s}, \mathbf{a}_{ \scriptscriptstyle\mathcal{C}_{i} \cup \{i\} }) \ \hat{Q}_{i}(\mathbf{s}, a_{i})$. With the condition such that
            \begin{equation*}
                \mathbb{E}_{\mathcal{C}_{i} \sim Pr(\mathcal{C}_{i} | \mathcal{N} \backslash \{i\})} \left[ \sigma(\mathbf{s}, \mathbf{a}_{ \scriptscriptstyle\mathcal{C}_{i} \cup \{i\} }) \right] =
                \begin{cases} 
                     1 & \ \ a_{i} = \arg\max_{a_{i}} Q^{\phi}_{i}(\mathbf{s}, a_{i}), \\
                     K \in (0, 1) & \ \ a_{i} \neq \arg\max_{a_{i}} Q^{\phi}_{i}(\mathbf{s}, a_{i}),
                \end{cases}
            \end{equation*}
            we have
            \begin{equation}
                \begin{cases} 
                     Q_{i}^{\phi}(\mathbf{s}, a_{i}) = \hat{Q}_{i}(\mathbf{s}, a_{i}) & \ \ a_{i} = \arg\max_{a_{i}} \hat{Q}_{i}(\mathbf{s}, a_{i}), \\
                     \alpha_{i}(\mathbf{s}, a_{i}) \ Q^{\phi}_{i}(\mathbf{s}, a_{i}) = \hat{\alpha}_{i}(\mathbf{s}, a_{i}) \ \hat{Q}_{i}(\mathbf{s}, a_{i}) & \ \ a_{i} \neq \arg\max_{a_{i}} \hat{Q}_{i}(\mathbf{s}, a_{i}),
                \end{cases}
            \label{eq:shapley_q_approximate}
            \end{equation}
            where $\hat{\alpha}_{i}(\mathbf{s}, a_{i}) = \mathbb{E}_{\mathcal{C}_{i} \sim Pr(\mathcal{C}_{i} | \mathcal{N} \backslash \{i\})} \left[ \hat{\psi}_{i}(\mathbf{s}, a_{i}; \mathbf{a}_{ \scriptscriptstyle\mathcal{C}_{i} }) \right]$ and $\hat{\psi}_{i}(\mathbf{s}, a_{i}; \mathbf{a}_{ \scriptscriptstyle\mathcal{C}_{i} }) := \alpha_{i}(\mathbf{s}, a_{i}) \ \sigma(\mathbf{s}, \mathbf{a}_{ \scriptscriptstyle\mathcal{C}_{i} \cup \{i\} })$.
        \end{proposition}

        Compatible with the decentralised execution, we use only one parametric function $\hat{Q}_{i}(\tau_{i}, a_{i})$ to directly approximate $Q_{i}^{\phi}(\tau_{i}, a_{i})$. By inserting Eq.\ref{eq:shapley_q_approximate} into Eq.\ref{eq:td_error_shapley_q_learning}, $\delta_{i}(\mathbf{s}, a_{i})$ is transformed into the form as follows:
        \begin{equation}
            \hat{\delta}_{i}(\mathbf{s}, a_{i}) = \begin{cases} 
                                                  1 & a_{i} = \arg\max_{a_{i}} \hat{Q}_{i}(\mathbf{s}, a_{i}), \\
                                                  \hat{\alpha}_{i}(\mathbf{s}, a_{i}) & a_{i} \neq \arg\max_{a_{i}} \hat{Q}_{i}(\mathbf{s}, a_{i}),
                                             \end{cases}
        \label{eq:delta_new}
        \end{equation}
        where $\hat{\alpha}_{i}(\mathbf{s}, a_{i}) = \mathbb{E}_{\mathcal{C}_{i} \sim Pr(\mathcal{C}_{i} | \mathcal{N} \backslash \{i\})} \left[ \hat{\psi}_{i}(\mathbf{s}, a_{i}; \mathbf{a}_{ \scriptscriptstyle\mathcal{C}_{i} }) \right]$. To solve partial observability, $\hat{Q}_{i}(\tau_{i}, a_{i})$ is empirically represented as recurrent neural network (RNN) with GRUs \citep{chung2014empirical}. $\hat{\psi}_{i}(\mathbf{s}, a_{i}; \mathbf{a}_{ \scriptscriptstyle\mathcal{C}_{i} })$ is directly approximated by a parametric function $\mathlarger{F}_{\mathbf{s}} + 1$ and thus $\hat{\alpha}_{i}(\mathbf{s}, a_{i})$ can be expressed as follows:
        \begin{equation}
            \hat{\alpha}_{i}(\mathbf{s}, a_{i}) = \frac{1}{M} \sum_{k = 1}^{M} \mathlarger{F}_{\mathbf{s}} \left( \hat{Q}_{\mathcal{C}_{i}^{k}}(\tau_{\mathcal{C}_{i}^{k}}, \mathbf{a}_{\mathcal{C}_{i}^{k}}), \ \hat{Q}_{i}(\tau_{i}, a_{i}) \right) + 1,
        \label{eq:alpha_deep_representation}
        \end{equation}
        
        where $\hat{Q}_{\mathcal{C}_{i}^{k}}(\tau_{\mathcal{C}_{i}^{k}}, \mathbf{a}_{\mathcal{C}_{i}^{k}}) = \frac{1}{|\mathcal{C}_{i}^{k}|}\sum_{j \in \mathcal{C}_{i}^{k}} \hat{Q}_{j}(\tau_{j}, a_{j})$ and $\mathcal{C}_{i}^{k}$ is sampled $\mathit{M}$ times from $\mathit{Pr}(\mathcal{C}_{i} | \mathcal{N} \backslash \{i\})$ (i.e., implemented as Remark \ref{rmk:coalition_generation} suggests) to approximate $\mathbb{E}_{\mathcal{C}_{i} \sim \mathit{Pr}(\mathcal{C}_{i} | \mathcal{N} \backslash \{i\})}[ \hat{\psi}_{i}(\mathbf{s}, a_{i}; \mathbf{a}_{ \scriptscriptstyle\mathcal{C}_{i} }) ]$ using Monte Carlo approximation; and $\mathlarger{F}_{\mathbf{s}}$ is a monotonic function, followed by an absolute activation function, whose weights are generated from hyper-networks w.r.t. the global state. We show that Eq.\ref{eq:alpha_deep_representation} satisfies the condition to $w_{i}(\mathbf{s}, a_{i})$ in Theorem \ref{thm:shapley_q_optimal}  (see Appendix \ref{subsubsec:implementation_of_alpha_appendix}), so it is a reliable implementation.
        
        By using the framework of fitted Q-learning \citep{ernst2005tree} to solve large number of states (i.e., could be usually infinite) and plugging in the above designed modules, the practical least-square-error loss function derived from Eq.\ref{eq:td_error_shapley_q_learning} is therefore stated as follows:
        \begin{equation}
            \min_{\theta, \lambda} \mathbb{E}_{\mathbf{s}, \mathbf{\tau}, \mathbf{a}, R, \mathbf{\tau}'} \bigg[ \ \Big( \ R \ + \ 
            \gamma \sum_{i \in \mathcal{N}} \max_{a_{i}} \hat{Q}_{i}(\tau_{i}', a_{i}; \theta^{-})
            - 
            \sum_{i \in \mathcal{N}} \hat{\delta}_{i}(\mathbf{s}, a_{i}; \lambda) \ \hat{Q}_{i}(\tau_{i}, a_{i}; \theta) \ \Big)^{2} \ \bigg],
        \label{eq:deep_shapley_q_learning_loss}
        \end{equation}
        where all agents share the parameters of $\hat{Q}_{i}(\mathbf{s}, a_{i}; \theta)$ and $\hat{\alpha}_{i}(\mathbf{s}, a_{i}; \lambda)$ respectively; and $\hat{Q}_{i}(\mathbf{s}', a_{i}; \theta^{-})$ works as the target where $\theta^{-}$ is periodically updated. The general training procedure follows the paradigm of DQN \citep{mnih2013playing}, with a replay buffer to store the online collection of agents' episodes. To depict an overview of the algorithm, the pseudo code is shown in Appendix \ref{sec:algorithm_shapley_q}.
    
\section{Related Work}
\label{sec:related_work}
    \textbf{Value Factorisation in MARL.} To deal with the instability during training in global reward game by independent learners \citep{claus1998dynamics}, the centralised training and decentralised execution (CTDE) \citep{oliehoek2008optimal} was proposed and it became a general paradigm for MARL. Based on CTDE, MADDPG \citep{lowe2017multi} learns a global Q-value that can be regarded as assigning the same credits to all agents during training \citep{Wang_2020}, which may cause the unfair credit assignment \citep{wolpert2002optimal}. To avoid this problem, VDN \citep{SunehagLGCZJLSL18} was proposed to learn the factorised Q-value, assuming that any global Q-value equals to the sum of decentralised Q-values. Nevertheless, this factorisation may limit the representation of the global Q-value. To mitigate this issue, QMIX \citep{RashidSWFFW18} and QTRAN \citep{SonKKHY19} were proposed to represent the global Q-value with a richer class w.r.t. decentralised Q-values, based on the assumption (called Individual-Global-Max) of convergence to the optimal joint deterministic policy. Markov Shapley value proposed in this paper belongs to the family of value factorisation, based on the game-theoretical framework called MCG that enjoys the interpretability. From the conventional cooperative games (e.g., network flow game \citep{kalai1982generalized}, induced subgraph game \citep{deng1994complexity} that can be used for modelling social networks, and facility location game \citep{deng1999algorithmic}), it is insightful that the coalition introduced in this paper exists. In many scenarios, however, the information of coalition might be unknown. Therefore, the latent coalition is assumed, and we only need to concentrate on the observable information, e.g., the global reward.
    
    \textbf{Relationship to VDN.} By setting $\delta_{i}(\mathbf{s}, a_{i}) = 1$ for all state-action pairs, SHAQ degrades to VDN \citep{SunehagLGCZJLSL18}. Although VDN tried to tackle the problem of dummy agents, \citet{SunehagLGCZJLSL18} did not give a theoretical guarantee on identifying it. The Markov Shapley value theory proposed in this paper well addresses this issue from both theoretical and empirical aspects. These aspects show that VDN is a subclass of SHAQ. The theoretical framework proposed in this paper answers to why VDN works well in most scenarios but performs poorly in some scenarios (i.e., $\delta_{i}(\mathbf{s}, a_{i})=1$ in Eq.\ref{eq:td_error_shapley_q_learning} was incorrectly defined over the suboptimal actions).
    
    \textbf{Relationship to COMA.} Compared with COMA \citep{foerster2018counterfactual}, each agent $i$'s credit assignment $\bar{Q}_{i}(\mathbf{s}, a_{i})$ is mathematically expressed as follows:
    \begin{equation*}
        \begin{split}
            \bar{Q}_{i}(\mathbf{s}, a_{i}) = \bar{Q}^{\pi}(\mathbf{s}, \mathbf{a}) - \bar{Q}^{\pi_{-i}}(\mathbf{s}, \mathbf{a}_{- i}), \\
            \bar{Q}^{\pi_{-i}}(\mathbf{s}, \mathbf{a}_{-i}) = \sum_{a_{i}} \pi_{i}(a_{i}|\mathbf{s}) \bar{Q}^{\pi} \left(\mathbf{s}, (\mathbf{a}_{-i}, a_{i}) \right),
        \end{split}
    \end{equation*}
    where subscript $-i$ indicates the agents excluding $i$. $\bar{Q}_{i}(\mathbf{s}, a_{i})$ can be seen as the action marginal contribution between the grand coalition Q-value and the coalition Q-value excluding the agent $i$, under \textit{some permutation to form the grand coalition} wherein agent $i$ is located at the \textit{last position}. The efficiency is obviously violated (i.e., the sum of optimal action marginal contributions defined here is unlikely to be equal to the optimal grand coalition Q-value). In contrast to COMA, SHAQ considers all permutations to form the grand coalition to preserve the efficiency.
    
    \textbf{Relationship to Independent Learning.} Independent learning (e.g. IQL \citep{claus1998dynamics}) can be also seen as a special credit assignment, however, the credit assigned to each agent is still with no intuitive interpretation. Mathematically, suppose that $\bar{Q}_{i}(\mathbf{s}, a_{i})$ is the independent Q-value of agent $i$, we can rewrite it in the form consisting of action marginal contributions such that
    \begin{equation*}
        \bar{Q}_{i}(\mathbf{s}, a_{i}) = \mathbb{E}_{\mathcal{C}_{i} \sim Pr(\mathcal{C}_{i} | \mathcal{N} \backslash \{i\})} \left[ \bar{\Upphi}_{i}(\mathbf{s}, a_{i} | \mathcal{C}_{i}) \right].
    \end{equation*}
    It is intuitive to see that the independent Q-value is a direct approximation of MSQ, ignoring coalition formation, while SHAQ considers coalition formation in approximation. This gives an explanation for why independent learning works well in some cooperative tasks \citep{PapoudakisC0A21}. Nevertheless, it encounters the same issue as in COMA, the loss of properties led by the coalition formation.
    
    \textbf{Relationship to SQDDPG.} We now discuss the relationship between SQDDPG \citep{Wang_2020} and SHAQ. In terms of algorithms, SQDDPG belongs to policy gradient methods (i.e. an approximation of policy iteration) while SHAQ belongs to value based methods (i.e. an approximation of value iteration). Since policy iteration (with one-step policy evaluation) is equivalent to value iteration \citep{bertsekas2019reinforcement} (at least under a finite state space and a finite action space), the theory behind SHAQ directly \textit{fills the gap in SQDDPG on theoretical guarantees of convergence to optimal joint policy}. Specifically, the learning procedure of SQDDPG iteratively performs the following two stages:
    \begin{equation*}
        \begin{split}
            &\textbf{Stage 1:} \quad \min_{\theta} \mathbb{E}_{\mathbf{s}, \mathbf{a}, R, \mathbf{s}'} \bigg[ \ \Big( \ R \ + \ 
            \gamma \sum_{i \in \mathcal{N}} \hat{Q}_{i}^{\phi}(\mathbf{s}', a_{i}'; \theta^{-})
            - 
            \sum_{i \in \mathcal{N}} \hat{Q}_{i}^{\phi}(\mathbf{s}, a_{i}; \theta) \ \Big)^{2} \ \bigg]. \\
            &\textbf{Stage 2:} \quad \pi_{i}(\mathbf{s}) \in \arg\max_{a_{i}} \hat{Q}_{i}^{\phi}(\mathbf{s}, a_{i}; \theta).
        \end{split}
    \end{equation*}
    
    It can be observed that both SQDDPG and SHAQ ideally converge to the same optimal MSQs w.r.t. the optimal actions such that 
    \begin{equation*}
        \mathbb{E}_{\mathbf{s}, \mathbf{s}'} \bigg[ \ \Big( \ \max_{\mathbf{a}} R(\mathbf{s}, \mathbf{a}) \ + \ 
            \gamma \sum_{i \in \mathcal{N}} \max_{a_{i}'} \hat{Q}_{i}^{\phi^{*}}(\mathbf{s}', a_{i}')
            - 
            \sum_{i \in \mathcal{N}} \max_{a_{i}} \hat{Q}_{i}^{\phi^{*}}(\mathbf{s}, a_{i}) \ \Big)^{2} \ \bigg] = 0.
    \end{equation*}
    
    However, about suboptimal actions, \textit{SQDDPG does not provide any theoretical guarantee}, whereas SHAQ does with specific implementations as shown in Eq.\ref{eq:alpha_deep_representation} to match the theoretical results shown in this paper. Note that this is critical to reliable interpretations of the optimal MSQ w.r.t. suboptimal actions (e.g., for detecting adversarial attacks on controllers if deployed in industry \citep{fawzi2014secure}).
    
\section{Experiments}
\label{sec:experiments}
    In this section, we show the experimental results of SHAQ on Predator-Prey \citep{bohmer2020deep} and various tasks in StarCraft Multi-Agent Challenge (SMAC) \footnote{The version that we use in this paper is SC2.4.6.2.69232 rather than the newer SC2.4.10. As reported from \cite{rashid2020weighted}, the performance is not comparable across versions.}. The baselines that we select for comparison are COMA \cite{foerster2018counterfactual}, VDN \citep{SunehagLGCZJLSL18}, QMIX \citep{RashidSWFFW18}, MASAC \citep{iqbal2019actor}, QTRAN \citep{SonKKHY19}, QPLEX \citep{wang2020qplex} and W-QMIX (including CW-QMIX and OW-QMIX) \citep{rashid2020weighted}. The implementation details of our algorithm are shown in Appendix \ref{subsec:implementation_details_shapley_q_learning}, whereas the implementation of baselines are from \cite{rashid2020weighted} \footnote{The source code of baseline implementation is from \url{https://github.com/oxwhirl/wqmix}.}. We also compare SHAQ with SQDDPG \citep{Wang_2020} \footnote{The code of SQDDPG is implemented based on \url{https://github.com/hsvgbkhgbv/SQDDPG}.}, which is shown in Appendix \ref{subsec:comparison_with_sqddpg}. For all experiments, we use the $\epsilon$-greedy exploration strategy, where $\epsilon$ is annealed from 1 to 0.05. The annealing time steps vary among different experiments. For Predator-Prey, we apply 1 million time steps for annealing, following the setup from \cite{wang2020qplex}. For the easy and hard maps in SMAC, we apply 50k time steps for annealing, the same as that in \cite{samvelyan2019starcraft}; while for the super-hard maps in SMAC, we apply 1 million time steps for annealing to obtain more explorations so that more state-action pairs can be visited. About the replay buffer size, we set as 5000 for all algorithms that is the same as that in \cite{rashid2020weighted}. To fairly evaluate all algorithms, we run each experiment with 5 random seeds. All graphs showing experimental results are plotted with the median and 25\%-75\% quartile shading. About the interpretability of algorithms, we evaluate the algorithms with both both $\epsilon$-greedy policy (i.e., $\epsilon=0.8$) for obtaining mixed optimal and suboptimal actions and greedy policy for obtaining pure optimal actions. The ablation study of SHAQ is shown in Appendix \ref{subsec:ablation_study}. 
    
    \subsection{Predator-Prey}
    \label{subsec:predator-prey}
        
        We firstly run the experiments on a partially-observable task called Predator-Prey \citep{bohmer2020deep}, wherein 8 predators that are feasible to be controlled aim to capture 8 preys with random policies in a 10x10 grid world. Each agent's observation is a 5x5 sub-grid centering around it. If a prey is captured by coordination of 2 agents, predators will be rewarded by 10. On the other hand, each unsuccessful attempt by only 1 agent will be punished by a negative reward p. In this experiment, we study the behaviors of each algorithm under different values of p (that describes different levels of coordination). As \cite{rashid2020weighted} reported, only QTRAN and W-QMIX can solve this task, while \cite{wang2020qplex} found that the failure was primarily due to the lack of explorations. As a result, we apply the identical epsilon annealing schedule (i.e. 1 million time steps) adopted in \cite{wang2020qplex}. 
        
        \begin{figure*}[ht!]
            \centering
            \begin{subfigure}[b]{0.32\textwidth}
                \centering
                \includegraphics[width=\textwidth]{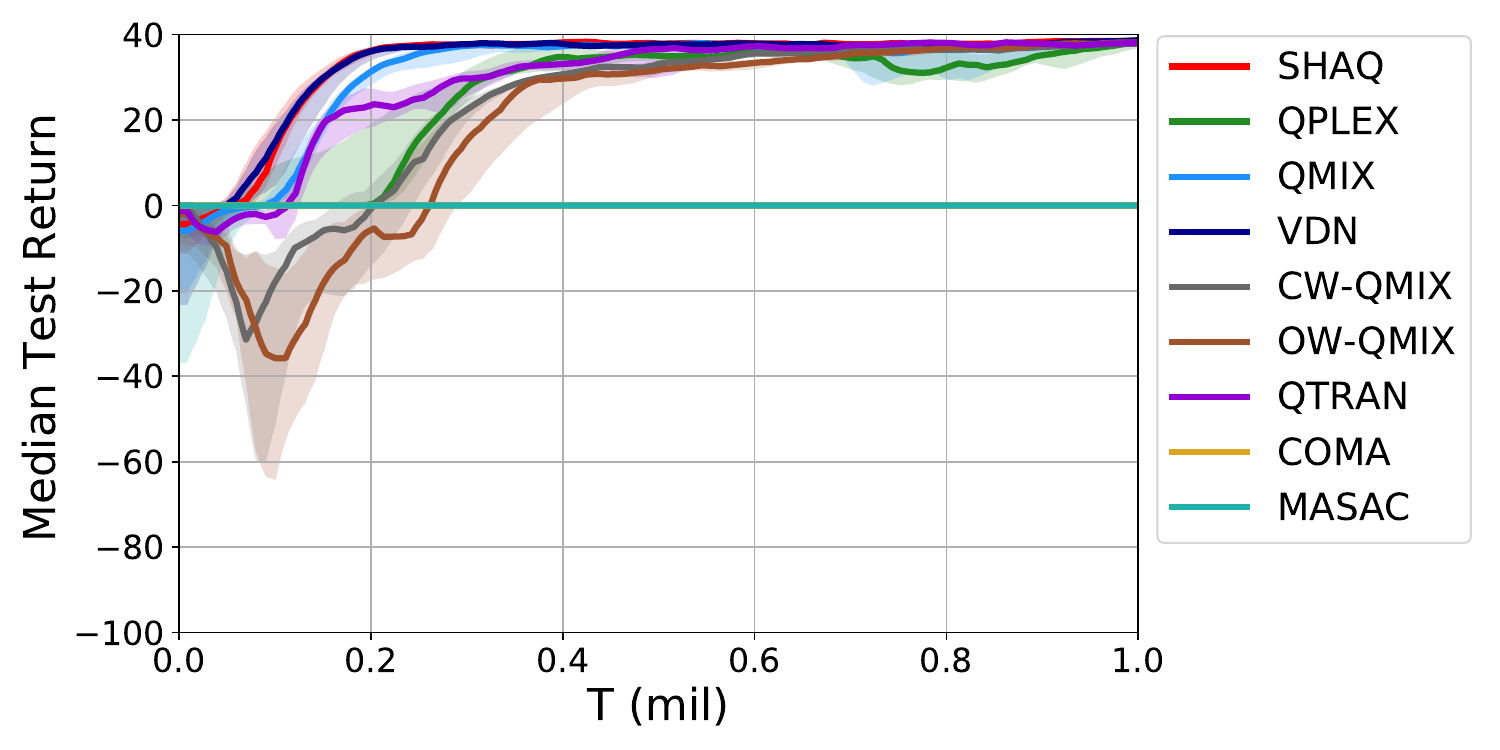}
                \caption{p=-0.5.}
            \end{subfigure}
            ~
            \begin{subfigure}[b]{0.32\textwidth}
                \centering
                \includegraphics[width=\textwidth]{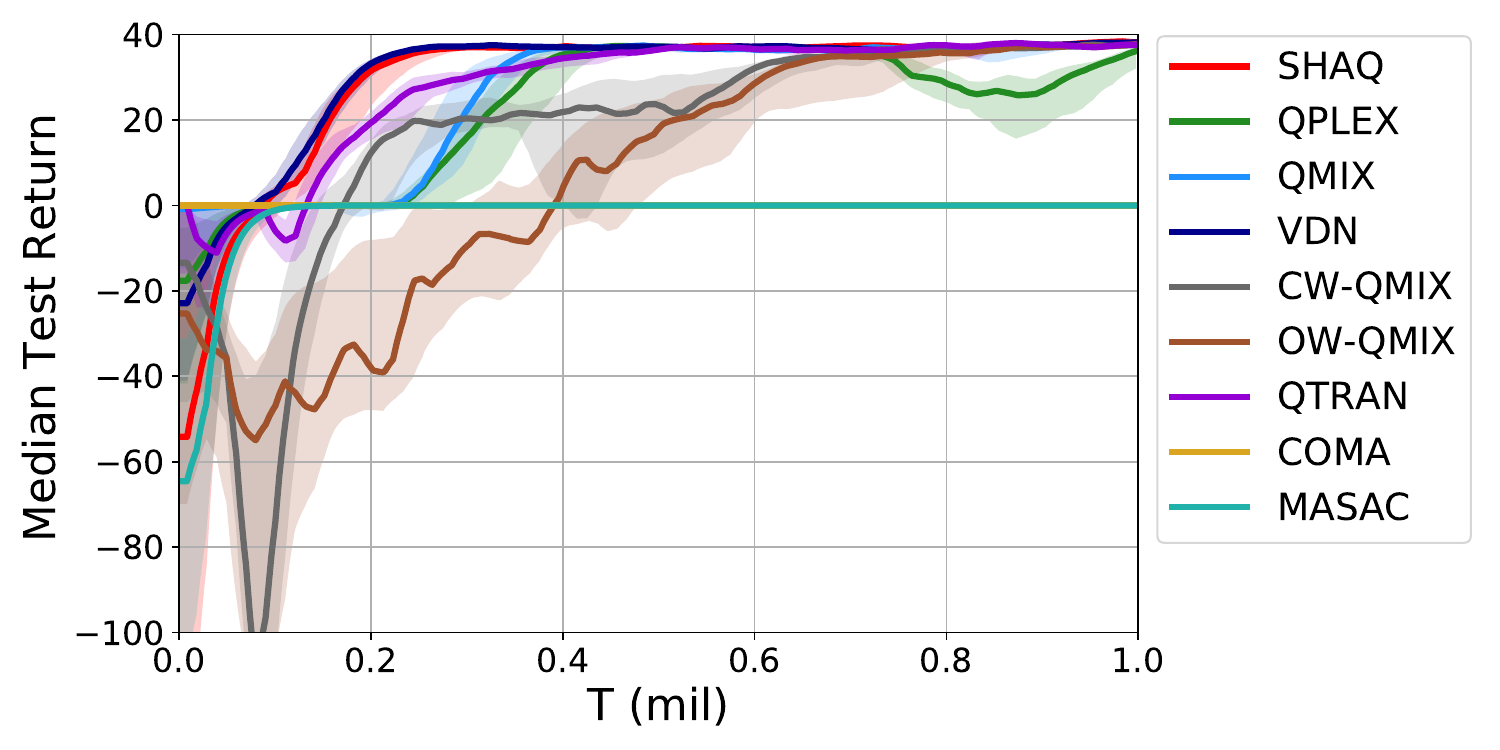}
                \caption{p=-1.}
            \end{subfigure}
            ~
            \begin{subfigure}[b]{0.32\textwidth}
                \centering
                \includegraphics[width=\textwidth]{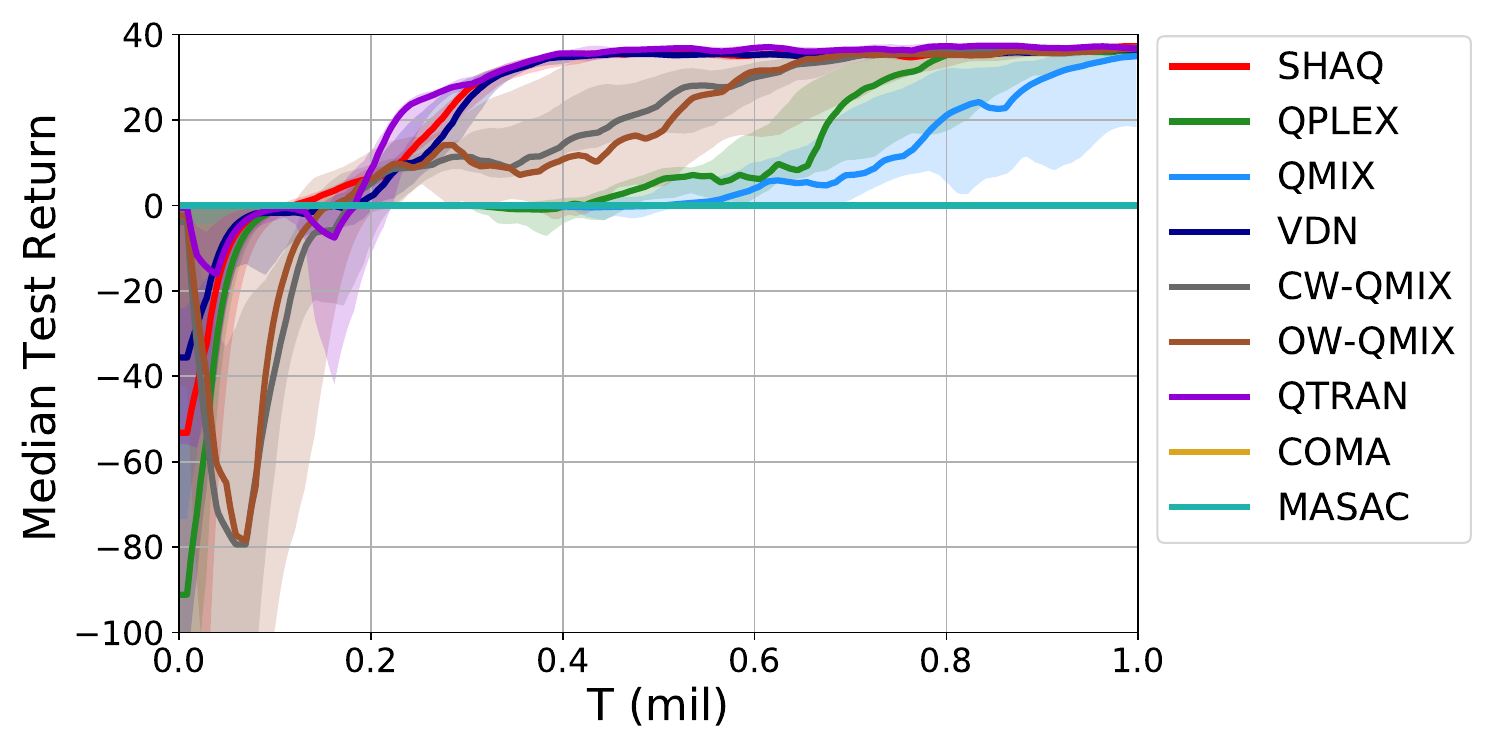}
                \caption{p=-2.}
            \end{subfigure}
            \caption{Median test return for Predator-Prey with different values of p.}
        \label{fig:predator_prey}
        \end{figure*}
        
        \textbf{Performance Analysis.} As Figure \ref{fig:predator_prey} shows, SHAQ can always solve the tasks with different values of p. With the epsilon annealing strategy from \cite{wang2020qplex}, W-QMIX does not perform as well as reported in \cite{rashid2020weighted}. The reason could be its poor robustness to the increased explorations \citep{rashid2020weighted} for this environment (see the evidential experimental results in Appendix \ref{subsec:extra_experimental_results_for_predator_prey}). The good performance of VDN validates our analysis in Section \ref{sec:related_work}, whereas the performance of QTRAN is surprisingly almost invariant to the value of p. The performances of QPLEX and QMIX become obviously worse when p=-2. The failure of MASAC and COMA could be due to that relative overgeneralisation\footnote{Relative overgeneralisation is a common game theoretic pathology that the suboptimal actions are preferred when matched with arbitrary actions from the collaborating agents \citep{wei2016lenient}.} prevents policy gradient methods from better coordination \citep{wei2018multiagent}.
        
        \begin{figure*}[ht!]
            \centering
            \begin{subfigure}[b]{0.3\textwidth}
            	\centering
        	    \includegraphics[width=\textwidth]{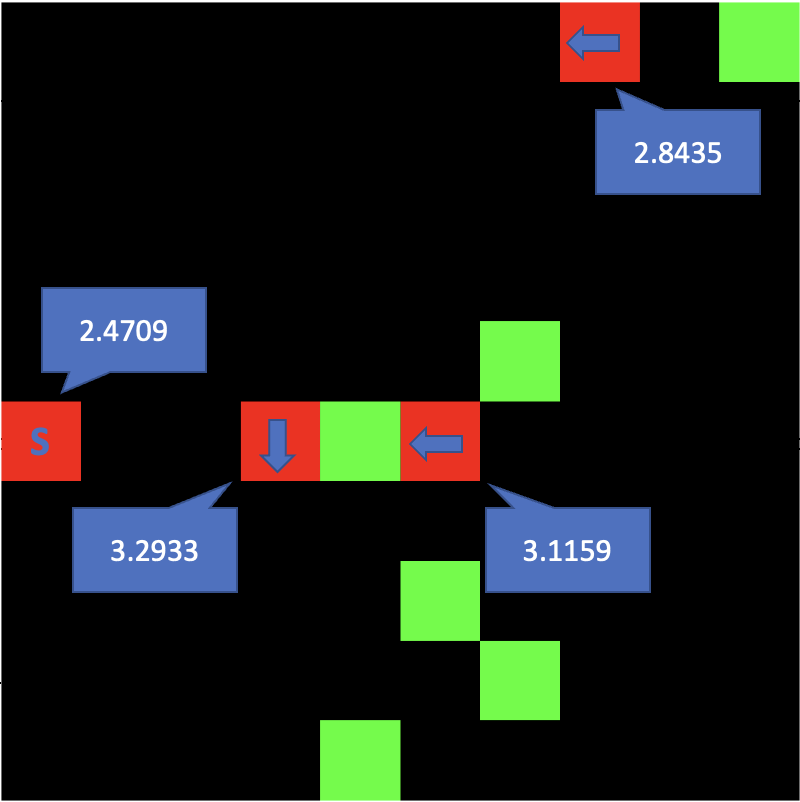}
        	    \caption{SHAQ: $\epsilon$-greedy.}
        	\label{fig:shaq_epsilon_greedy_pp}
            \end{subfigure}
            ~
            \begin{subfigure}[b]{0.3\textwidth}
                \centering                \includegraphics[width=\textwidth]{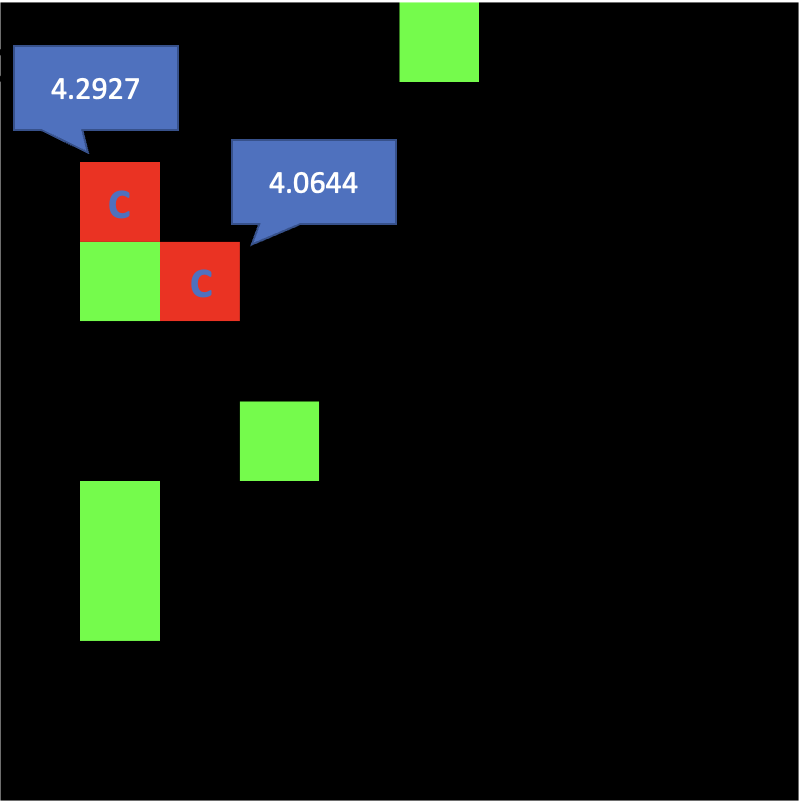}
                \caption{SHAQ: greedy.}
            \label{fig:shaq_greedy_pp}
            \end{subfigure}
            \caption{Visualisation of the evaluation for SHAQ on Predator-Prey: each red square is a controllable agent, whereas each green square indicates a prey. Each agent's factorised Q-value is reported in the bubble in blue and the symbols within the squares indicate the action of each agent (i.e., arrows imply the movement direction, ``S'' implies staying and ``C'' implies capturing a prey that is valid only when the agent is around a prey).}
        \label{fig:study_on_factorised_q_values_pp}
        \end{figure*}
        
        \textbf{Interpretability of SHAQ.} To verify that SHAQ possesses the interpretability, we show its credit assignment on Predator-Prey. As we see from Figure \ref{fig:shaq_greedy_pp}, all agents are around and capture a prey, so both of them perform the optimal actions and deserve almost the equal optimal credit assignment as $4.2927$ and $4.0644$, which verifies our theoretical claim. From Figure \ref{fig:shaq_epsilon_greedy_pp}, it can be seen that two agents are far away from preys, so they receive low credits as $2.4709$ and $2.8435$. On the other hand, the other two agents are around a prey, but they do not perform the optimal action ``capture'', so they receive less credits than the two agents in Figure \ref{fig:shaq_greedy_pp}. Nevertheless, they are around a prey, so they perform better than those agents that are far away from preys and receive comparatively greater credits as $3.2933$ and $3.1159$. The coherent credit assignments in both Figure \ref{fig:shaq_epsilon_greedy_pp} and \ref{fig:shaq_greedy_pp} implies that the assigned credits reflect agents' contributions (verifying (iii) in Proposition \ref{prop:shapley_value_properties})
        , i.e., each agent receives the credit that is consistent with its decision.
        
    \subsection{StarCraft Multi-Agent Challenge}
    \label{subsec:smac}
    
        We next evaluate SHAQ on the more challenging SMAC tasks, the environmental settings of which are the same as that in \cite{samvelyan2019starcraft}. To broadly compare the performance of SHAQ with baselines, we select 4 easy maps: 8m, 3s5z, 1c3s5z and 10m\_vs\_11m; 3 hard maps: 5m\_vs\_6m, 3s\_vs\_5z and 2c\_vs\_64zg; and 4 super-hard maps: 3s5z\_vs\_3s6z, Corridor, MMM2 and 6h\_vs\_8z. All training is through online data collection. Due to the limited space, we only show partial results in the main part of paper and leave the rest in Appendix \ref{subsec:experimental_results_on_extra_smac_maps}. 
        
        \textbf{Performance Analysis.} It shows in Figure \ref{fig:easy_hard_smac} that SHAQ outperforms all baselines on all maps, except for 6h\_vs\_8z. On 6h\_vs\_8z, SHAQ can beat all baselines except for CW-QMIX. VDN performs well on 4 maps but bad on the other 2 maps, which still verifies our analysis in Section \ref{sec:related_work}. QMIX and QPLEX perform well on the most of maps, except for 3s\_vs\_5z, 2c\_vs\_64zg and 6h\_vs\_8z. As for COMA, MADDPG and MASAC, their poor performances could be due to the weak adaptability to challenging tasks. Although QTRAN can theoretically represent the complete class of the global Q-value \citep{SonKKHY19}, its complicated learning paradigm could impede the convergence to the value function for challenging tasks and therefore result in the poor performance. Although W-QMIX performs well on some maps, owing to lacking a law on hyperparameter tuning \cite{rashid2020weighted} it is difficult to be adapted for all scenarios (see Appendix \ref{subsec:extra_weighted_qmix_variant_hyperparameters}).
        
        \begin{figure*}[ht!]
            \centering
            \begin{subfigure}[b]{0.32\linewidth}
                \centering
                \includegraphics[width=\textwidth]{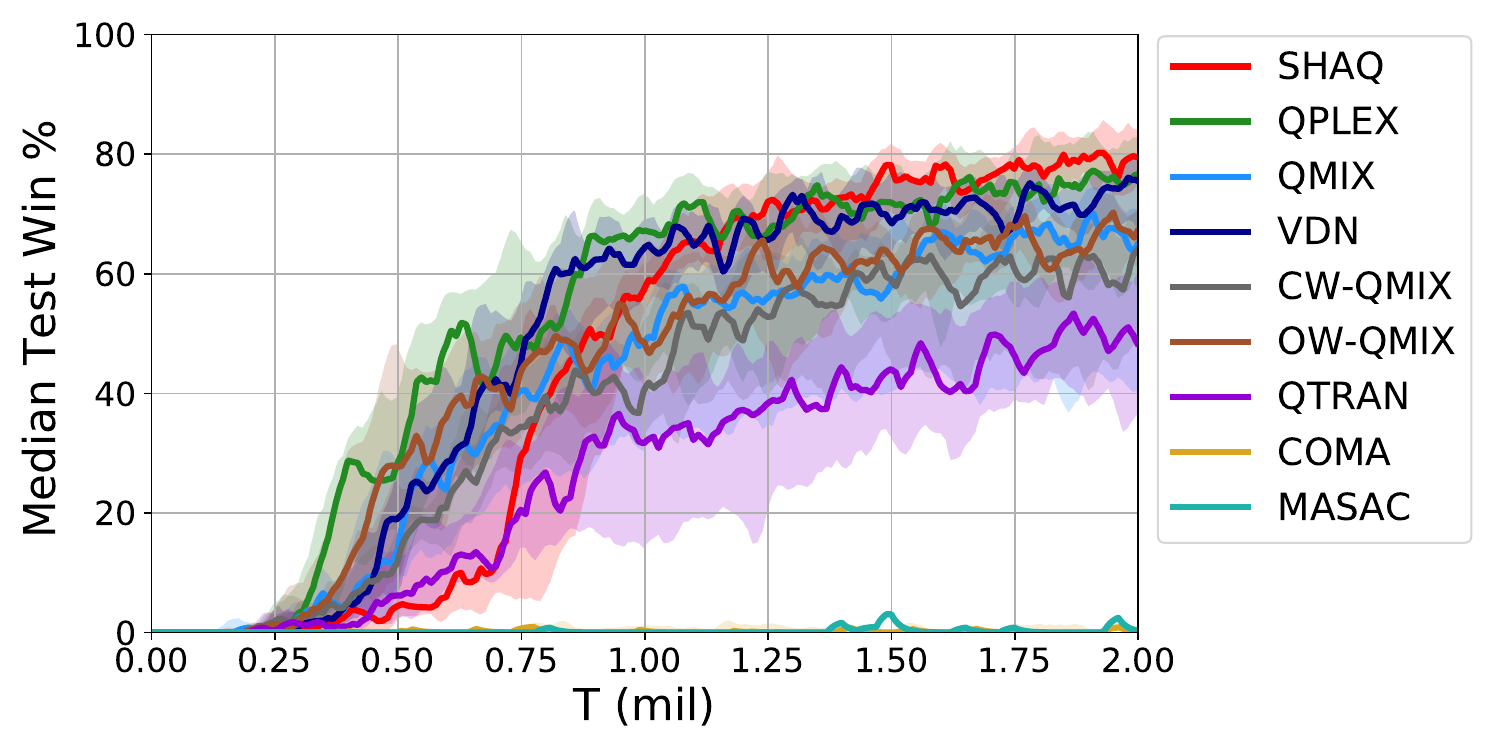}
                \caption{5m\_vs\_6m.}
            \end{subfigure}
            ~
            \begin{subfigure}[b]{0.32\linewidth}
                \centering
                \includegraphics[width=\textwidth]{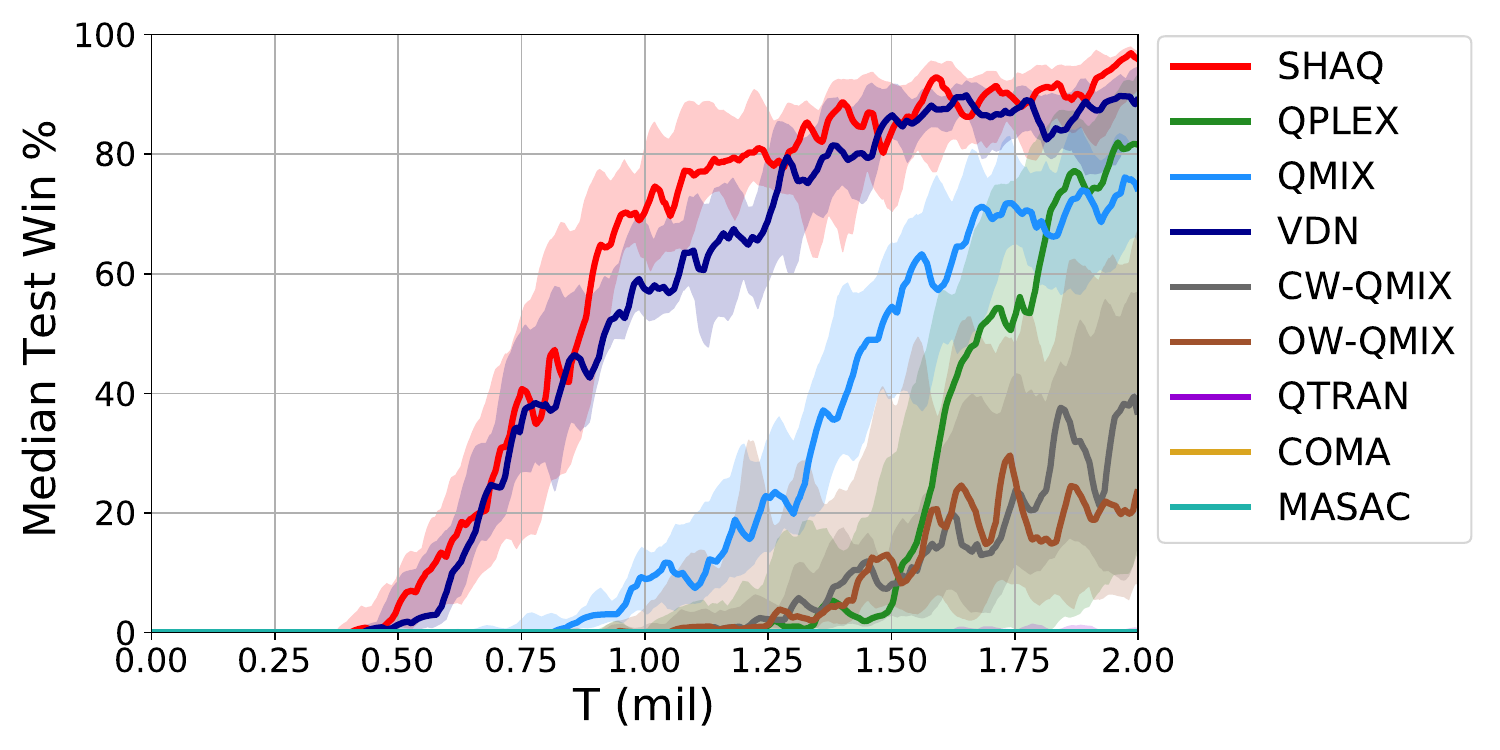}
                \caption{3s\_vs\_5z.}
            \end{subfigure}
            ~
            \begin{subfigure}[b]{0.32\linewidth}
                \centering
                \includegraphics[width=\textwidth]{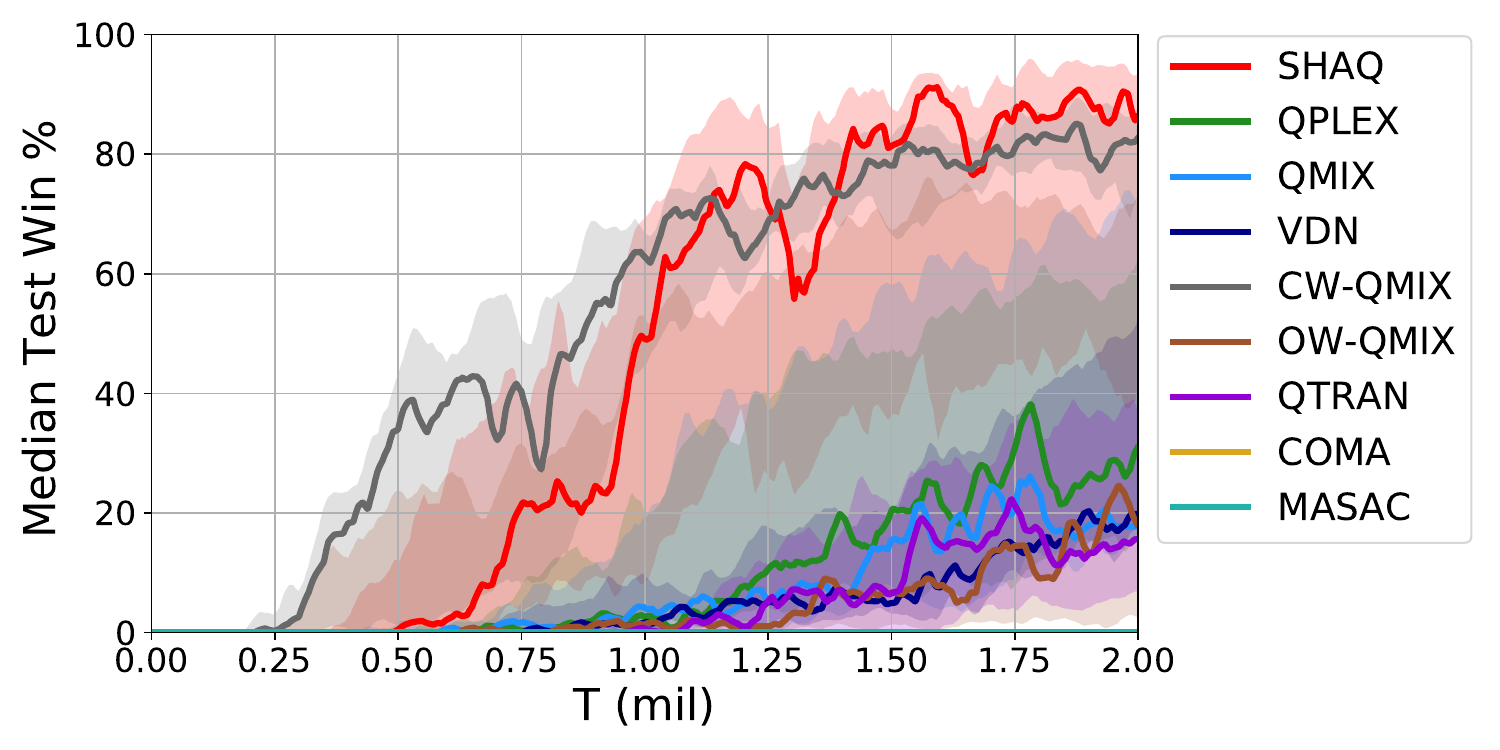}
                \caption{2c\_vs\_64zg.}
            \end{subfigure}
            ~
            \begin{subfigure}[b]{0.32\linewidth}
                \centering
                \includegraphics[width=\textwidth]{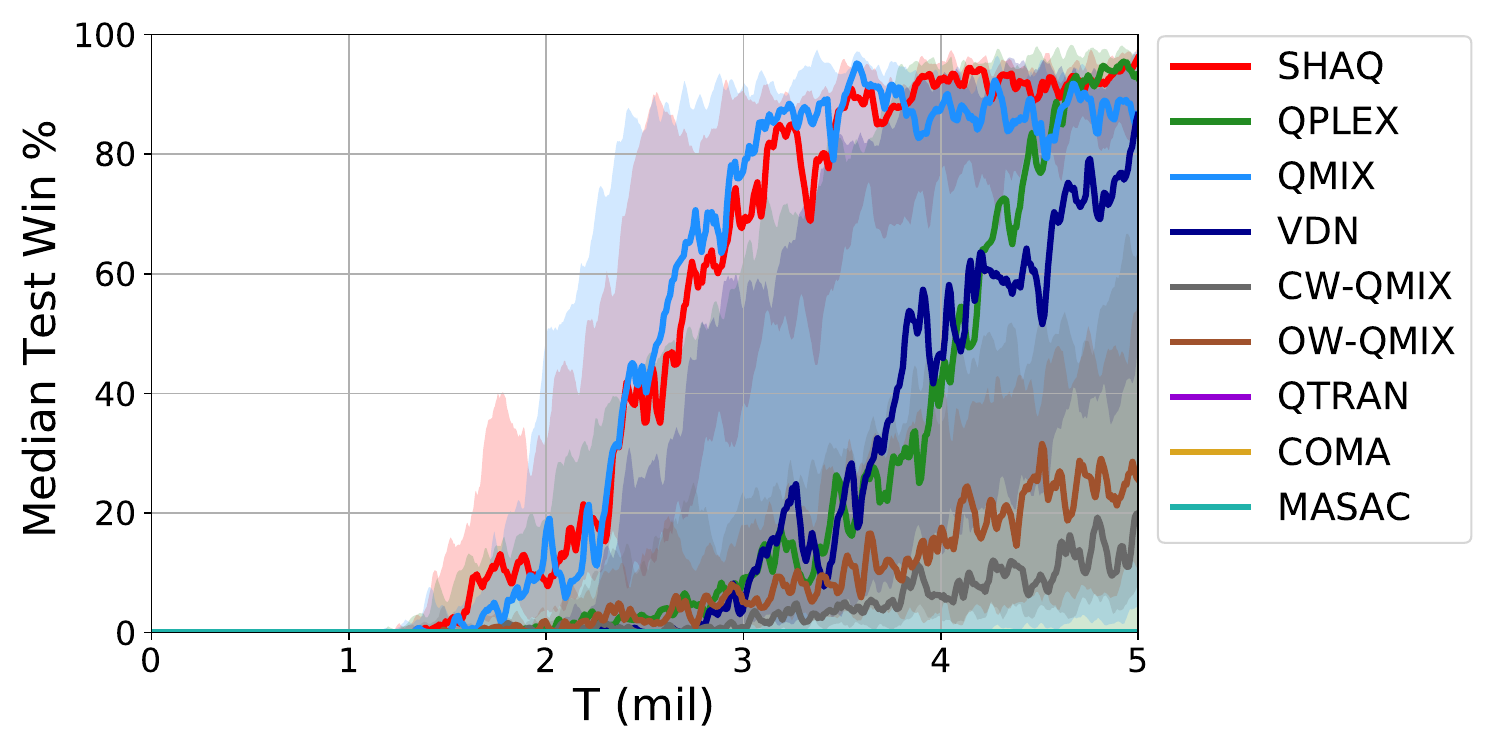}
                \caption{3s5z\_vs\_3s6z.}
            \end{subfigure}
            ~
            \begin{subfigure}[b]{0.32\linewidth}
                \centering
                \includegraphics[width=\textwidth]{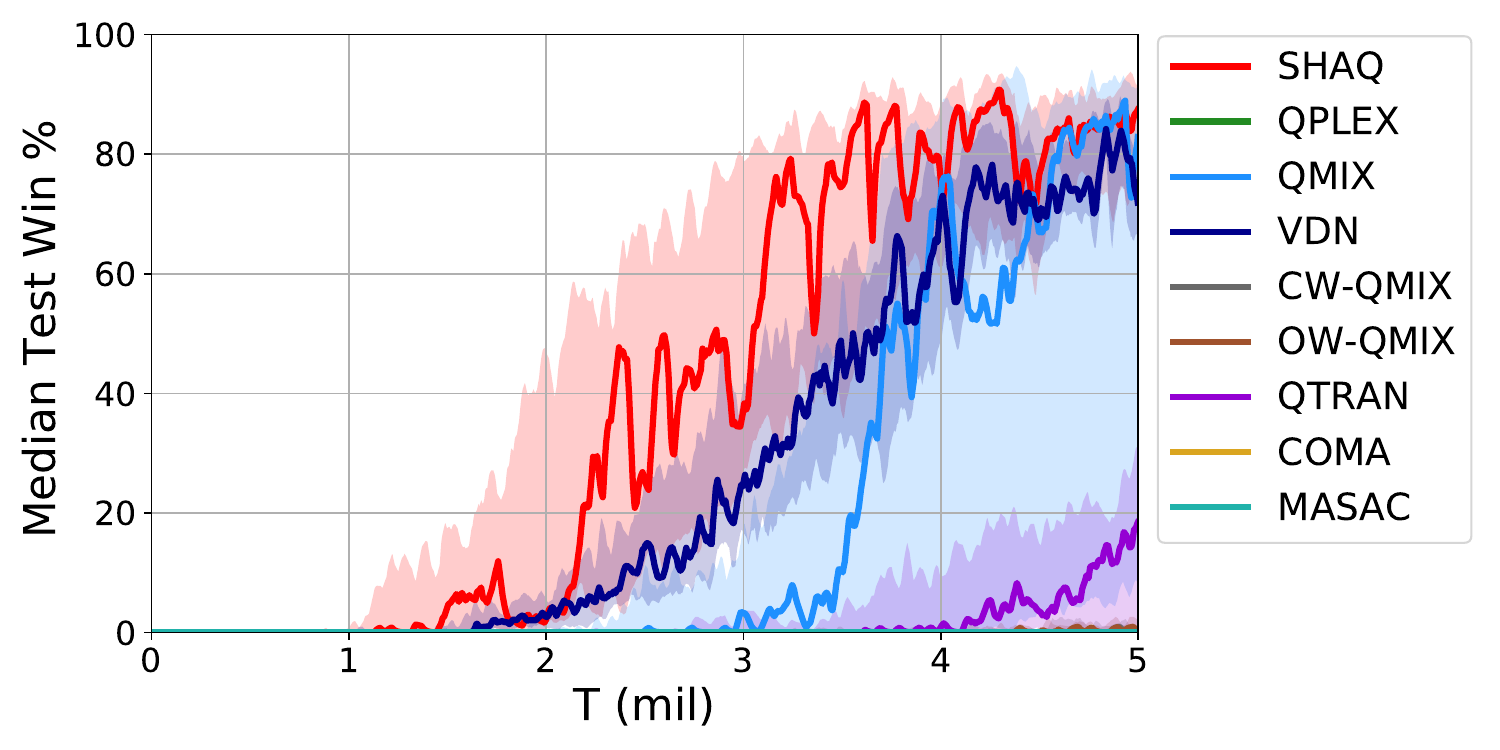}
                \caption{Corridor.}
            \end{subfigure}
            ~
            \begin{subfigure}[b]{0.32\linewidth}
                \centering
                \includegraphics[width=\textwidth]{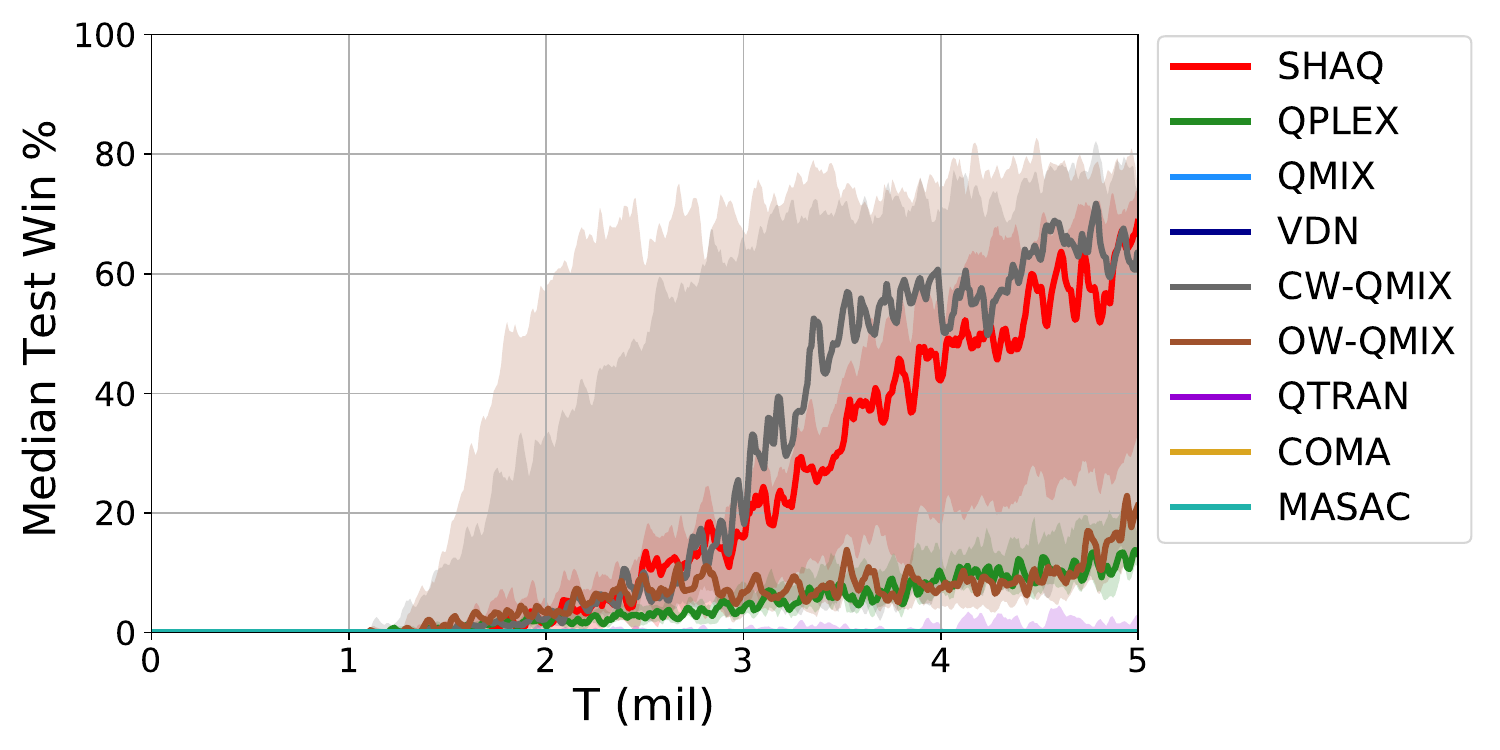}
                \caption{6h\_vs\_8z.}
            \end{subfigure}
            \caption{Median test win \% for hard (a-c), and super-hard (d-f) maps of SMAC.}
        \label{fig:easy_hard_smac}
        \end{figure*}
        
        \textbf{Interpretability of SHAQ.} To further show the interpretability of SHAQ, we also conduct a test on 3m (i.e. a simple task in SMAC). As seen from Figure \ref{fig:shaq_epsilon_greedy}, Agent 3 faces the direction opposite to enemies, meanwhile, the enemies are out of its attacking range. It can be understood as that Agent 3 does not contribute to the team and thus it is almost a dummy agent. Its MSQ is 0.84 (around 0) that correctly catch the manner of a dummy agent (verifying (i) in Proposition \ref{prop:shapley_value_properties}). In contrast, Agent 1 and Agent 2 are attacking enemies, while Agent 1 suffers from more attacks (with lower health) than Agent 2. As a result, Agent 1 contributes more than Agent 2 and therefore its MSQ is greater, which implies that the credits reflect agents' contributions (verifying (iii) in Proposition \ref{prop:shapley_value_properties}). On the other hand, we can see from Figure \ref{fig:shaq_greedy} that with the optimal policies all agents receive almost identical MSQs (verifying the theoretical results in Section \ref{subsec:definition_and_formulation}). The above results well verify the theoretical analysis that we deliver before. 
        
        \begin{figure*}[ht!]
            \centering
            \begin{subfigure}[b]{0.23\textwidth}
            	\centering
        	    \includegraphics[width=\textwidth]{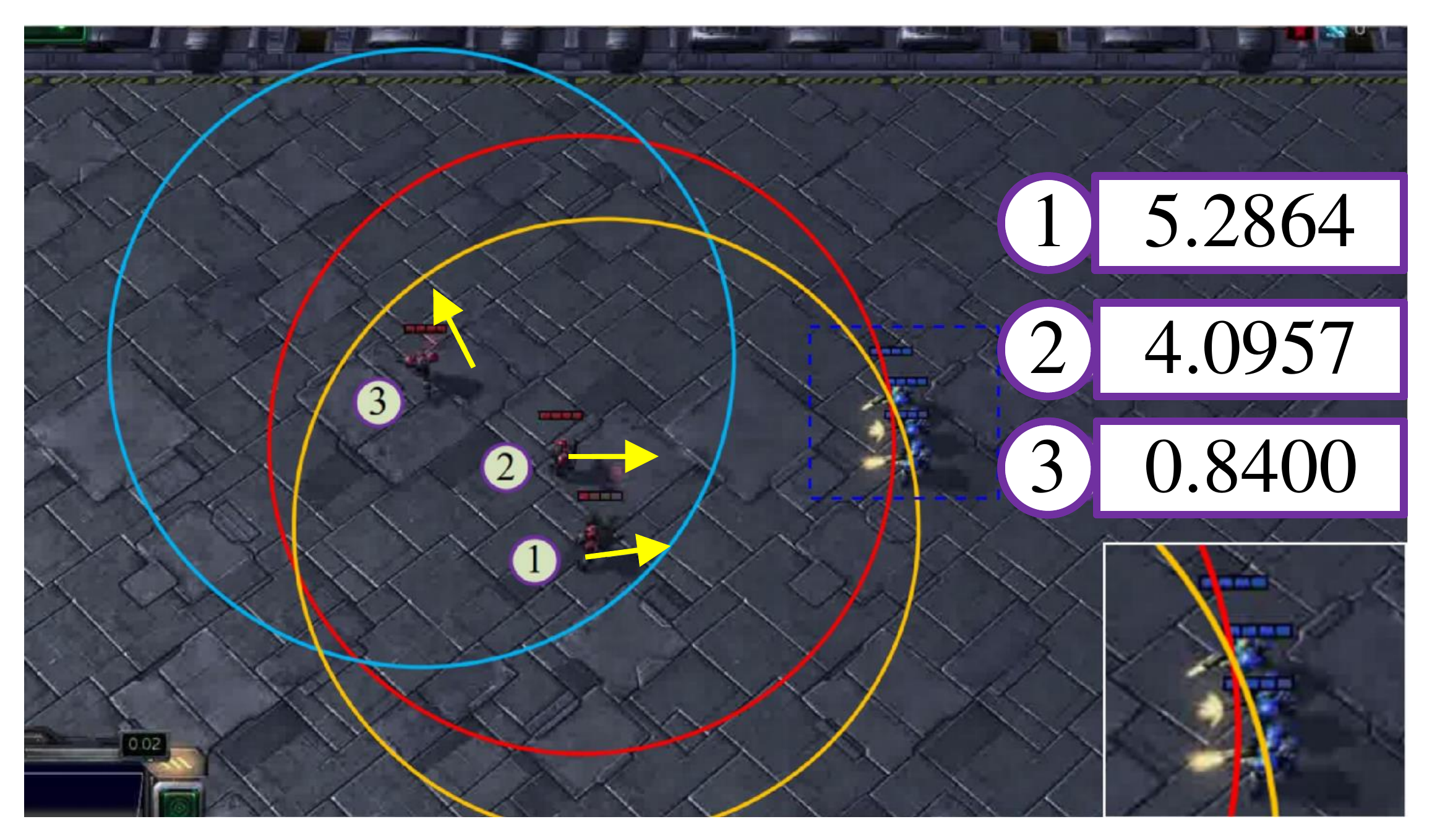}
        	    \caption{SHAQ: $\epsilon$-greedy.}
        	\label{fig:shaq_epsilon_greedy}
            \end{subfigure}
            ~
            \begin{subfigure}[b]{0.23\textwidth}
                \centering                \includegraphics[width=\textwidth]{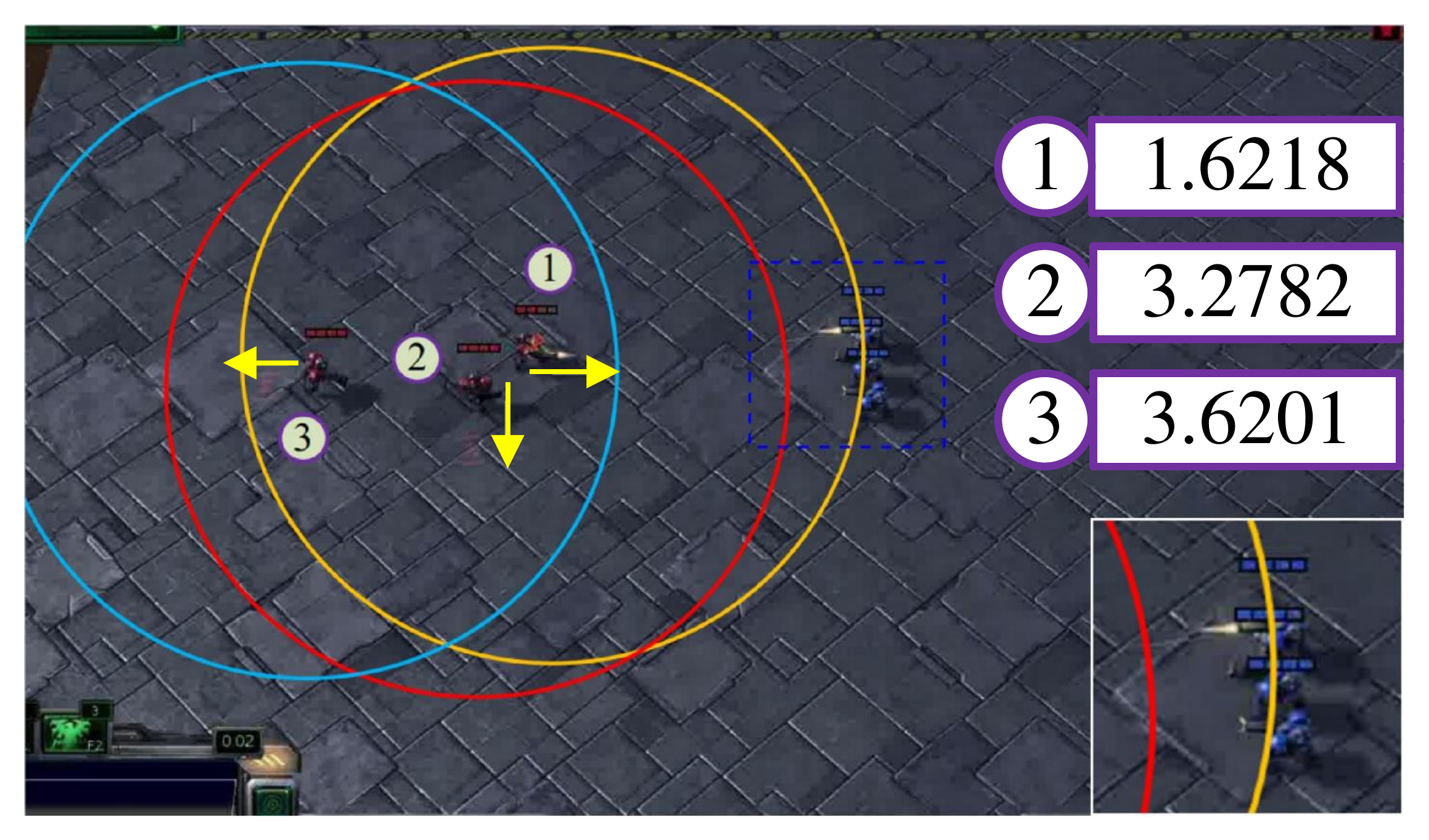}
                \caption{VDN: $\epsilon$-greedy.}
            \label{fig:vdn_epsilon_greedy}
            \end{subfigure}
            ~
            \begin{subfigure}[b]{0.23\textwidth}
                \centering                \includegraphics[width=\textwidth]{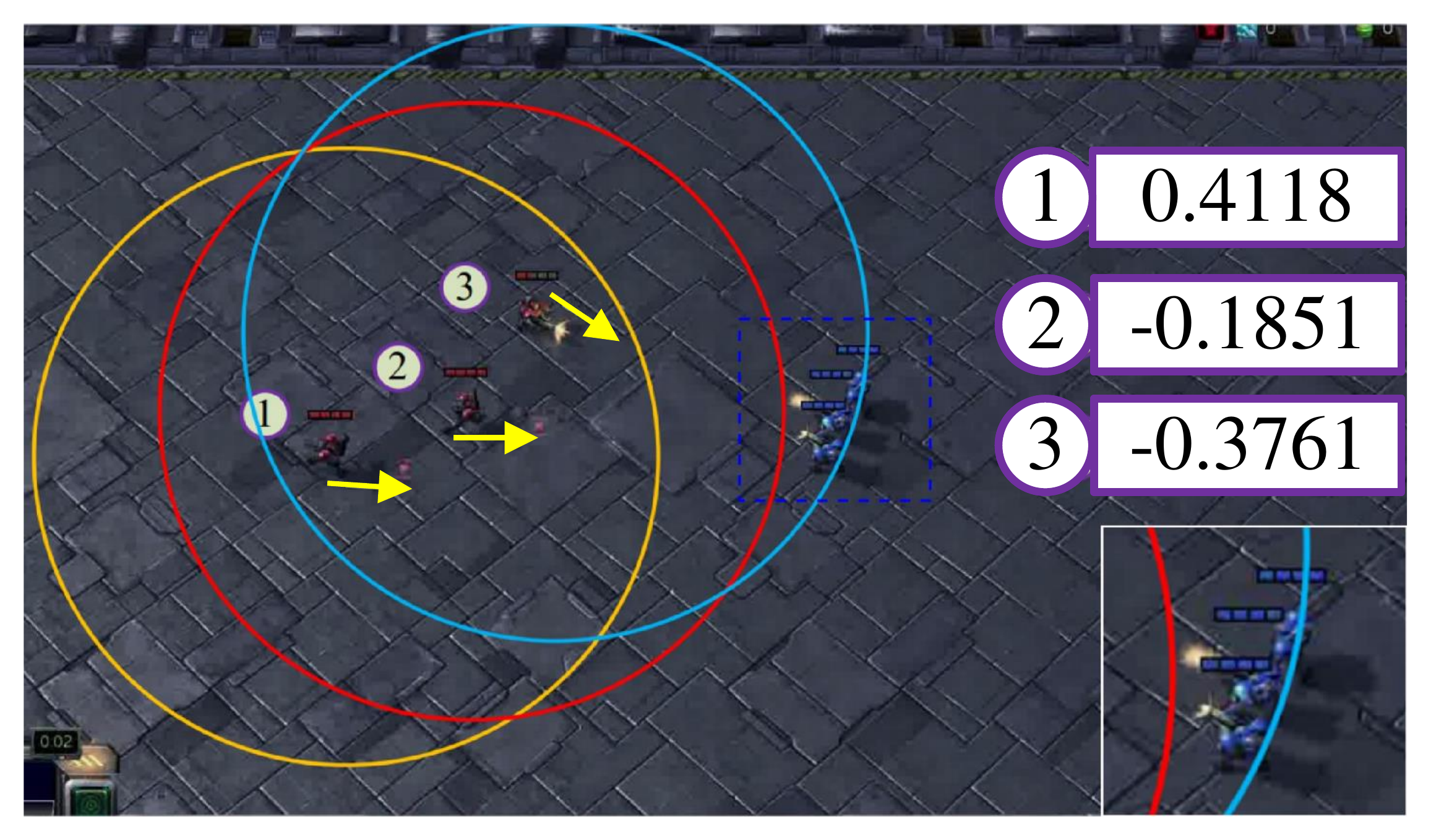}
                \caption{QMIX: $\epsilon$-greedy.}
            \label{fig:qmix_epsilon_greedy}
            \end{subfigure}
            ~
            \begin{subfigure}[b]{0.23\textwidth}
        	    \centering        	    \includegraphics[width=\textwidth]{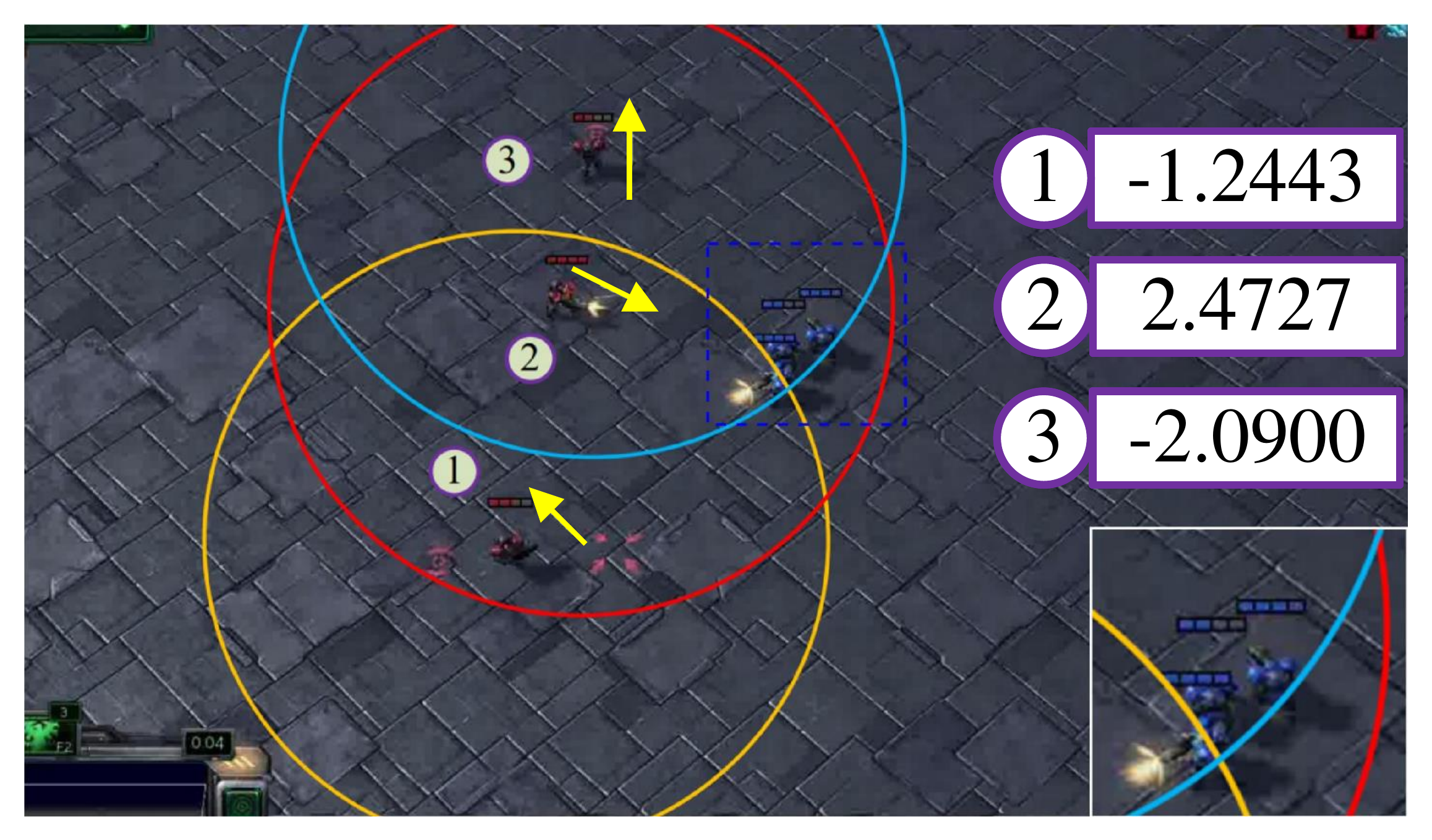}
        	    \caption{QPLEX: $\epsilon$-greedy.}
        	\label{fig:qplex_epsilon_greedy}
            \end{subfigure}
            
            \begin{subfigure}[b]{0.23\textwidth}
        	    \centering        	    \includegraphics[width=\textwidth]{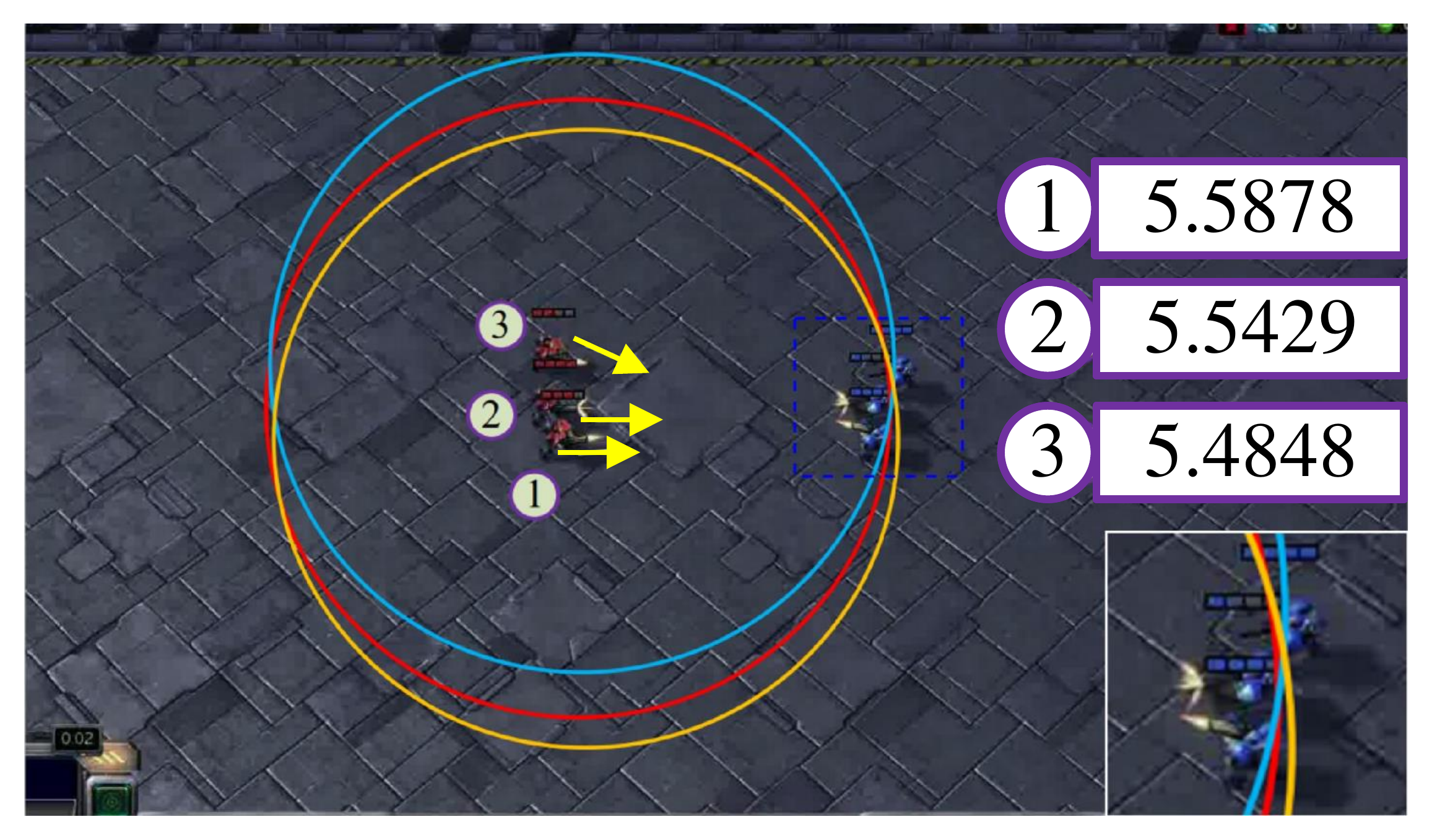}
        	    \caption{SHAQ: greedy.}
        	\label{fig:shaq_greedy}
            \end{subfigure}
            ~
            \begin{subfigure}[b]{0.23\textwidth}
            	\centering        	    \includegraphics[width=\textwidth]{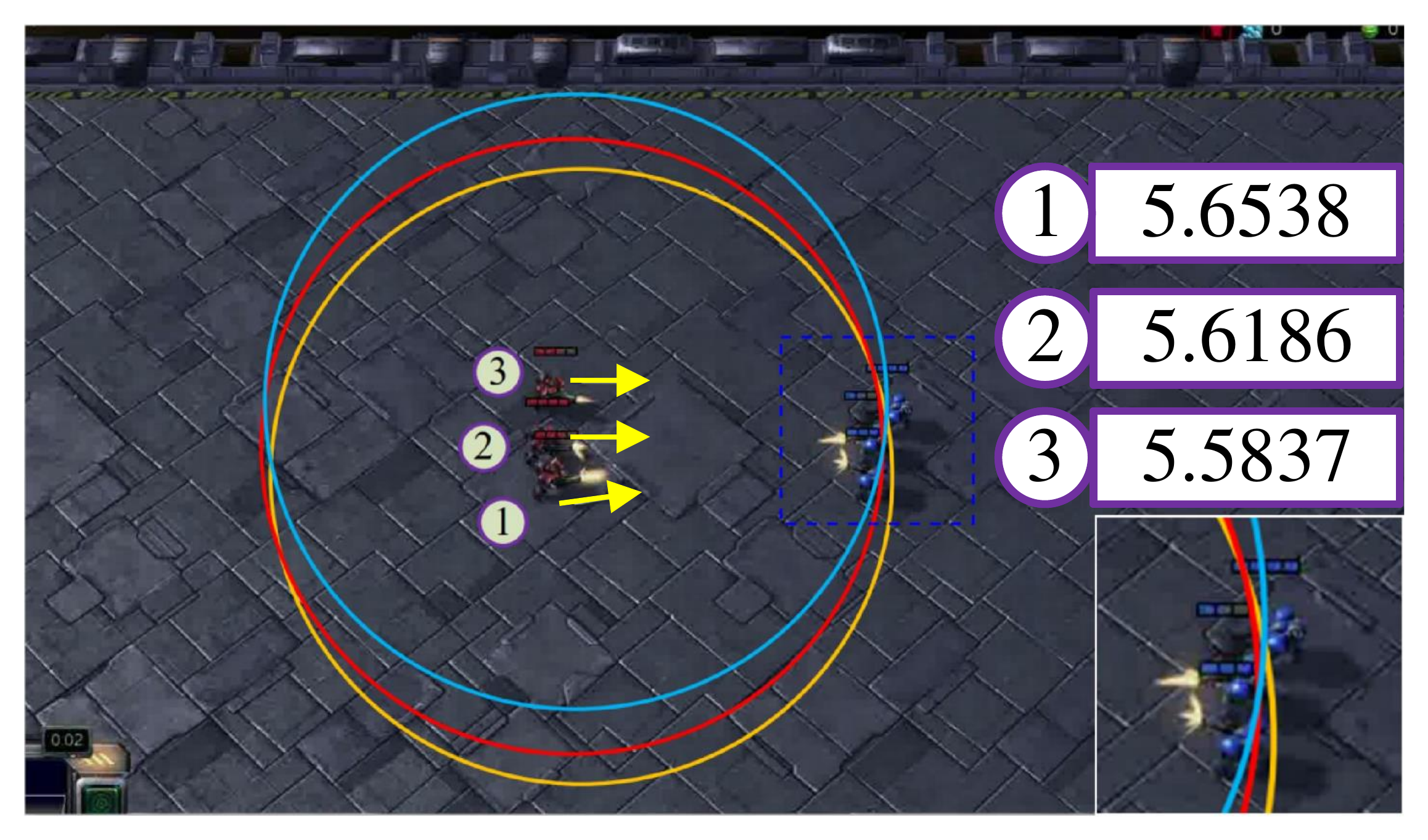}
        	    \caption{VDN: greedy.}
        	\label{fig:vdn_greedy}
            \end{subfigure}
            ~
            \begin{subfigure}[b]{0.23\textwidth}
        	    \centering        	    \includegraphics[width=\textwidth]{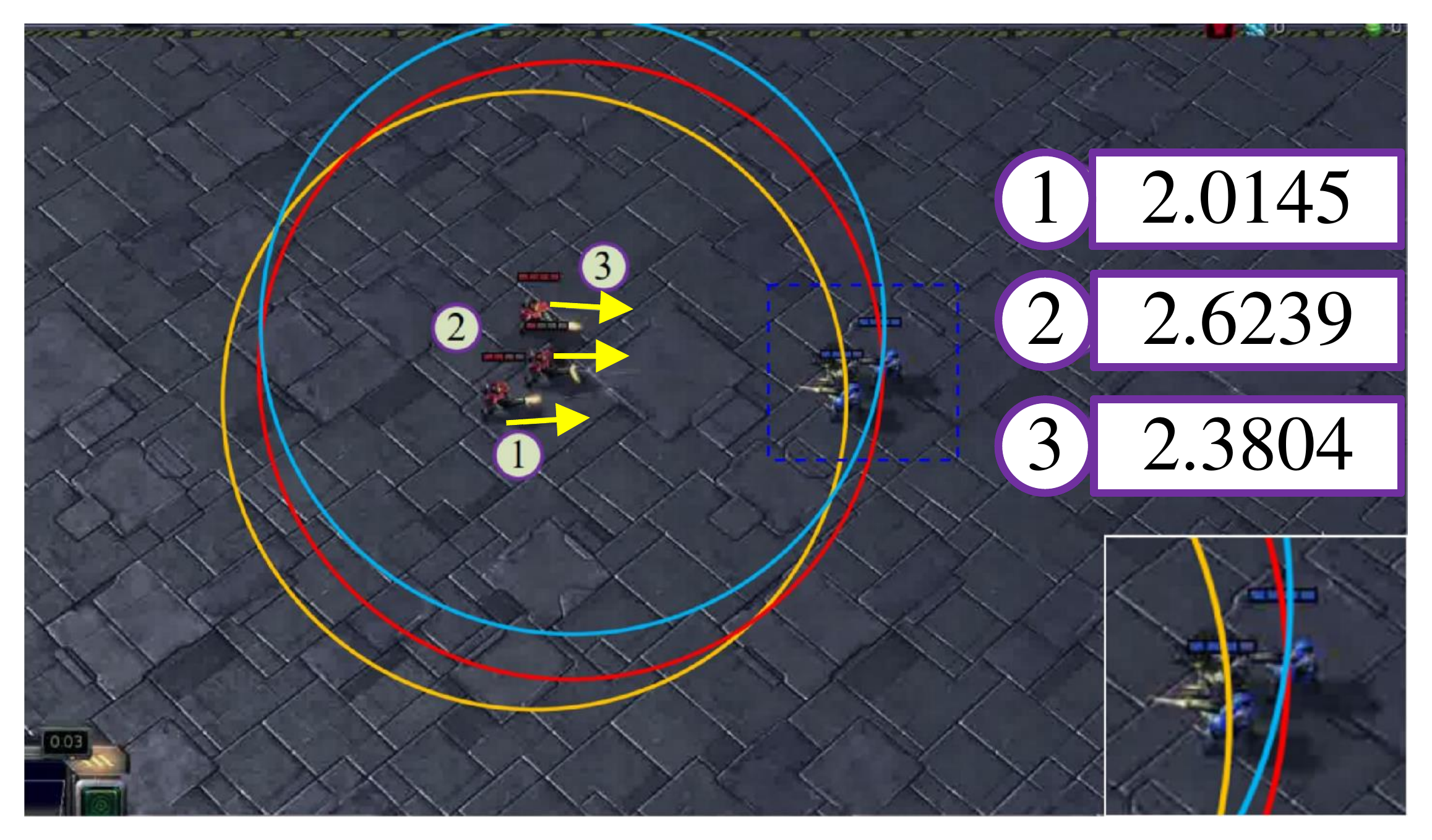}
        	    \caption{QMIX: greedy.}
        	\label{fig:qmix_greedy}
            \end{subfigure}
            ~
            \begin{subfigure}[b]{0.23\textwidth}
        	    \centering        	    \includegraphics[width=\textwidth]{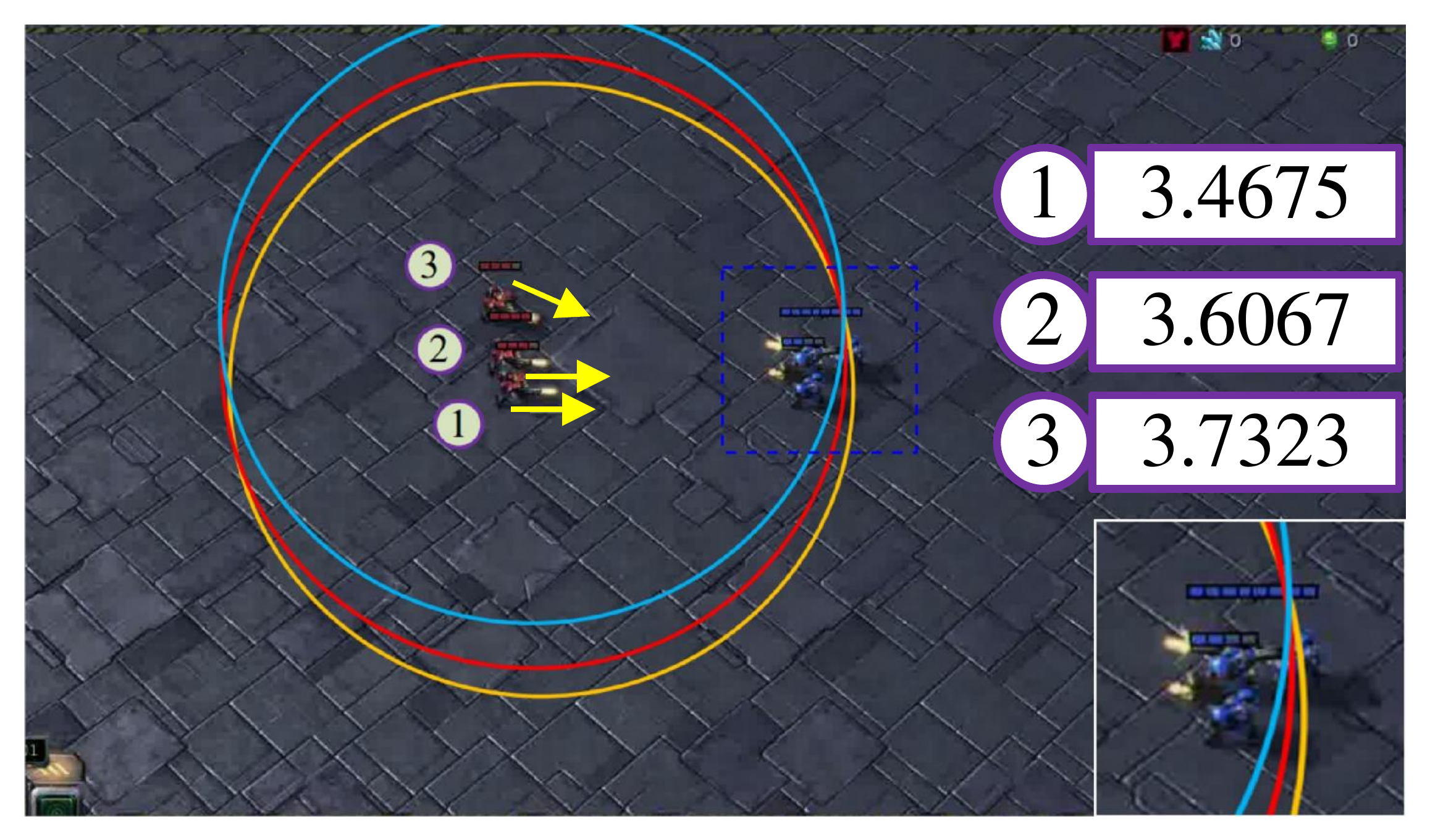}
        	    \caption{QPLEX: greedy.}
        	\label{fig:qplex_greedy}
            \end{subfigure}
            \caption{Visualisation of the test for SHAQ and baselines on 3m in SMAC: each colored circle is the centered attacking range of a controllable agent (in red), and each agent's factorised Q-value is reported on the right. We mark the direction that each agent face by an arrow for clearness.}
        \label{fig:study_on_factorised_q_values}
        \end{figure*}
        
        To justify that the MSQs learned by SHAQ are non-trivial, we also show the results of VDN, QMIX and QPLEX. It is surprising that the Q-values of these baselines are also almost identical among agents for the optimal actions (however, the property disappears in more complicated scenarios as shown in Appendix \ref{subsec:more_visualisation} while the property of SHAQ is still valid). Since VDN is a subclass of SHAQ and possesses the same form of loss function for optimal actions, it is reasonable that it obtains the similar results to SHAQ. As for the suboptimal actions, VDN does not possess an explicit interpretation as SHAQ due to the incorrect definition of $\delta_{i}(\mathbf{s}, a_{i})=1$ over suboptimal actions (verifying the statement in Section \ref{sec:related_work}). The values of QMIX and QPLEX are difficult to be explained.
        
\section{Conclusion}
\label{sec:conclusion}
    \textbf{Summary.} This paper generalises Shapley value to Markov convex game, called Markov Shapley value. Markov Shapley value inherits a number of properties: (i) identifiability of dummy agents; (ii) efficiency; (iii) reflecting the contribution and (iv) symmetry. Based on Property (ii), we derive Shapley-Bellman optimality equation, Shapley-Bellman operator and SHAQ. We prove that solving Shapley-Bellman optimality equation is equivalent to solving the Markov core (i.e., no agent has incentives to deviate from the grand coalition). Markov convex game with the grand coalition is equivalent to global reward game \citep{Wang_2020}, wherein Markov Shapley value plays the role of value factorisation. Since SHAQ is a stochastic approximation of Shapley-Bellman operator that is proved to solve Shapley-Bellman optimality equation, global reward game with value factorisation becomes valid standing by the cooperative game theoretical framework (i.e. solving Markov core). Property (i) and (iii) in Proposition \ref{prop:shapley_value_properties} are demonstrated in the experiments showing the interpretability of SHAQ.
    
    \textbf{Limitation and Future Work.} The value of Markov convex game is not limited to solving problems with the grand coalition, though in this paper we design SHAQ that only focuses on the scenario with the grand coalition. By removing the condition of supermodularity (see Eq.\ref{eq:mcg_assumption}), this framework can be used to study more general coalition games where different coalitions of agents as units may compete/cooperate with each other. Since the grand coalition and Markov Shapley value is not a solution in Markov core yet, the learning process becomes more complicated to converge to Markov core. A possible research direction in future is to investigate dynamically forming the coalition structure and conducting credit assignments simultaneously.
    
    
    
\section*{Acknowledgements}
    This work is sponsored by the Engineering and Physical Sciences Research Council of UK (EPSRC) under awards EP/S000909/1. Tae-Kyun Kim is partly sponsored by KAIA grant (22CTAP-C163793-02, MOLIT), NST grant (CRC 21011, MSIT), KOCCA grant (R2022020028, MCST) and the Samsung Display corporation. Yuan Zhang is sponsored by the European Union’s Horizon 2020 research and innovation program under the Marie Skłodowska-Curie grant agreement No. 953348 (ELO-X).
    
\bibliography{shapley_value}
\bibliographystyle{unsrtnat}

\section*{Checklist}


\begin{enumerate}

\item For all authors...
\begin{enumerate}
  \item Do the main claims made in the abstract and introduction accurately reflect the paper's contributions and scope?
    \answerYes{}
  \item Did you describe the limitations of your work?
    \answerYes{See Section 7.}
  \item Did you discuss any potential negative societal impacts of your work?
    \answerYes{See Appendix F.}
  \item Have you read the ethics review guidelines and ensured that your paper conforms to them?
    \answerYes{}
\end{enumerate}

\item If you are including theoretical results...
\begin{enumerate}
  \item Did you state the full set of assumptions of all theoretical results?
    \answerYes{See Appendix E.1.}
        \item Did you include complete proofs of all theoretical results?
    \answerYes{See Appendix E.}
\end{enumerate}

\item If you ran experiments...
\begin{enumerate}
  \item Did you include the code, data, and instructions needed to reproduce the main experimental results (either in the supplemental material or as a URL)?
    \answerYes{See supplementary material.}
  \item Did you specify all the training details (e.g., data splits, hyperparameters, how they were chosen)?
    \answerYes{See Appendix B.}
        \item Did you report error bars (e.g., with respect to the random seed after running experiments multiple times)?
    \answerYes{See Section 6.}
        \item Did you include the total amount of compute and the type of resources used (e.g., type of GPUs, internal cluster, or cloud provider)?
    \answerYes{See Appendix B.1.}
\end{enumerate}

\item If you are using existing assets (e.g., code, data, models) or curating/releasing new assets...
\begin{enumerate}
  \item If your work uses existing assets, did you cite the creators?
    \answerYes{}
  \item Did you mention the license of the assets?
    \answerNA{}
  \item Did you include any new assets either in the supplemental material or as a URL?
    \answerYes{}
  \item Did you discuss whether and how consent was obtained from people whose data you're using/curating?
    \answerNA{}
  \item Did you discuss whether the data you are using/curating contains personally identifiable information or offensive content?
    \answerNA{}
\end{enumerate}

\item If you used crowdsourcing or conducted research with human subjects...
\begin{enumerate}
  \item Did you include the full text of instructions given to participants and screenshots, if applicable?
    \answerNA{}
  \item Did you describe any potential participant risks, with links to Institutional Review Board (IRB) approvals, if applicable?
    \answerNA{}
  \item Did you include the estimated hourly wage paid to participants and the total amount spent on participant compensation?
    \answerNA{}
\end{enumerate}

\end{enumerate}

\newpage
\appendix
\onecolumn

\section{Algorithm of Shapley Q-learning}
\label{sec:algorithm_shapley_q}

    In this section, we present the pseudo code of Shapley Q-learning in Algorithm \ref{alg:shapley_q}. 
    The general paradigm can be divided into such parts: (1) collecting samples through $\epsilon$-greedy strategy and store the collected samples to a replay buffer for training; (2) sampling a batch of episodes of samples from the replay buffer; (3) calculating $\hat{Q}_{i}(\tau_{i}^{t+1}, a_{i}^{t+1}; \theta^{-})$, $\hat{\alpha}_{i}(\mathbf{s}^{k}, a_{i}^{k}; \lambda)$ and $\hat{Q}_{i}(\tau_{i}^{t}, a_{i}^{t}; \theta)$; and (4) constructing a loss of Shapley Q-learning and updating parameters to minimise the loss.
    \begin{algorithm}[ht!]
        \caption{Shapley Q-learning}
        \label{alg:shapley_q}
        \scalebox{0.8}{
        \begin{minipage}{1.0\linewidth}
        \begin{algorithmic}[1]
            \STATE Initialise a set of agents $\mathcal{N}$ and set $N = |\mathcal{N}|$
            \STATE Initialise $\hat{Q}_{i}(\tau_{i}, a_{i}; \theta)$ with the shared parameters among agents
            \STATE Initialise $\hat{\alpha}_{i}(\mathbf{s}, a_{i}; \lambda)$ with the shared parameters among agents
            \STATE Initialise $\hat{Q}_{i}(\tau_{i}, a_{i}; \theta^{-})$ by copying $\hat{Q}_{i}(\tau_{i}, a_{i}; \theta)$ with the shared parameters among agents
            \STATE Initialise a replay buffer $\mathcal{B}$
            \REPEAT
                \STATE Initialise a container $\mathcal{E}$ for storing an episode
                \STATE Observe an initial global state $\mathbf{s}^{1}$ and each agent's partial observation $\mathit{o}_{i}^{1}$ from an environment
                \FOR{t=1:T}
                    \STATE Get $\tau_{i}^{t} = (o_{i}^{m})_{m=1:t}$ for each agent
                    \STATE For each agent $\mathit{i}$, select an action 
                    \begin{equation*}
                        \mathit{a}_{i}^{t} = \begin{cases}
                                            \text{a random action} & \text{with probability }\epsilon \\
                                            \arg\max_{a_{i}} \hat{Q}^{*}_{i}(\tau_{i}^{t}, a_{i}; \theta) & \text{otherwise}
                                         \end{cases}
                    \end{equation*}
                    \STATE Execute $\mathit{a}_{i}^{t}$ of each agent to get the global reward $\mathit{R}^{t}$, $\mathbf{s}^{t+1}$ and each agent's $\mathit{o}_{i}^{t+1}$
                    \STATE Store $\big\langle \mathbf{s}^{t}, (o_{i}^{t})_{i=1:N}, (a_{i}^{t})_{i=1:N}, R^{t}, \mathbf{s}^{t+1}, (o_{i}^{t+1})_{i=1:N} \big\rangle$ to $\mathcal{E}$
                \ENDFOR
                \STATE Store $\mathcal{E}$ to $\mathcal{B}$
                \STATE Sample a batch of episodes with batch size B from $\mathcal{B}$
                    \FOR{each sampled episode}
                        \FOR{k=1:T}
                            \STATE Get each transition $\big\langle \mathbf{s}^{k}, (o_{i}^{k})_{i=1:N}, (a_{i}^{k})_{i=1:N}, R^{k}, \mathbf{s}^{k+1}, (o_{i}^{k+1})_{i=1:N} \big\rangle$
                            \STATE For each agent $\mathit{i}$, get $\tau_{i}^{k} = (o_{i}^{m})_{m=1:k}$
                            \STATE For each agent $\mathit{i}$, calculate $\hat{Q}_{i}(\tau_{i}^{k}, a_{i}^{k}; \theta)$
                            \STATE For each agent $\mathit{i}$, calculate $\alpha_{i}(\mathbf{s}^{k}, a_{i}^{k}; \lambda)$ by Algorithm \ref{alg:getting_alpha}
                            \STATE For each agent $\mathit{i}$, calculate $\delta_{i}(\mathbf{s}^{k}, a_{i}^{k}; \lambda)$ as follows:
                                \begin{equation*}
                                    \hat{\delta}_{i}(\mathbf{s}^{k}, a_{i}^{k}; \lambda) = \begin{cases} 
                                                                          1 & a_{i}^{k} = \arg\max_{a_{i}} \hat{Q}_{i}(\mathbf{s}^{k}, a_{i}; \theta) \\
                                                                          \hat{\alpha}_{i}(\mathbf{s}^{k}, a_{i}^{k}; \lambda) & a_{i}^{k} \neq \arg\max_{a_{i}} \hat{Q}_{i}(\mathbf{s}^{k}, a_{i}; \theta) \ \ (\text{via Algorithm \ref{alg:getting_alpha}})
                                                                     \end{cases}
                                \end{equation*}
                            \STATE For each agent $\mathit{i}$, get $\tau_{i}^{k+1} = (o_{i}^{m})_{m=1:k+1}$
                            \STATE For each agent $\mathit{i}$, get $\mathit{a}_{i}^{k+1}$ by $\arg\max_{a_{i}} \hat{Q}_{i}(\tau_{i}^{k+1}, a_{i}; \theta)$
                            \STATE For each agent $\mathit{i}$, calculate $\hat{Q}_{i}(\tau_{i}^{k+1}, a_{i}^{k+1}; \theta^{-})$
                        \ENDFOR
                    \ENDFOR
                    \STATE Construct a loss as follows:
                        \begin{equation*}
                            \begin{split}
                                \min_{\theta, \lambda} \frac{1}{B} \sum_{k=1}^{B} \Big[ \ \big( \ R^{k} \ + \ 
                                \gamma \sum_{i \in \mathcal{N}} \max_{a_{i}^{k}} \hat{Q}_{i}(\tau_{i}^{k+1}, a_{i}^{k+1}; \theta^{-}) - 
                                \sum_{i \in \mathcal{N}} \hat{\delta}_{i}(\mathbf{s}^{k}, a_{i}^{k}; \lambda) \ \hat{Q}_{i}(\tau_{i}^{k}, a_{i}^{k}; \theta) \ \big)^{2} \ \Big]
                            \end{split}
                        \end{equation*}
                    \STATE Update $\theta$ and $\lambda$ through the above loss
                    \STATE Periodically update $\theta^{-}$ by copying $\theta$
           \UNTIL{$\hat{Q}_{i}(\tau_{i}, a_{i}; \theta)$ converges}
        \end{algorithmic}
        \end{minipage}
        }
    \end{algorithm}
    
    \begin{algorithm}[ht!]
        \caption{Calculating $\hat{\alpha}_{i}(\mathbf{s}, a_{i})$}
        \label{alg:getting_alpha}
        \scalebox{0.8}{
        \begin{minipage}{1.0\linewidth}
        \begin{algorithmic}[1]
           \STATE {\bfseries Input:} $\mathbf{s}$, $\big( \hat{Q}_{i}(\tau_{i}, a_{i}; \theta) \big)_{i=1:N}$, $\mathit{M}$
           \STATE {\bfseries Output: $\big( \hat{\alpha}_{i}(\mathbf{s}, a_{i}) \big)_{i=1:N}$}
           \FOR{each agent $\mathit{i}$}
                \STATE Sample $\mathit{M}$ preceding coalitions $\mathcal{C}_{i}^{k} \sim \mathit{Pr}(\mathcal{C}_{i} | \mathcal{N} \backslash \{i\})$
                \FOR{k=1:M}
                    \STATE Get $\hat{Q}_{\mathcal{C}_{i}^{k}}(\tau_{\mathcal{C}_{i}^{k}}, \mathbf{a}_{\mathcal{C}_{i}^{k}}) = \frac{1}{|\mathcal{C}_{i}^{k}|}\sum_{j \in \mathcal{C}_{i}^{k}} \hat{Q}_{j}(\tau_{j}, a_{j})$
                \ENDFOR
                \STATE Get $\hat{\alpha}_{i}(\mathbf{s}, a_{i}) = \mathlarger{\frac{1}{M}}\mathlarger{\sum}_{k = 1}^{M} \mathlarger{F}_{\mathbf{s}} \Big( \hat{Q}_{\mathcal{C}_{i}^{k}}(\tau_{\mathcal{C}_{i}^{k}}, \mathbf{a}_{\mathcal{C}_{i}^{k}}), \ \hat{Q}_{i}(\tau_{i}, a_{i}) \Big) + 1$
            \ENDFOR
        \end{algorithmic}
        \end{minipage}
        }
    \end{algorithm}
    
    \textbf{Implementation of Sampling from $\mathit{Pr}(\mathcal{C}_{i} | \mathcal{N} \backslash \{i\})$ (Line 4 in Algorithm \ref{alg:getting_alpha}). } As introduced in Remark \ref{rmk:coalition_generation}, the analytic form of  $\mathit{Pr}(\mathcal{C}_{i} | \mathcal{N} \backslash \{i\})$ is $\frac{|\mathcal{C}_{i}|!(|\mathcal{N}|-|\mathcal{C}_{i}|-1)!}{|\mathcal{N}|!}$ that is actually the occurrence frequency of correlated coalition $\mathcal{C}_{i}$. Since each coalition is formed by different permutations, it can be instead sampled from permutations directly with uniform distribution where $\frac{1}{|\mathcal{N}|!}$ is as the probability distribution over each permutation. It is not difficult to find that these two sampling strategy induce the same probability distribution for obtaining $\mathcal{C}_{i}$, so they are equivalent. In practice, we sample multiple permutations (saying $M$) from the uniform distribution in parallel. From each sampled permutation, we extract the the relevant $\mathcal{C}_{i}$ for each agent $i$. Afterwards, to each agent $i$, $M$ coalitions are obtained to calculate $\hat{\alpha}_{i}(\mathbf{s}, a_{i})$.

\section{Experimental Setups}
\label{sec:experimental_setups}

    \subsection{Implementation Details of Shapley Q-learning}
    \label{subsec:implementation_details_shapley_q_learning}
    
        We now provide the additional implementation details that are omitted from the main part of paper. First, $\mathit{F}_{s}(\cdot, \cdot)$ is a 3-layer network (consecutively with two affine transformation and an activation of absolute), where the hidden-layer dimension is 32. The parameters of each affine transformation are generated by hyper-networks \citep{ha2017hyper} with input as the global state, whose details are shown in Table \ref{tab:SHAQ_hypernet}. The architecture of each agent's Q-value is a RNN with GRUs cell \citep{chung2014empirical}, whose hidden-layer dimension is 64. The input dimension is state dimension and the output dimension is action dimension.
        
        \begin{table}[ht!]
        	\caption{Table of specifications for $\mathit{F}_{s}(\cdot, \cdot)$.}
        	\begin{center}
        		\begin{small}
        			\begin{sc}
        			  \scalebox{0.95}{
        				\begin{tabular}{ll}
        					\toprule
        					\textbf{Network} & \textbf{Structure} \\
        					\midrule
        					1st weight matrix & [ linear(state\_dim, 64), ReLU, linear(64, 32*2), absolute ]\\
        					1st bias & [ linear(state\_dim, 64) ]\\
        				    2nd weight matrix & [ linear(state\_dim, 64), ReLU, linear(64, 32), absolute ]\\
        				    2nd bias & [ linear(state\_dim, 32), ReLU, linear(32, 1) ] \\
        					\bottomrule 
        				\end{tabular}
        				}
        			\end{sc}
        		\end{small}
        	\end{center}
        \label{tab:SHAQ_hypernet}
        \end{table}
        
        Taking the lessons of training two coupling modules from GANs \citep{goodfellow2014generative}, we take separate learning rates for $\hat{\alpha}_{i}(\mathbf{s}, a_{i})$ and $\hat{Q}_{i}(\mathbf{s}, a_{i})$. The learning rate for $\hat{Q}_{i}(\mathbf{s}, a_{i})$ is fixed at 0.0005 for all tasks. Nevertheless, the learning rate for $\hat{\alpha}_{i}(\mathbf{s}, a_{i})$ is dependent on the number of controllable agents. We use RMSProp optimizer for training in all tasks. All models are implemented in PyTorch 1.4.0 and each experiment is run on Nvidia GeForce RTX 2080Ti for 4 to 26 hours with a single process of environment.
    
    \subsection{Hyperparameters of Baselines}
    \label{subsec:hyperparameters_of_baselines}
        
        The hyperparameters of all baselines except for SQDDPG \citep{Wang_2020} are consistent with \citet{rashid2020weighted} and \citet{wang2020qplex}. The hyperparamers of SQDDPG are shown as follows: (1) The policy network is consistent with the other baselines, while the critic network is with 3 hidden layers and each layer is with 64 neurons. (2) The policy network is updated every 2 time steps, while the critic network is updated each time step. (3) The multiplier of the entropy of policy is 0.005. The rest of settings are identical with other baselines.
        
    \subsection{Predator-Prey for Modelling Relative Overgeneralisation}
    \label{subsec:predator_prey}
    
        We give the experimental setups of Predator-Prey \citep{bohmer2020deep} in Table \ref{tab:prey_and_predator_hyperparameters}.
        
        \begin{table}[ht!]
        	\caption{Table of experimental setups of Predator-Prey.}
        	\begin{center}
        		\begin{small}
        			\begin{sc}
        			  \scalebox{0.9}{
        				\begin{tabular}{lcll}
        					\toprule
        					\textbf{Hyperparameters} & \textbf{Value} & \textbf{Description}  \\
        					\midrule
        					batch size & 32 & The number of episodes for each update\\
        					discount factor $\gamma$ & 0.99 & The importance of future rewards  \\
        					replay buffer size & 5,000 & The maximum number of episodes to store in memory\\
        					episode length & 200 & Maximum time steps per episode \\
        					test episode & 16 & The number of episodes for evaluating the performance   \\
        					test interval  & 10,000 & The time step frequency for evaluating the performance \\
        					epsilon start & 1.0 & The start epsilon $\epsilon$ value for exploration \\
        					epsilon finish & 0.05 & The final epsilon $\epsilon$ value for exploration \\
        					exploration step & 1,000,000 & The number of steps for linearly annealing $\epsilon$  \\
        					max training step & 1,000,000 & The number of training steps \\
        				    target update interval & 200  &  The update frequency for target network \\
        					learning rate  & 0.0001 & The learning rate for $\delta_{i}(\mathbf{s}, a_{i})$  \\
        					$\alpha$ for W-QMIX variants & 0.1 & The weight for CW-QMIX and OW-QMIX  \\
        					sample size & 10 & The sample size for coalition sampling   \\
        					\bottomrule 
        				\end{tabular}
        				}
        			\end{sc}
        		\end{small}
        	\end{center}
        	\label{tab:prey_and_predator_hyperparameters}
        \end{table}
        
    \subsection{StarCraft Multi-Agent Challenge}
    \label{subsec:starcraft_multiagent_benchmark_tasks}
        The StarCraft Multi-Agent Challenge (SMAC) \citep{samvelyan2019starcraft} is a popular testbed for multi-agent reinforcement learning (MARL) algorithms. The main difficulties are (1) challenging dynamics, (2) partial observability and (3) high-dimensional observation space. During training, both the global state of the environment and each agent's local observation are able to be obtained; however, during execution, only each agent's local observation can be observed. For this reason, SMAC fits the centralised training and decentralised execution (CTDE) paradigm. In each micromanagement task, the ally units are controlled by agents and the enemy units are controlled by the built-in game AI. The agents need to learn a strategy to solve some challenging combat scenarios and defeat their opponents with maximum win rate.
        
        In this paper, we evaluate the proposed SHAQ on 11 typical combat scenarios in SMAC that can be classified into three categories: easy (8m, 3s5z, 1c3s5z and 10m\_vs\_11m), hard (5m\_vs\_6m, 3s\_vs\_5z and 2c\_vs\_64zg), and super-hard (3s5z\_vs\_3s6z, Corridor, MMM2 and 6h\_vs\_8z). More details of these tasks are provided in Table \ref{tab:smac_benchmarks}. The specific experimental setups for SMAC are shown in Table \ref{tab:smac_hyperparameters} and \ref{tab:relation_learning_rate_and_agents}.
    
        \begin{table}[t]
        	\caption{Introduction of maps and characters in SMAC.}
        	\begin{center}
        		\begin{small}
        			\begin{sc}
        			 \scalebox{0.8}{
        				\begin{tabular}{cllc}
        					\toprule
        					\textbf{Map Name} & \textbf{Ally Units} & \textbf{Enemy Units} & \textbf{Categories}  \\
        					\midrule
        					3s5z    &   3 Stalkers $\&$ 5 Zealots  &  3 Stalkers $\&$ 5 Zealots    &   easy \\
        					1c3s5z  &   1 Colossi $\&$ 3 Stalkers $\&$ 5 Zealots &    1 Colossi $\&$ 3 Stalkers $\&$ 5 Zealots   &   easy  \\
        					8m               &   8 Marines    & 8 Marines    & easy \\
        					10m\_vs\_11m     &   10 Marines   & 11 Marines   &  easy \\
        					5m\_vs\_6m       &   5 Marines  &  6 Marines   &   hard \\
        					3s\_vs\_5z       &   3 Stalkers   & 5 Zealots    &  hard\\
        					2c\_vs\_64zg     &   2 Colossi   &  64 Zerglings  & hard\\
        					3s5z\_vs\_3s6z   &  3 Stalkers $\&$ 5 Zealots  &  3 Stalkers $\&$ 6 Zealots   & super-hard \\
        					MMM2 & 1 Medivac, 2 Marauders $\&$  7 Marines     & 1 Medivac, 3 Marauders $\&$  8 Marines    &  super-hard  \\
        					6h\_vs\_8z &  6 Hydralisks    &  8 Zerglings   & super-hard  \\
        					Corridor & 6 Zealots & 24 Zerglings & super-hard \\
        					\bottomrule
        				\end{tabular}
        				}
        			\end{sc}
        		\end{small}
        	\end{center}
        	\label{tab:smac_benchmarks}
        \end{table}
    
        \begin{table}[ht!]
        	\caption{Table of experimental setups for SMAC.}
        	\begin{center}
        		\begin{small}
        			\begin{sc}
        			    \scalebox{0.7}{
        				\begin{tabular}{lcccl}
        					\toprule
        					\textbf{Hyperparameters} & \textbf{Easy} &\textbf{Hard} &\textbf{Super Hard} & \textbf{Description}  \\
        					\midrule
        					    batch size &    32   &   32   &   32   &   The number of episodes for each update \\
        						discount factor $\gamma$ &  0.99    &    0.99    &  0.99    &   The importance of future rewards  \\
        						replay buffer size & 5,000 & 5,000 & 5,000 & The maximum number of episodes to store in memory\\
        					    max training step  & 2,000,000    & 2,000,000  &  5,000,000    & The number of training steps\\
        					    test episode &   32   &   32   &   32  &  The number of episodes for evaluation  \\
        					    test interval  & 10,000 & 10,000 & 10,000  & The time step frequency for evaluating the performance \\
        					    epsilon start & 1.0  &  1.0  &  1.0  &  The start epsilon $\epsilon$ value for exploration \\
        					    epsilon finish  &  0.05  &   0.05   &    0.05   &   The final epsilon $\epsilon$  value for exploration \\
        					    exploration step &  50,000   &  50,000  &   1,000,000   & The number of steps for linearly annealing $\epsilon$ \\
        					    target update interval &    200   &   200   &   200   &  The update frequency for target network \\ 
        					  	$\alpha$ for OW-QMIX  &    0.5   &   0.5    &   0.5   &   The weight for OW-QMIX\\
        					  	$\alpha$ for CW-QMIX  &    0.75  &   0.75   &   0.75  &   The weight for CW-QMIX\\
        					  	sample size           &    10    &    10    &   10    &   The sample size for coalition sampling   \\
        					\bottomrule
        				\end{tabular}
        				}
        			\end{sc}
        		\end{small}
        	\end{center}
        	\label{tab:smac_hyperparameters}
        \end{table}
        
        \begin{table}[ht]
        	\caption{The learning rate for training $\hat{\alpha}_{i}(\mathbf{s}, a_{i})$ of SHAQ for various maps in SMAC.}
        	\begin{center}
        		\begin{small}
        			\begin{sc}
        				\scalebox{0.95}{
        					\begin{tabular}{ccc}
        						\toprule
        						\textbf{Map Name} & \textbf{Number of Agents} & \textbf{Learning Rate for $\hat{\alpha}_{i}(\mathbf{s}, a_{i})$} \\
        						\midrule
        						    2c\_vs\_64zg	&  2   & 0.002  \\
        						    3s\_vs\_5z 	    &  3   & 0.001  \\
        							5m\_vs\_6m		&  5   & 0.0005 \\
        							6h\_vs\_8z		&  6   & 0.0005 \\
        							Corridor        &  6   & 0.0005 \\
        							8m              &  8   & 0.0003 \\
        							3s5z			&  8   & 0.0003 \\
        							3s5z\_vs\_3s6z	&  8   & 0.0003 \\
        						    1c3s5z			&  9   & 0.0002 \\
        							10m\_vs\_11m    & 10   & 0.0001 \\
        						    MMM2            & 10   & 0.0001 \\
        						\bottomrule
        					\end{tabular}
        				}
        			\end{sc}
        		\end{small}
        	\end{center}
        	\label{tab:relation_learning_rate_and_agents}
        \end{table}

\section{Extra Experimental Results}
\label{subsec:extra_experimental_results}
    
    \subsection{Experimental Results on Extra SMAC Maps}
    \label{subsec:experimental_results_on_extra_smac_maps}
        To thoroughly compare the performance of SHAQ with baselines, we also run experiments on 5 extra maps in SMAC as Figure \ref{fig:extra_results_smac} shows. 8m, 3s5z, 1c3s5z and 10m\_vs\_11m are an easy maps and MMM2 is a super-hard map. The strategy of epsilon annealing is consistent with the previous experiments for SMAC. It is obvious that SHAQ also performs generally well on these 5 maps.
        \begin{figure*}[ht!]
            \centering
            \begin{subfigure}[b]{0.32\linewidth}
                \centering
                \includegraphics[width=\textwidth]{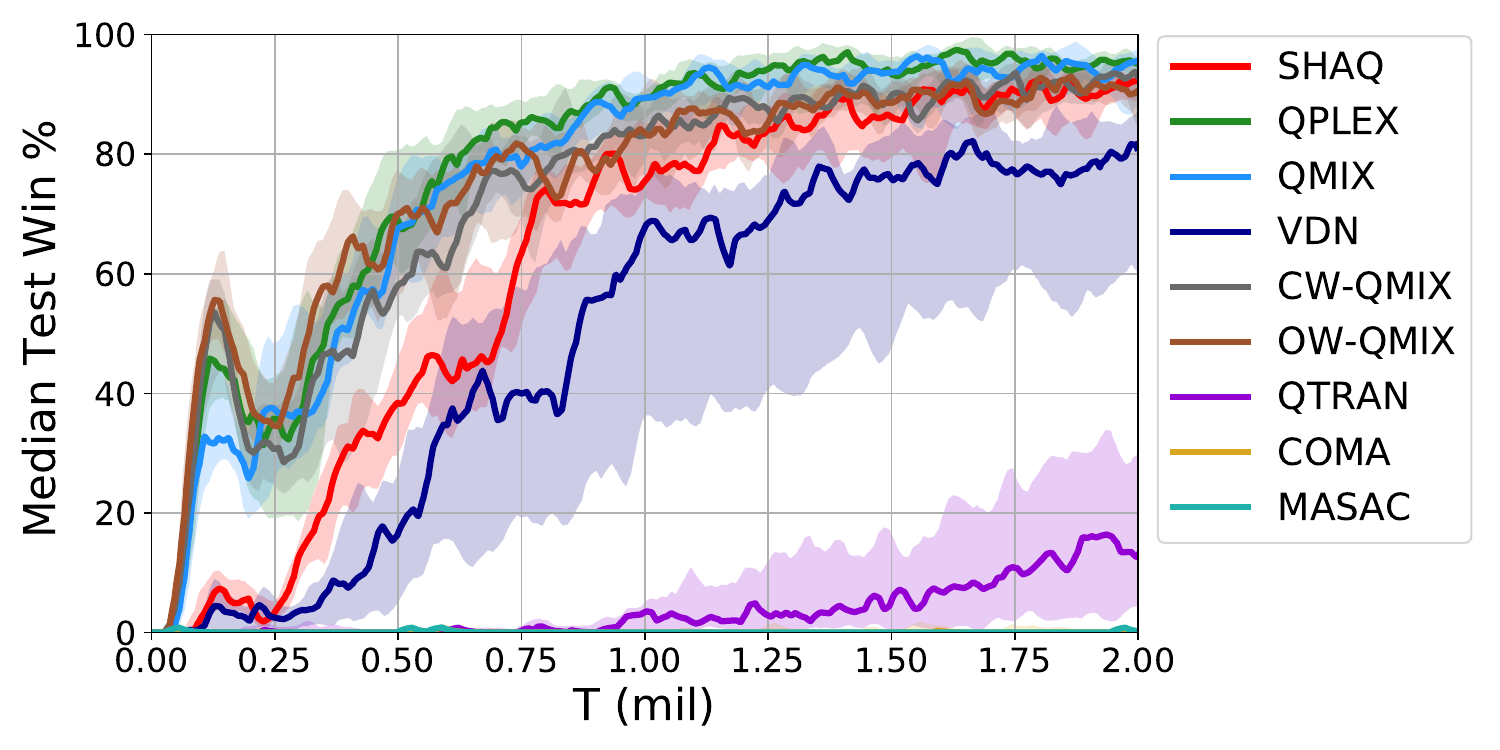}
                \caption{3s5z.}
            \end{subfigure}
            ~
            \begin{subfigure}[b]{0.32\linewidth}
                \centering
                \includegraphics[width=\textwidth]{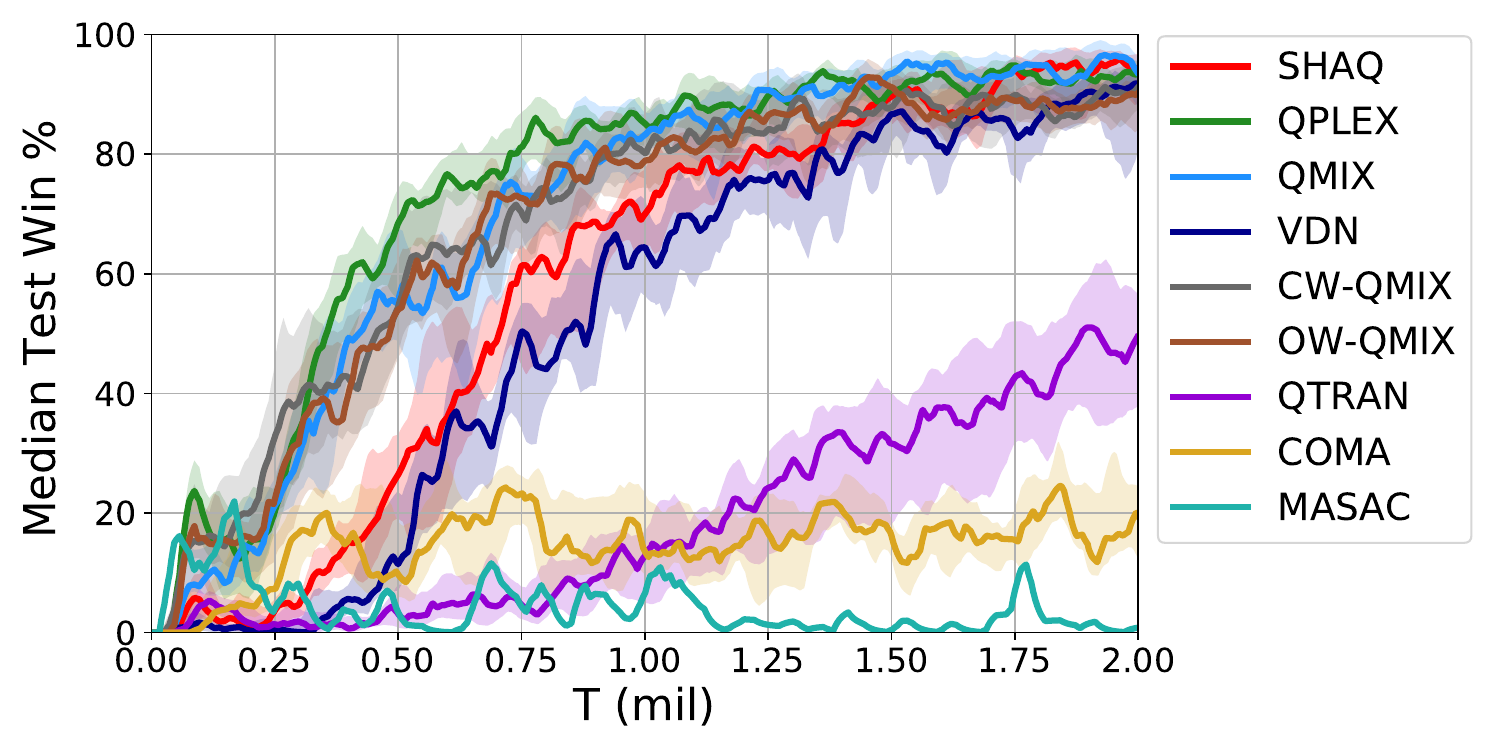}
                \caption{1c3s5z.}
            \end{subfigure}
            ~
            \begin{subfigure}[b]{0.32\linewidth}
                \centering
                \includegraphics[width=\textwidth]{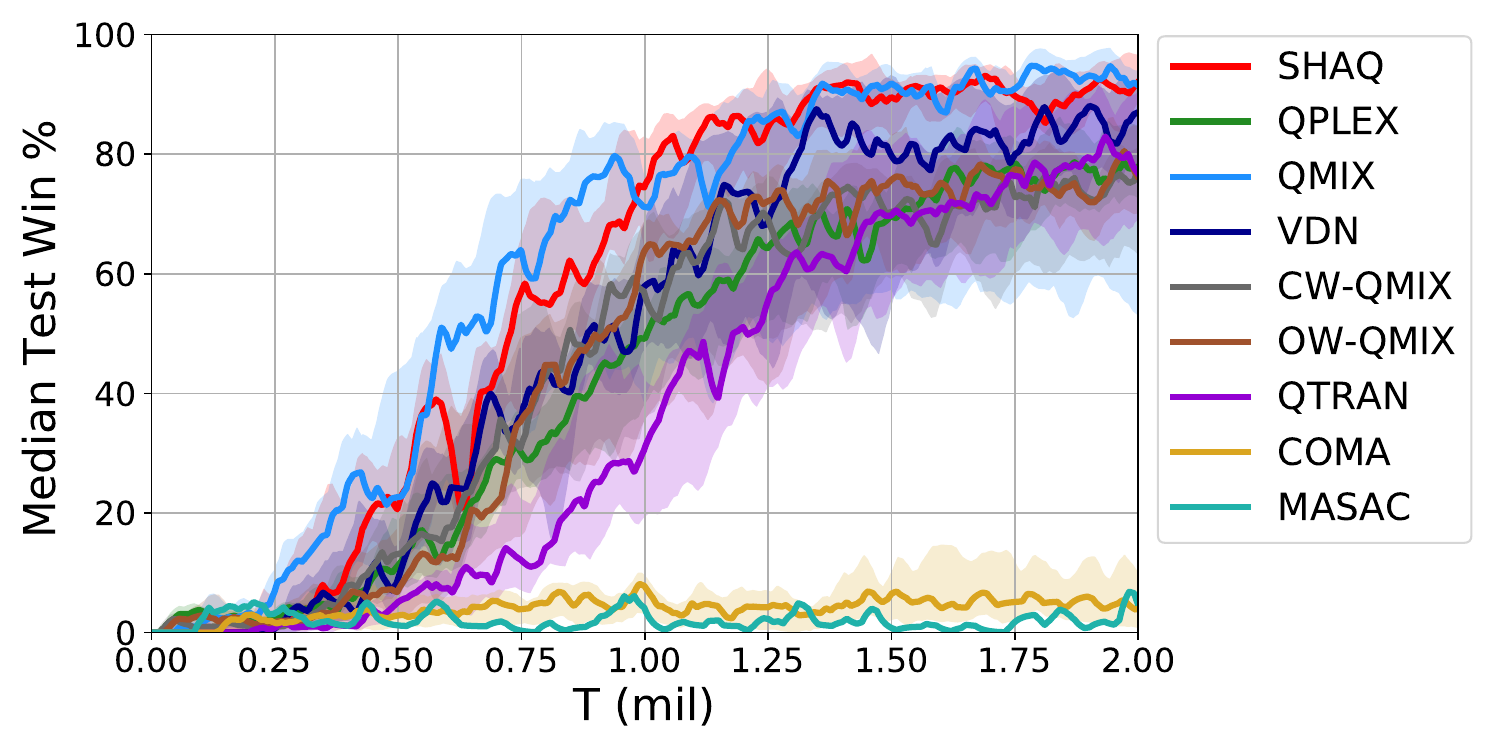}
                \caption{10m\_vs\_11m.}
            \end{subfigure}
            ~
            \begin{subfigure}[b]{0.32\textwidth}
                \centering                \includegraphics[width=\textwidth]{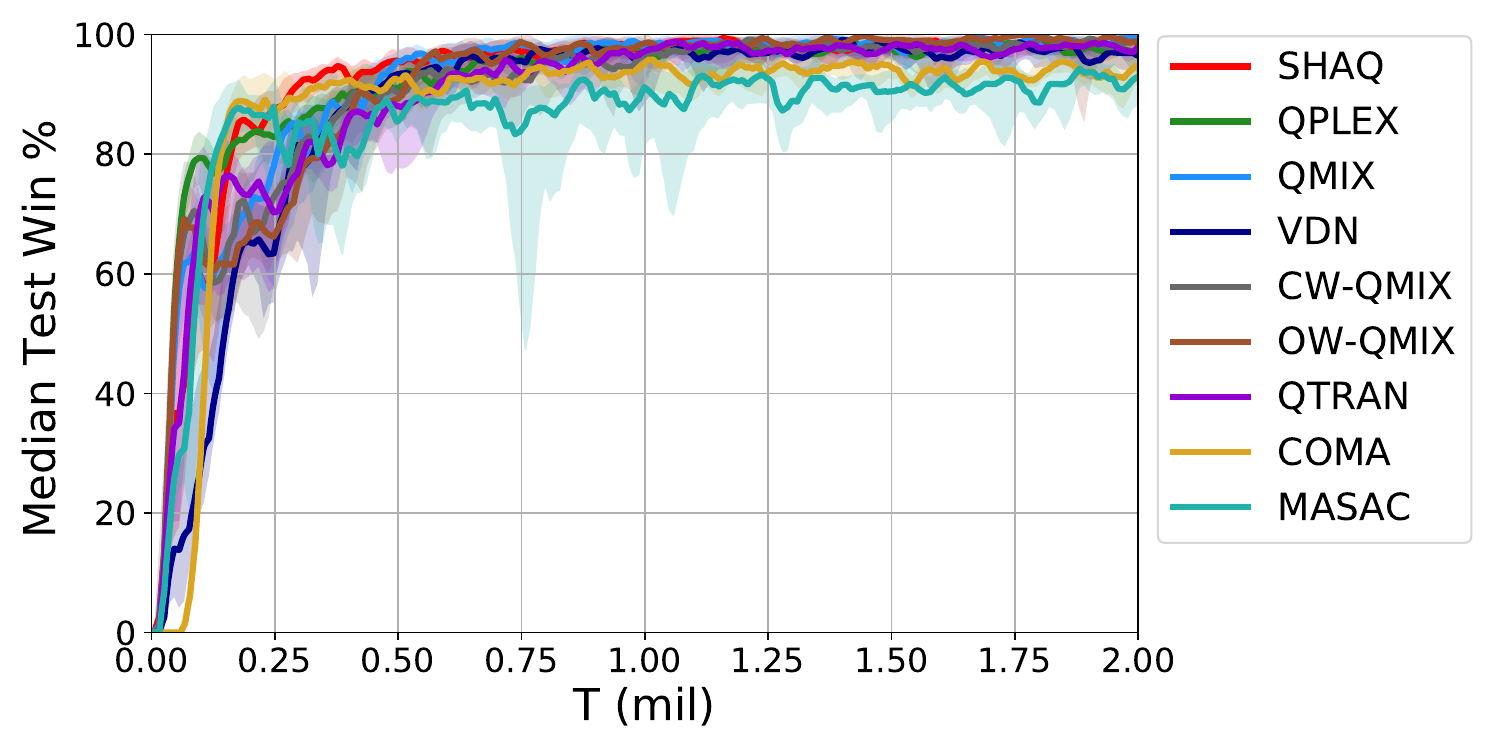}
                \caption{8m.}
            \end{subfigure}
            ~
            \begin{subfigure}[b]{0.32\textwidth}
                \includegraphics[width=\textwidth]{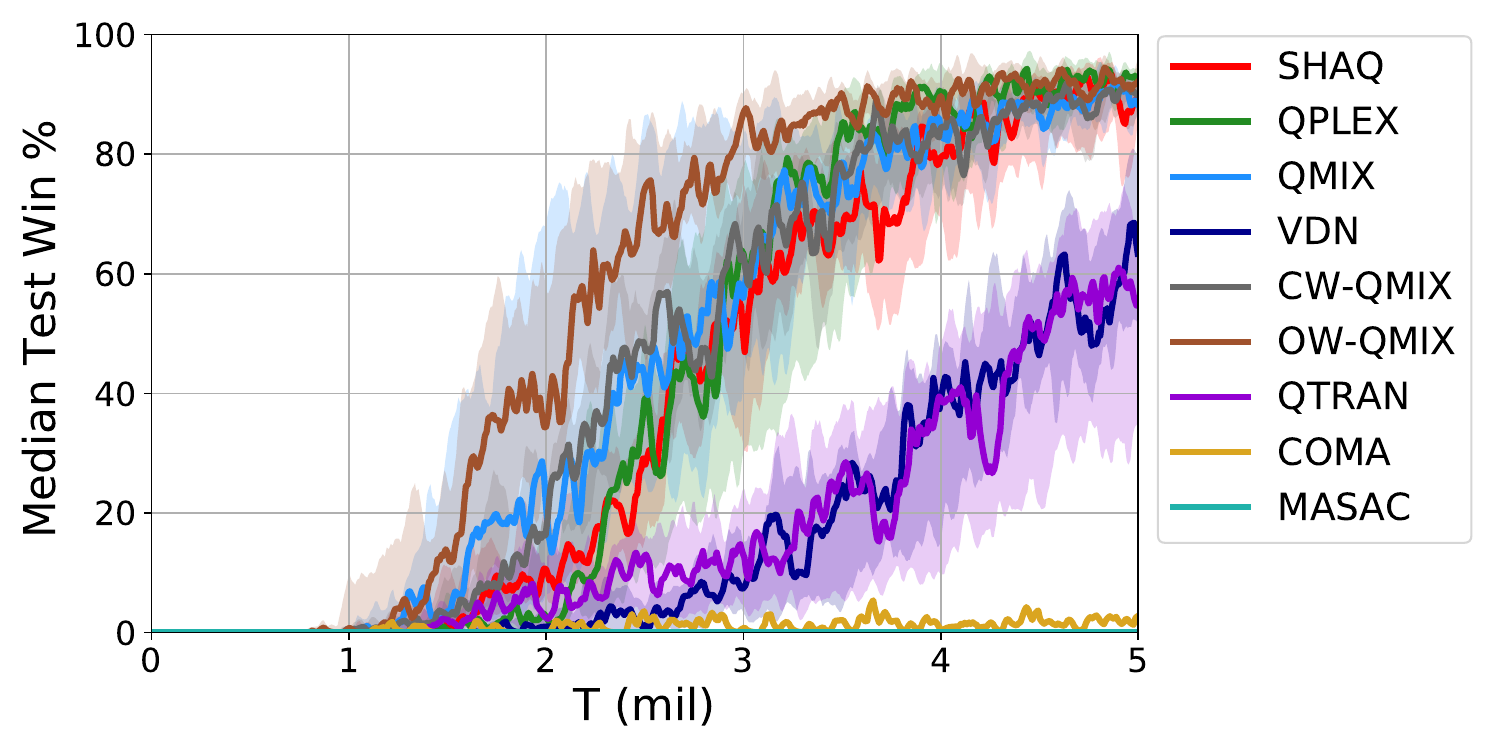}
                \caption{MMM2.}
            \end{subfigure}
            \caption{Median test win \% for 5 extra maps in SMAC.}
        \label{fig:extra_results_smac}
        \end{figure*}
        
    \subsection{Extra Experimental Results on W-QMIX with $\alpha=0.1$}
    \label{subsec:extra_weighted_qmix_variant_hyperparameters}
        To show the significance of tuning $\mathbf{\alpha}$ for W-QMIX, we also run W-QMIX with $\alpha=0.1$ in addition to the best $\alpha$ reported in \cite{rashid2020weighted}. We can observe from Figure \ref{fig:easy_hard_smac_wqmix} that the performances of W-QMIX are not comparatively identical for each choice of $\alpha$. As a result, W-QMIX suffers from the separate tuning of $\alpha$ for each scenario. Unfortunately, \citet{rashid2020weighted} did not provide an empirical law for selecting $\alpha$, while SHAQ enjoys an empirical law to select $\hat{\alpha}_{i}(\mathbf{s}, a_{i})$ as Figure \ref{fig:manual_approximate_alpha} shows.
        
        \begin{figure*}[ht]
            \centering
            \begin{subfigure}[b]{0.32\textwidth}
                \centering
                \includegraphics[width=\textwidth]{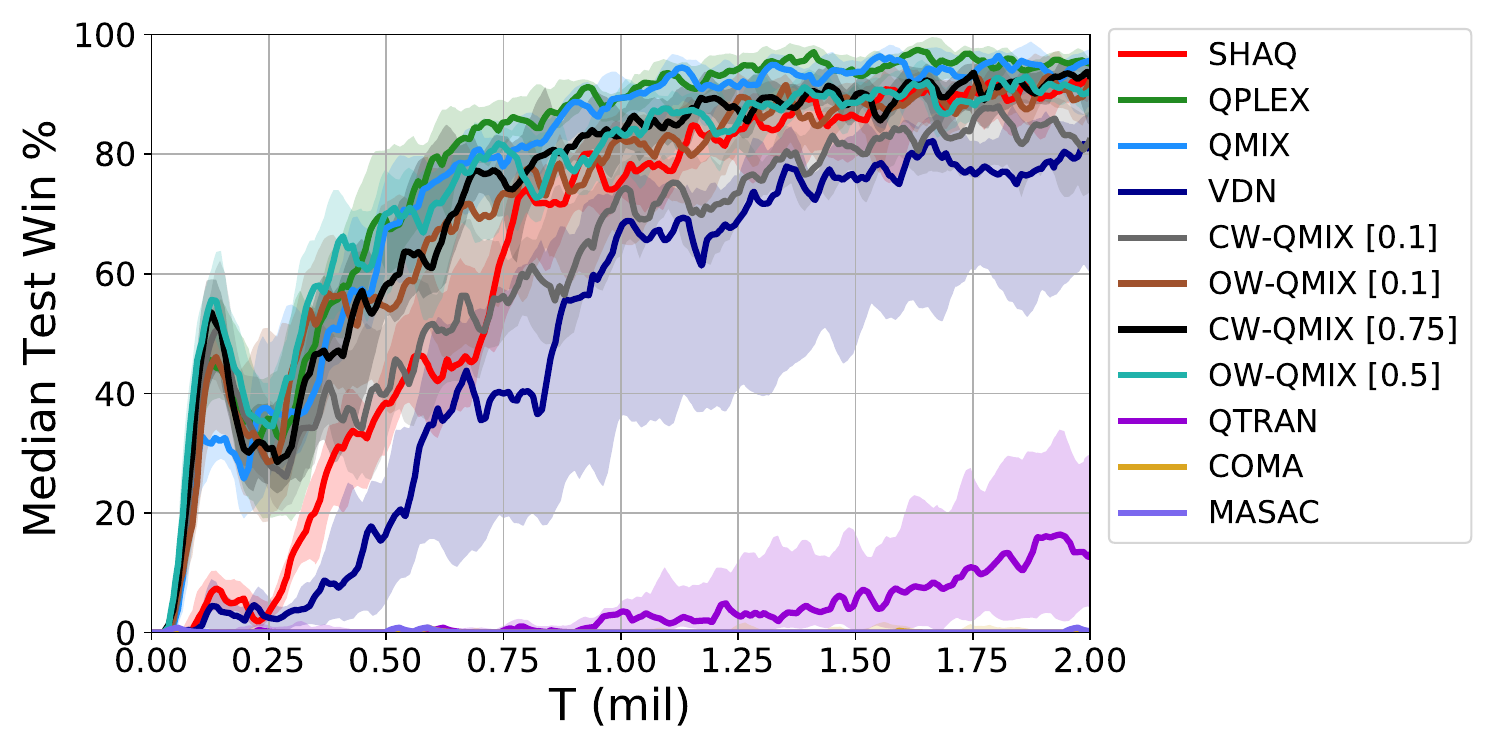}
                \caption{3s5z.}
            \end{subfigure}
            ~
            \begin{subfigure}[b]{0.32\textwidth}
                \centering
                \includegraphics[width=\textwidth]{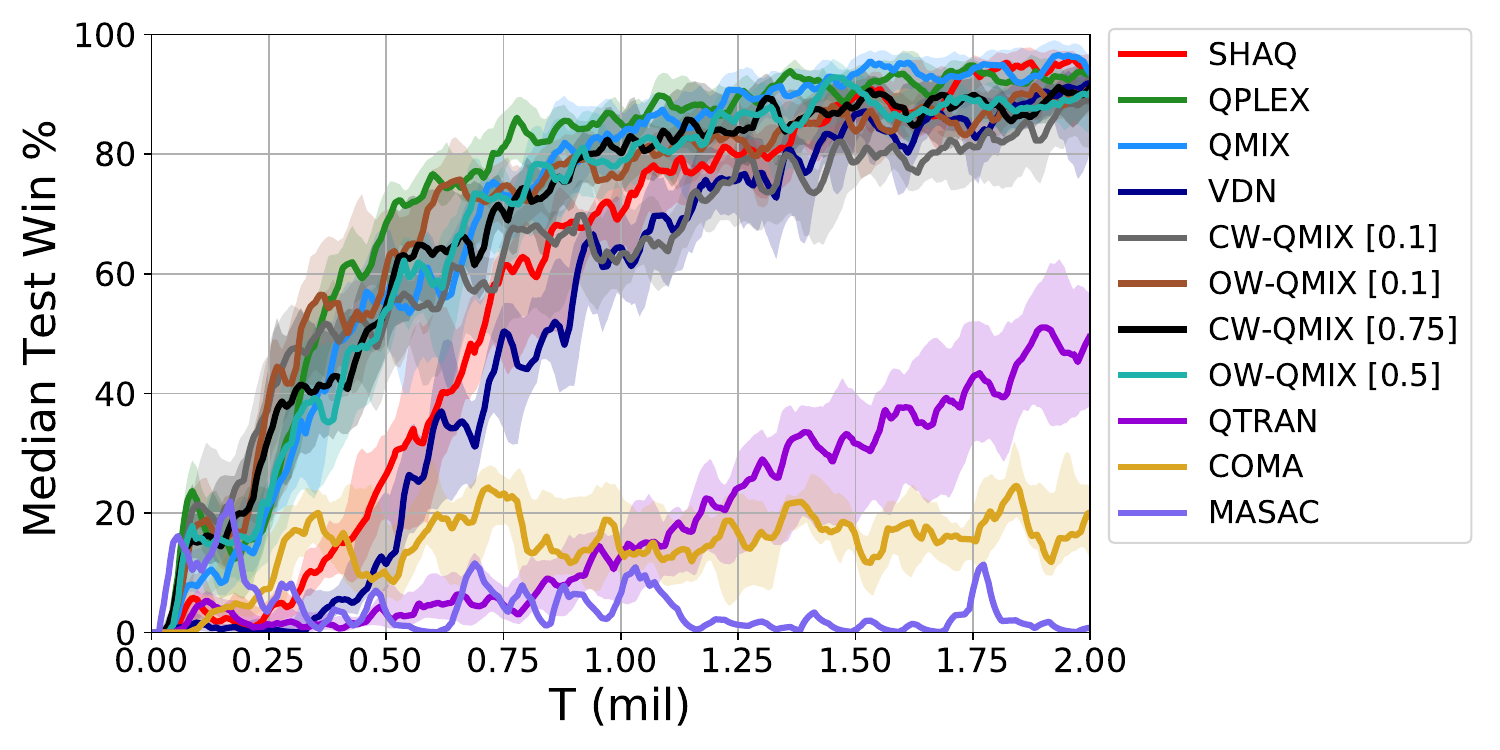}
                \caption{1c3s5z.}
            \end{subfigure}
            ~
            \begin{subfigure}[b]{0.32\textwidth}
                \centering
                \includegraphics[width=\textwidth]{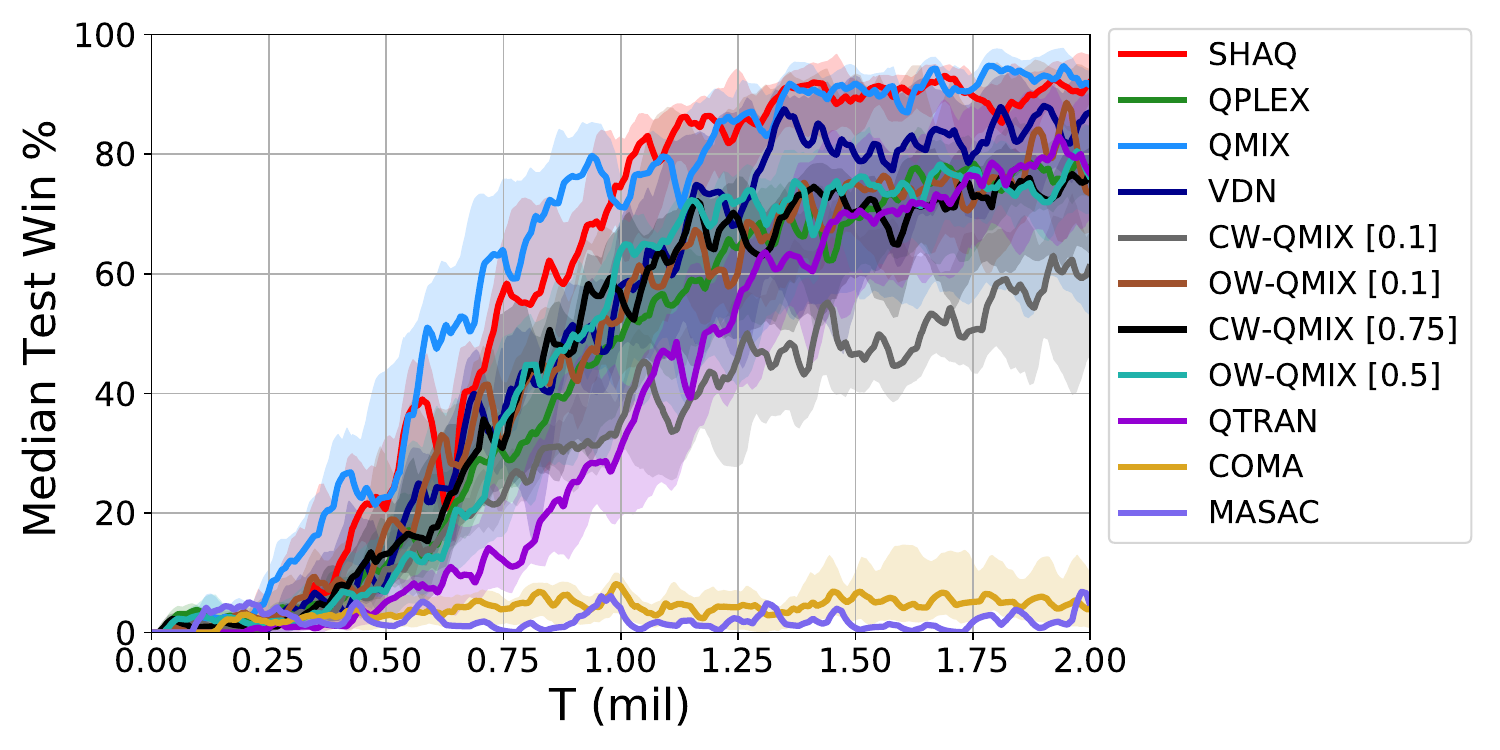}
                \caption{10m\_vs\_11m.}
            \end{subfigure}
            ~
            \begin{subfigure}[b]{0.32\textwidth}
                \centering
                \includegraphics[width=\textwidth]{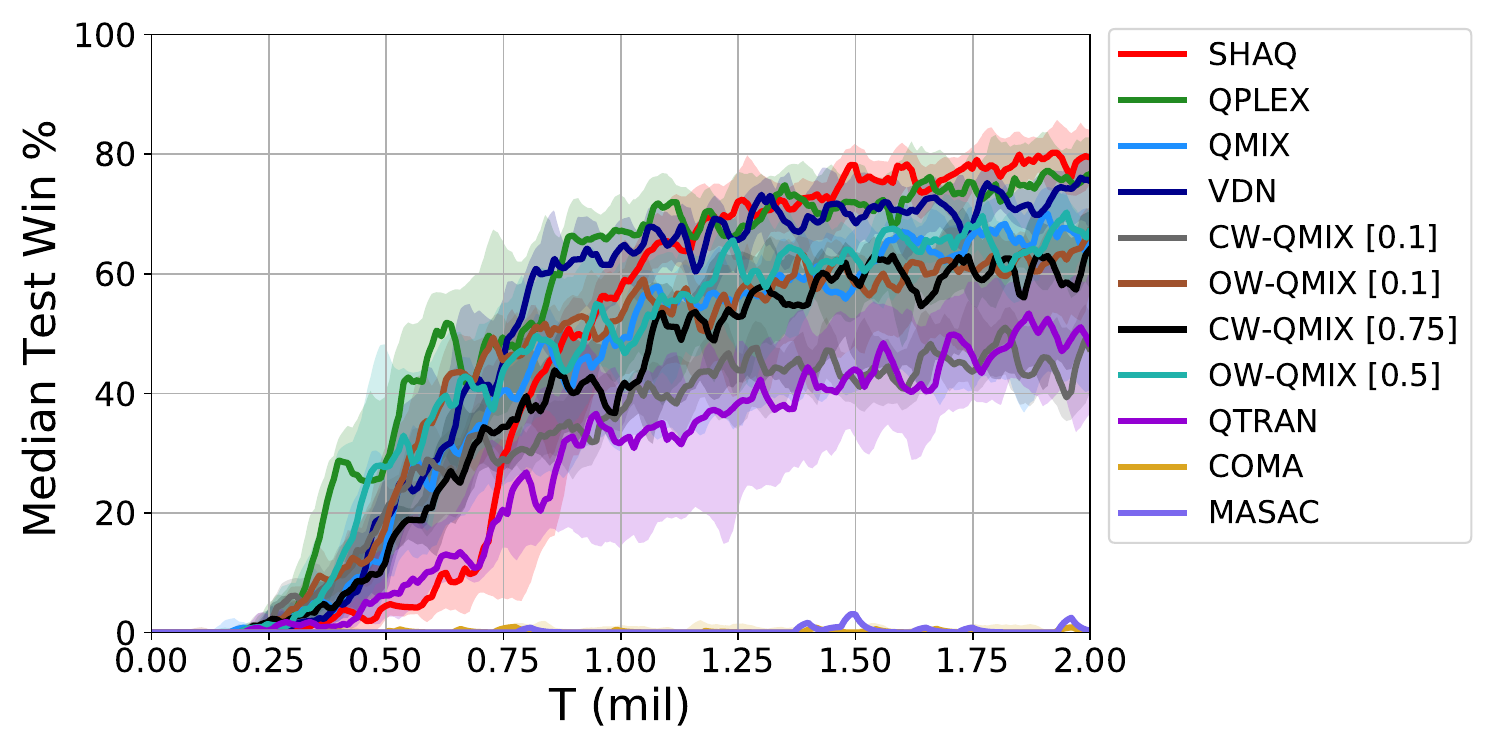}
                \caption{5m\_vs\_6m.}
            \end{subfigure}
            ~
            \begin{subfigure}[b]{0.32\textwidth}
                \centering
                \includegraphics[width=\textwidth]{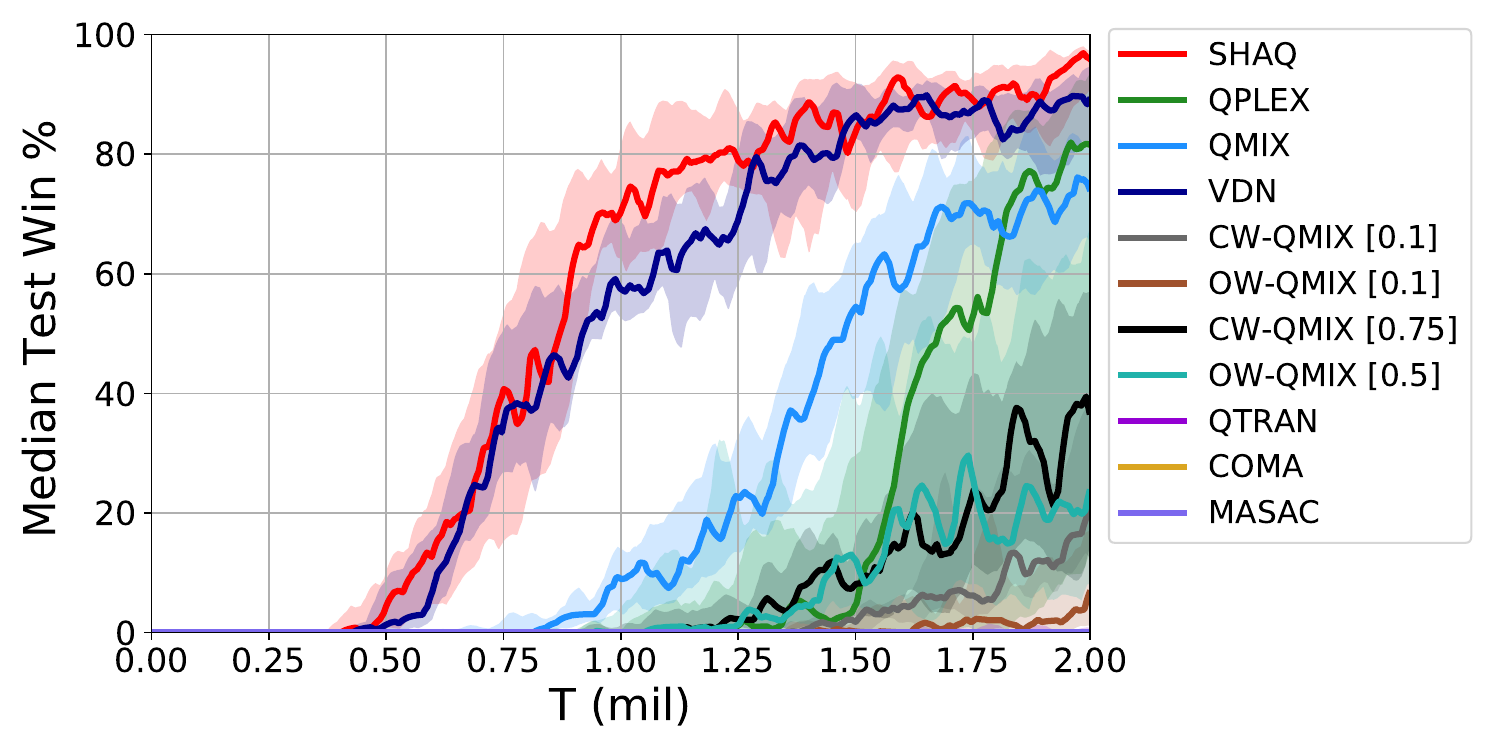}
                \caption{3s\_vs\_5z.}
            \end{subfigure}
            ~
            \begin{subfigure}[b]{0.32\textwidth}
                \centering
                \includegraphics[width=\textwidth]{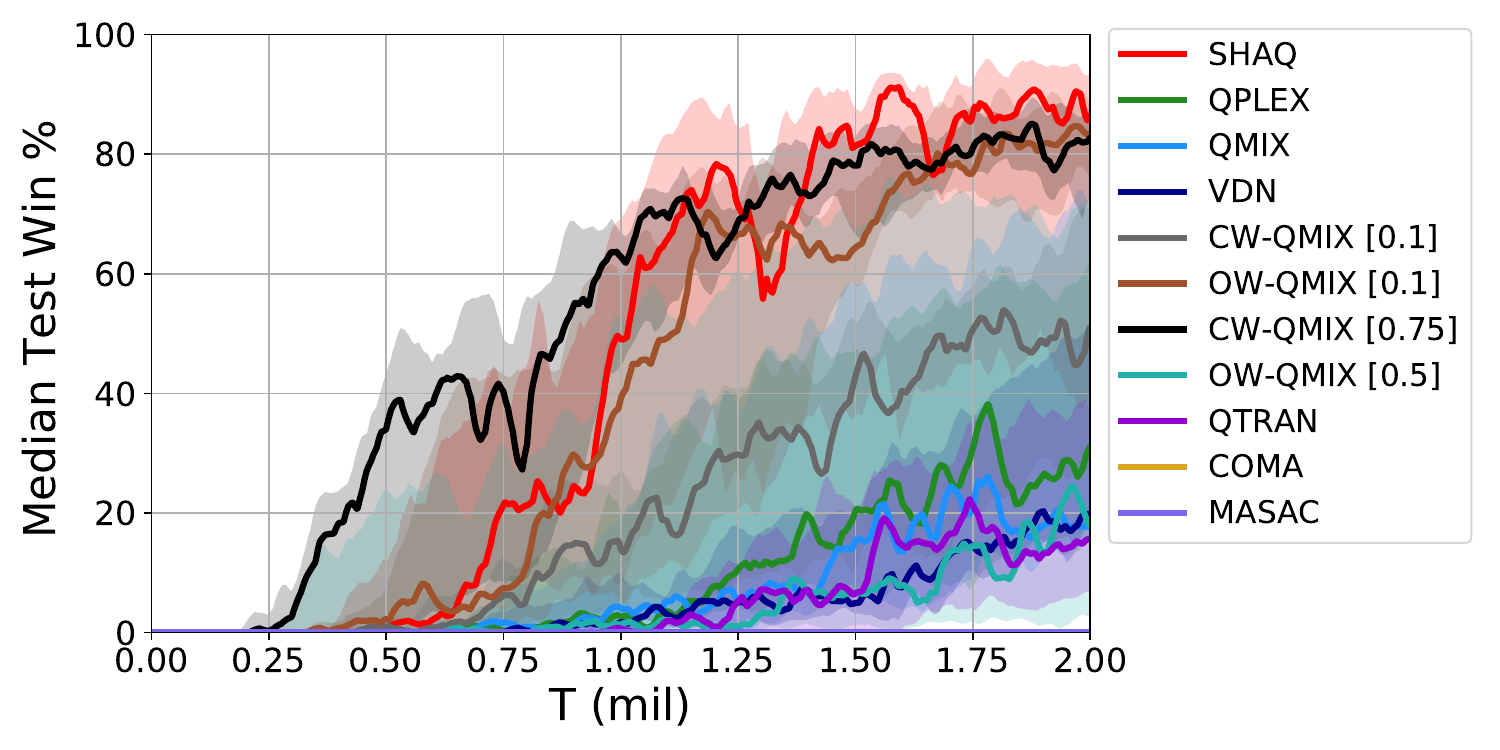}
                \caption{2c\_vs\_64zg.}
            \end{subfigure}
            \caption{Median test win \% for easy (1st row) and hard (2nd row) maps of SMAC for W-QMIX with different $\alpha$.}
            \label{fig:easy_hard_smac_wqmix}
        \end{figure*}
    
    \subsection{Comparison with SQDDPG}
    \label{subsec:comparison_with_sqddpg}
        To emphasize the improvement of SHAQ from SQDDPG \citep{Wang_2020}, we exclusively compare these two algorithms on 3 maps in SMAC. As Figure \ref{fig:sqddpg_shaq_smac} shows, the performance of SHAQ surpasses that of SQDDPG on all 3 maps, while SQDDPG can only learn on the simplest map 3m. The most possible reason for the failure of SQDDPG to complicated tasks is its sample complexity inefficiency for permutations of agents as discussed in Section \ref{sec:related_work} that leads to the difficulty in learning. Apparently, the implementation of coalition invariance of SHAQ mitigates this weakness so that it is able to solve more challenging tasks. We also show the results for SQDDPG on Predator-Prey with the same setups (i.e., the epsilon annealing steps are 1 mil), as Figure \ref{fig:sqddpg_wqmix_pp_a} shows. It is apparent that SHAQ can still outperform SQDDPG.
        
        \begin{figure*}[ht!]
            \centering
            \begin{subfigure}[b]{0.45\linewidth}
                \centering
                \includegraphics[width=\textwidth]{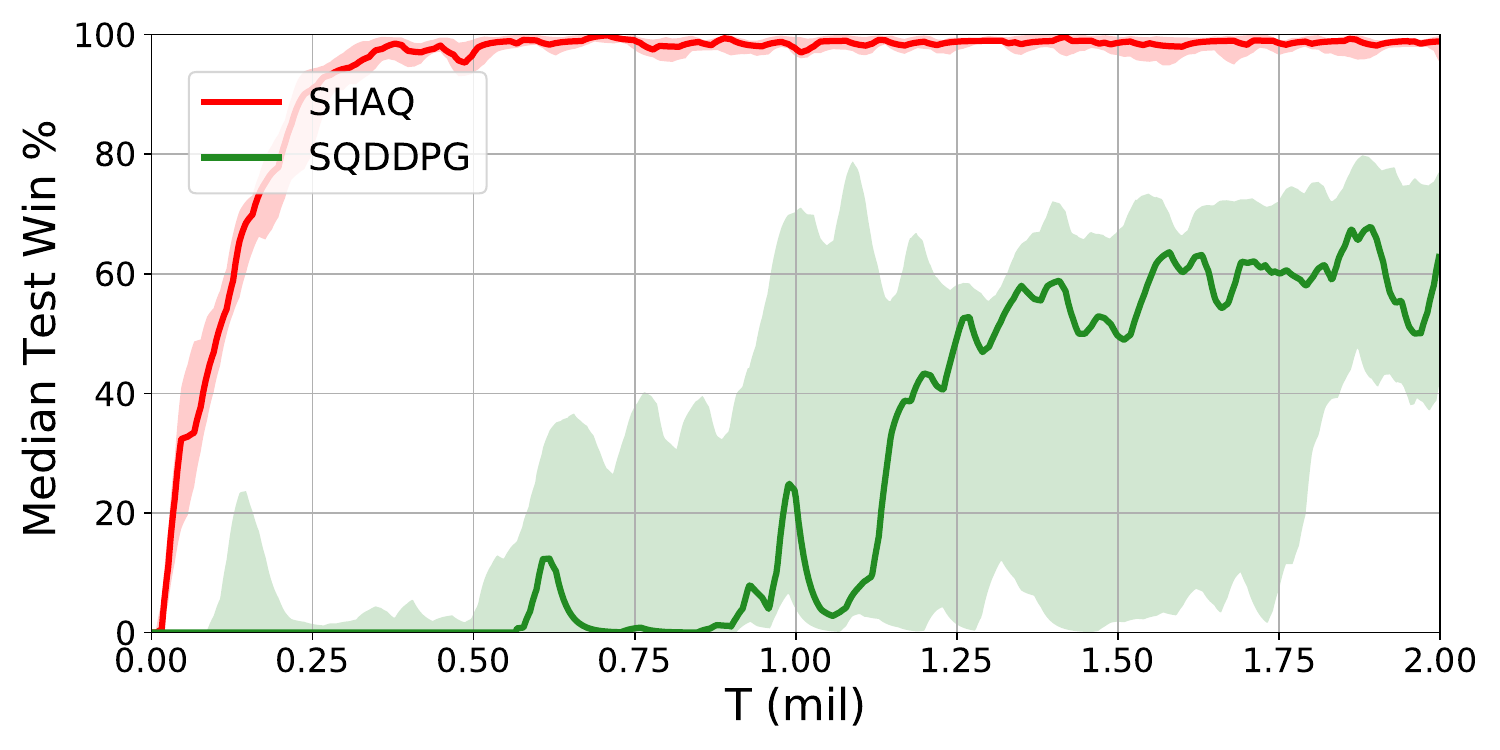}
                \caption{3m.}
            \end{subfigure}
            \quad
            \begin{subfigure}[b]{0.45\linewidth}
                \centering
                \includegraphics[width=\textwidth]{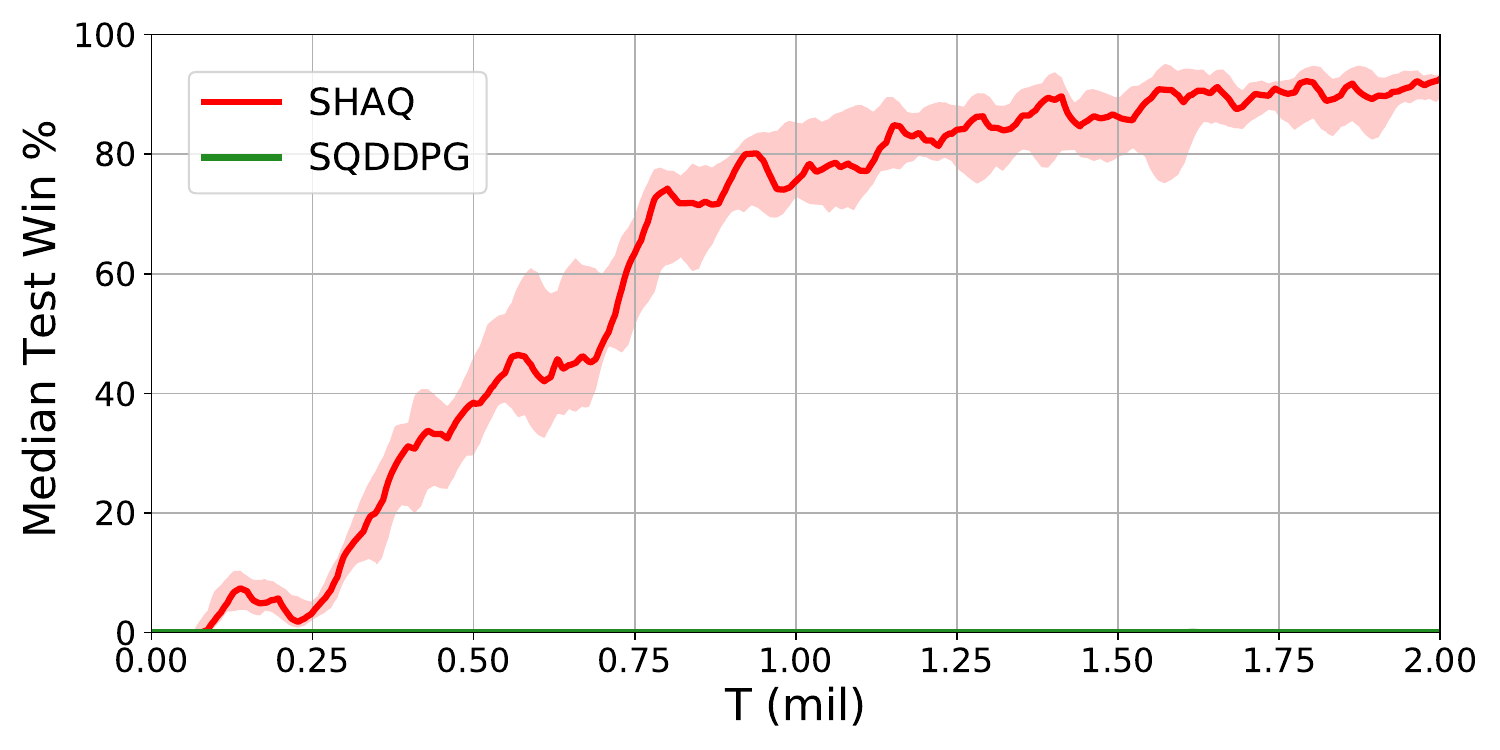}
                \caption{3s5z.}
            \end{subfigure}
            \quad
            \begin{subfigure}[b]{0.45\linewidth}
                \centering
                \includegraphics[width=\textwidth]{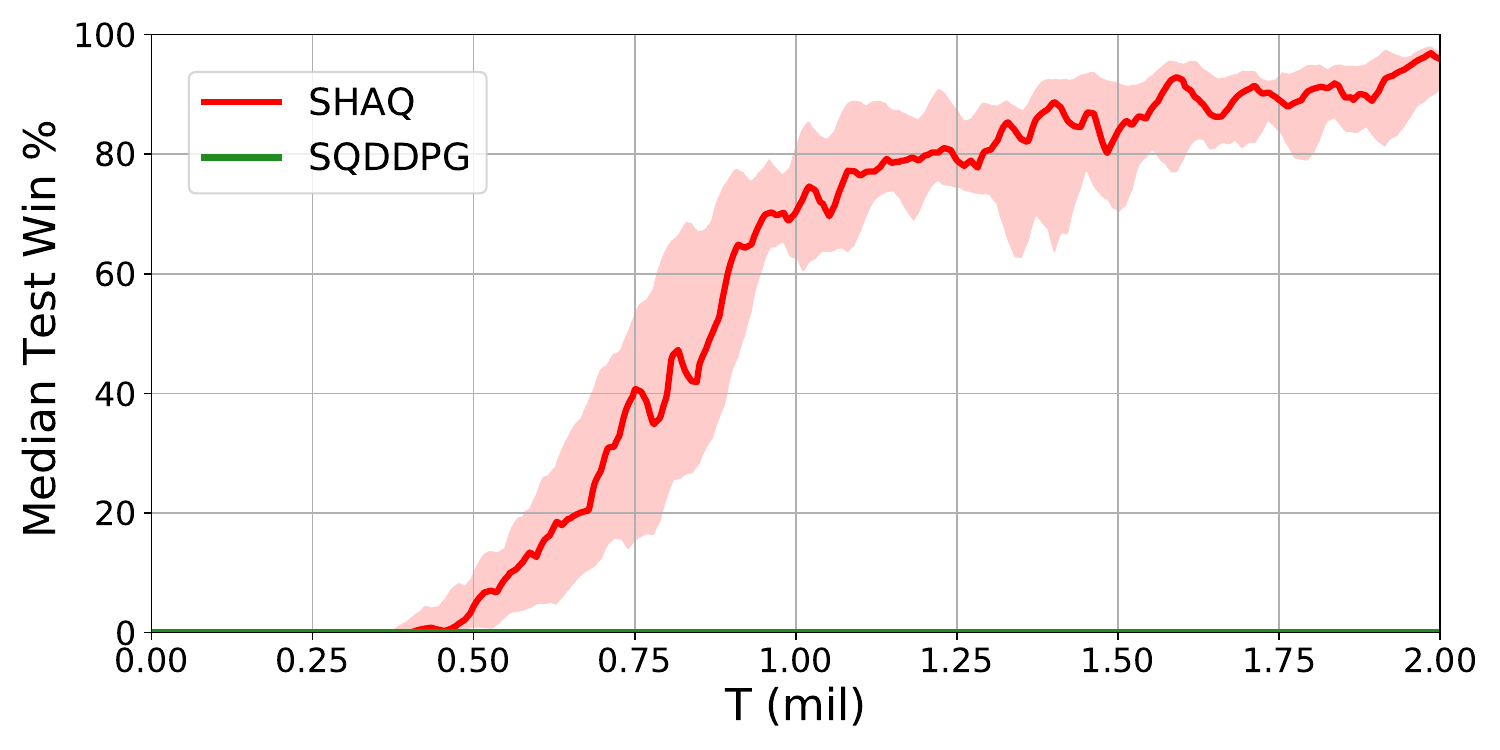}
                \caption{3s\_vs\_5z.}
            \end{subfigure}
            \caption{Median test win \% for 3 maps of SMAC to compare SHAQ with SQDDPG.}
        \label{fig:sqddpg_shaq_smac}
        \end{figure*}
    
    \subsection{Ablation Study}
    \label{subsec:ablation_study}
    
        \begin{figure*}[ht!]
        \centering
            \begin{subfigure}[b]{0.45\textwidth}
                \centering                \includegraphics[width=\textwidth]{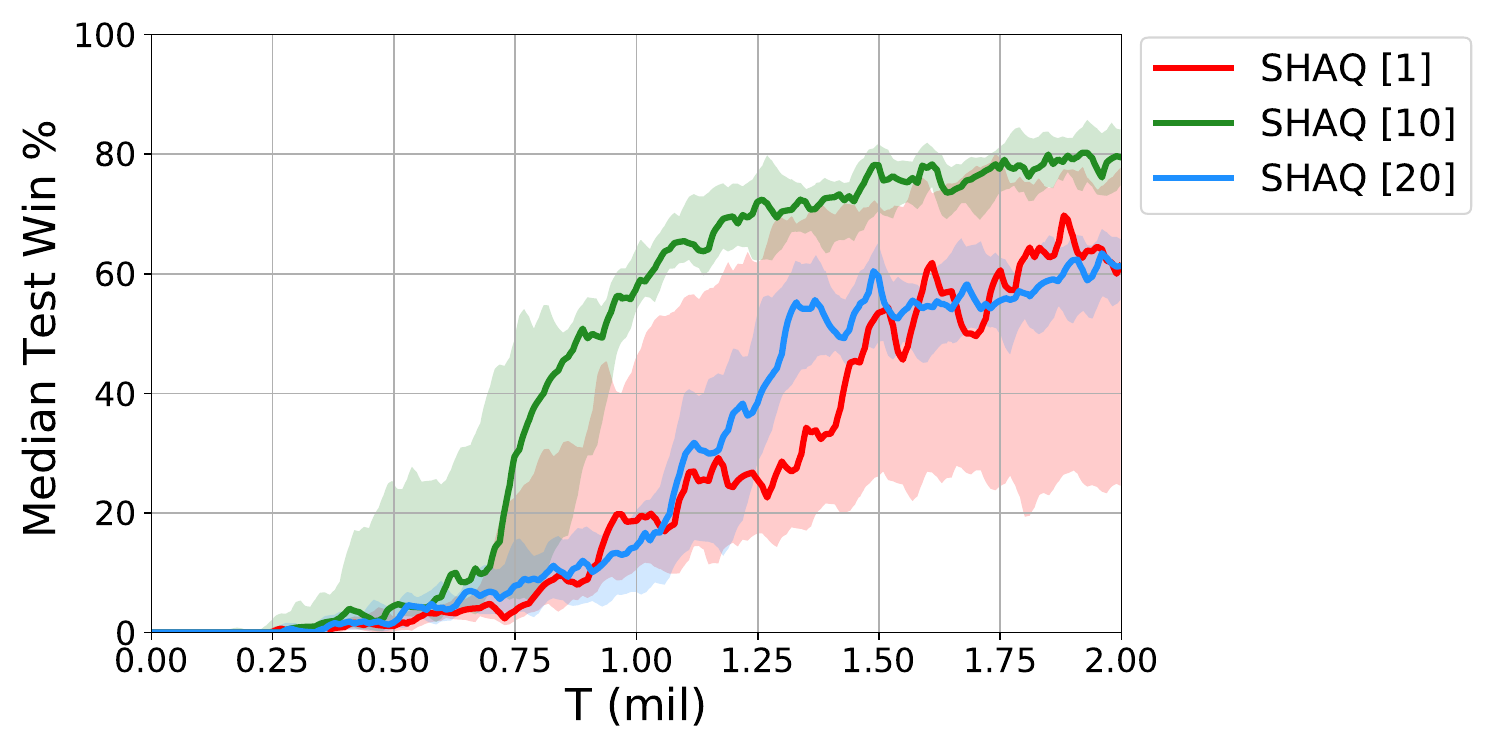}
                \caption{Comparison among different values of M on 5m\_vs\_6m. The $[\cdot]$ indicates the value of M.}
            \label{fig:ablation_5m_vs_6m}
            \end{subfigure}
            \quad
            \begin{subfigure}[b]{0.45\textwidth}
                \centering                 \includegraphics[width=\textwidth]{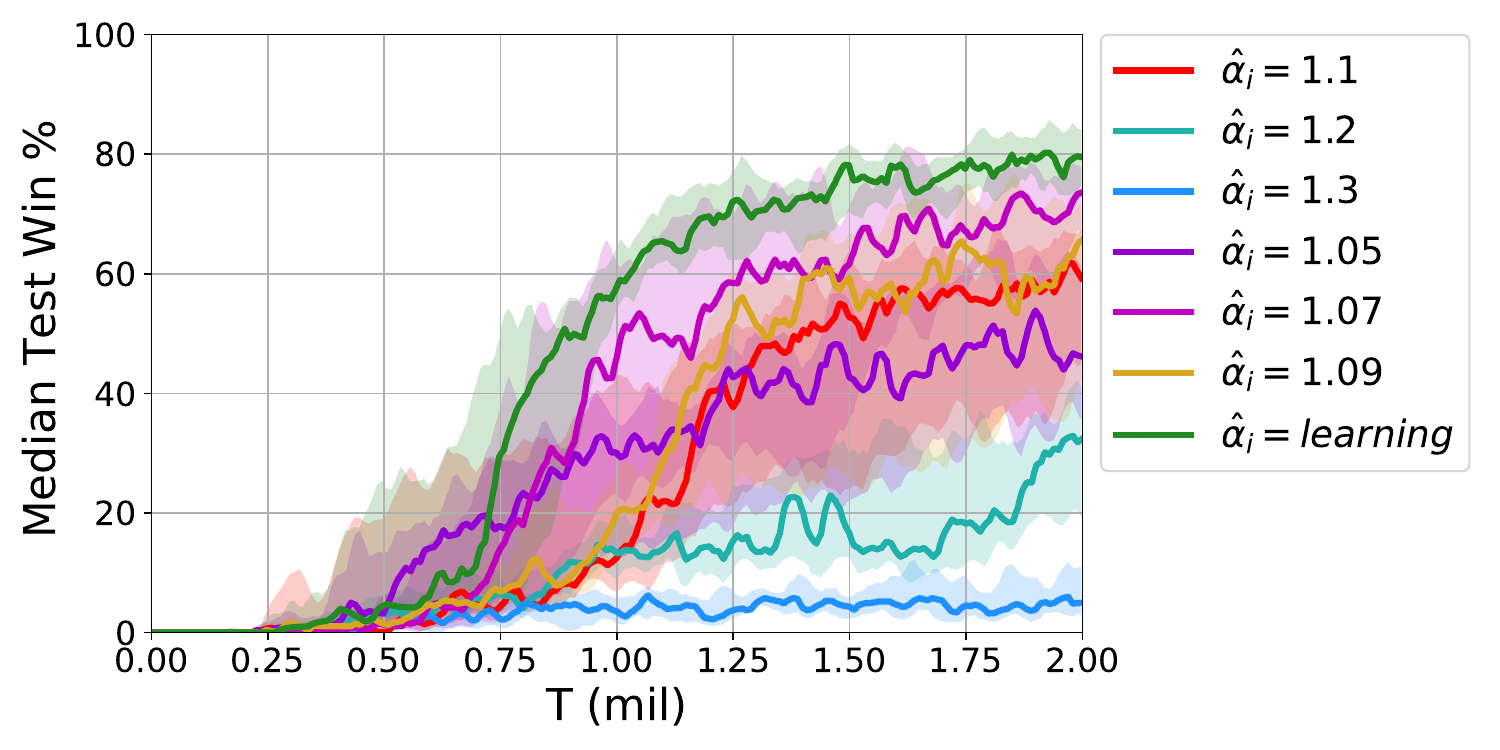}
                \caption{Comparison between the manually preset and the learned $\hat{\alpha}_{i}(s, a_{i})$ on 5m\_vs\_6m.}
            \label{fig:manual_approximate_alpha}
            \end{subfigure}
            \quad
            \begin{subfigure}[b]{0.45\textwidth}
                \centering                \includegraphics[width=\textwidth]{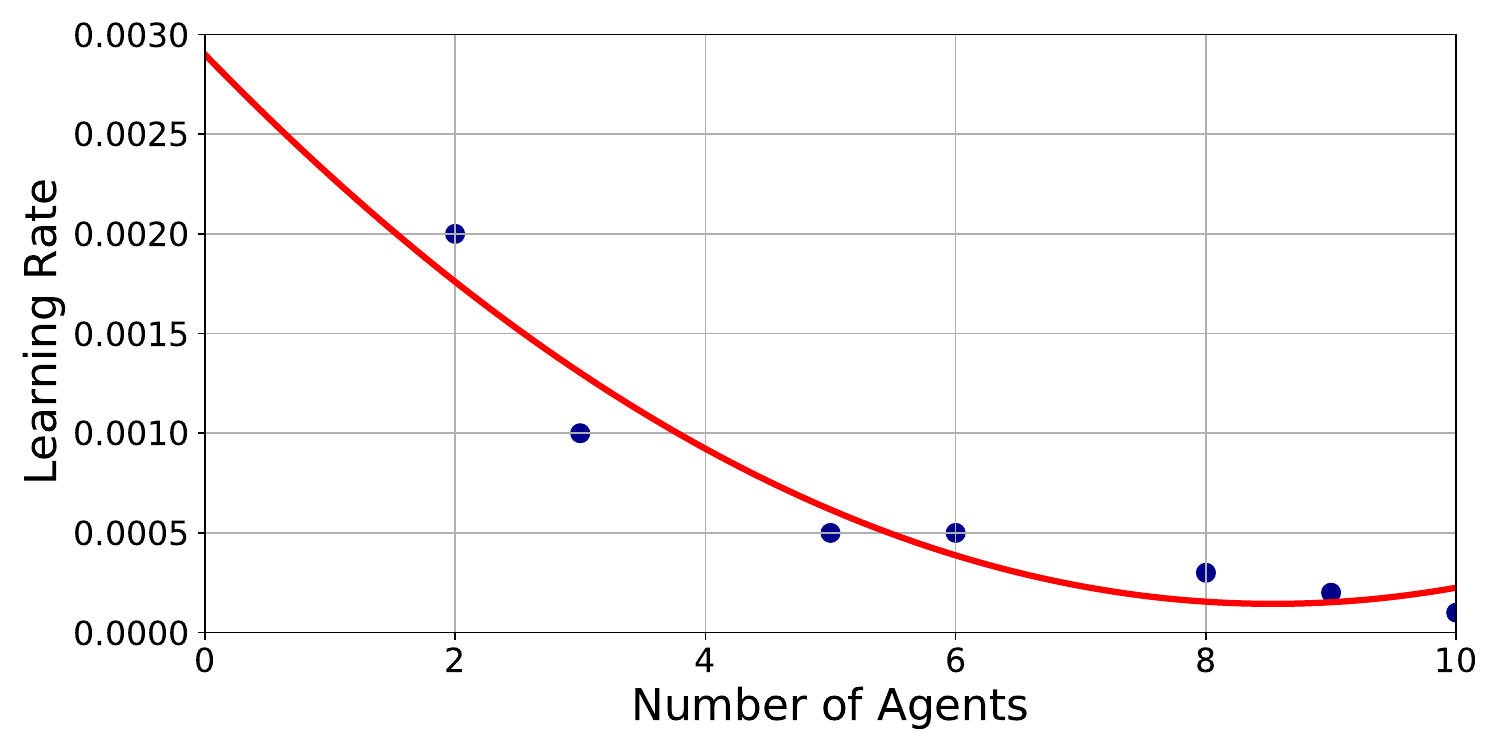}
                \caption{Relationship between learning rate of $\hat{\alpha}_{i}(s, a_{i})$ and the number of agents (in red curve).}
            \label{fig:correlation_num_agents_to_mixer_lr}
            \end{subfigure}
            \caption{The figures of 3 ablation studies of SHAQ on SMAC.}
        \label{fig:ablation_study}
        \end{figure*}
        
        We also conduct ablation study of SHAQ, such as the sample size M for approximating $\hat{\alpha}_{i}(\mathbf{s}, a_{i})$, the empirical selection law on the learning rate of $\hat{\alpha}_{i}(\mathbf{s}, a_{i})$, and the demonstration of the necessity of learning $\hat{\alpha}_{i}(\mathbf{s}, a_{i})$ rather than manual setting. These results show that SHAQ is an easy-to-use algorithm that is potential to be applied to other scenarios with less efforts on tuning hyperparameters.
        
        \textbf{Sample Size M for Approximating $\hat{\alpha}(\mathbf{s}, a_{i})$.} To study the impact of sample size M on the performance of SHAQ, we conduct an ablation study as Figure \ref{fig:ablation_5m_vs_6m} shows. We observe that the small M is able to achieve fast convergence rate but with high variance, while the large M is with low variance but comparatively slow convergence rate. The observations are consistent with the conclusions from stochastic optimisation \citep{Byrd2012,NIPS2015_effc299a}. As a result, we select M = 10 in practice, to trade off between convergence rate and variance. 
        
        \textbf{An Empirical Law on Selecting the Learning Rate of $\hat{\alpha}_{i}(s, a_{i})$.} To provide an empirical law on selecting the learning rate of $\hat{\alpha}_{i}(s, a_{i})$, we statistically fit a curve of the learning rate w.r.t. the number of controllable agents by the experimental results on SMAC that is shown in Figure \ref{fig:correlation_num_agents_to_mixer_lr}. It is seen that the learning rate of $\hat{\alpha}_{i}(s, a_{i})$ is generally negatively related to the number of agents. In other words, as the number of agents grows the learning rate of $\hat{\alpha}_{i}(s, a_{i})$ is recommended to be smaller. For example, if the number of agents is more than 10, the learning rate of $\hat{\alpha}_{i}(s, a_{i})$ is recommended to be 0.0001 as the guidance from Figure \ref{fig:correlation_num_agents_to_mixer_lr}.
        
        \textbf{The Necessity of Learning $\hat{\alpha}_{i}(\mathbf{s}, a_{i})$.} Some readers may be concerned about the necessity of learning $\hat{\alpha}_{i}(\mathbf{s}, a_{i})$. To answer this question, we study the necessity of learning $\hat{\alpha}_{i}(\mathbf{s}, a_{i})$ on 5m\_vs\_6m. Since the learned $\hat{\alpha}_{i}(\mathbf{s}, a_{i})$ finally converges to $1.1029$, we grid search the fixed values of $\hat{\alpha}_{i}(\mathbf{s}, a_{i})$ around this number. As Figure \ref{fig:manual_approximate_alpha} shows, $\hat{\alpha}_{i}(\mathbf{s}, a_{i})$ with manually preset fixed value cannot work as well as the learned $\hat{\alpha}_{i}(\mathbf{s}, a_{i})$. Therefore, we demonstrate the necessity of learning $\hat{\alpha}_{i}(\mathbf{s}, a_{i})$ here.
            
    \subsection{More Visualisation for Interpretability of SHAQ}
    \label{subsec:more_visualisation}
        \begin{figure*}[ht!]
            \centering
            \begin{subfigure}[b]{0.23\textwidth}
            	\centering
        	    \includegraphics[width=\textwidth]{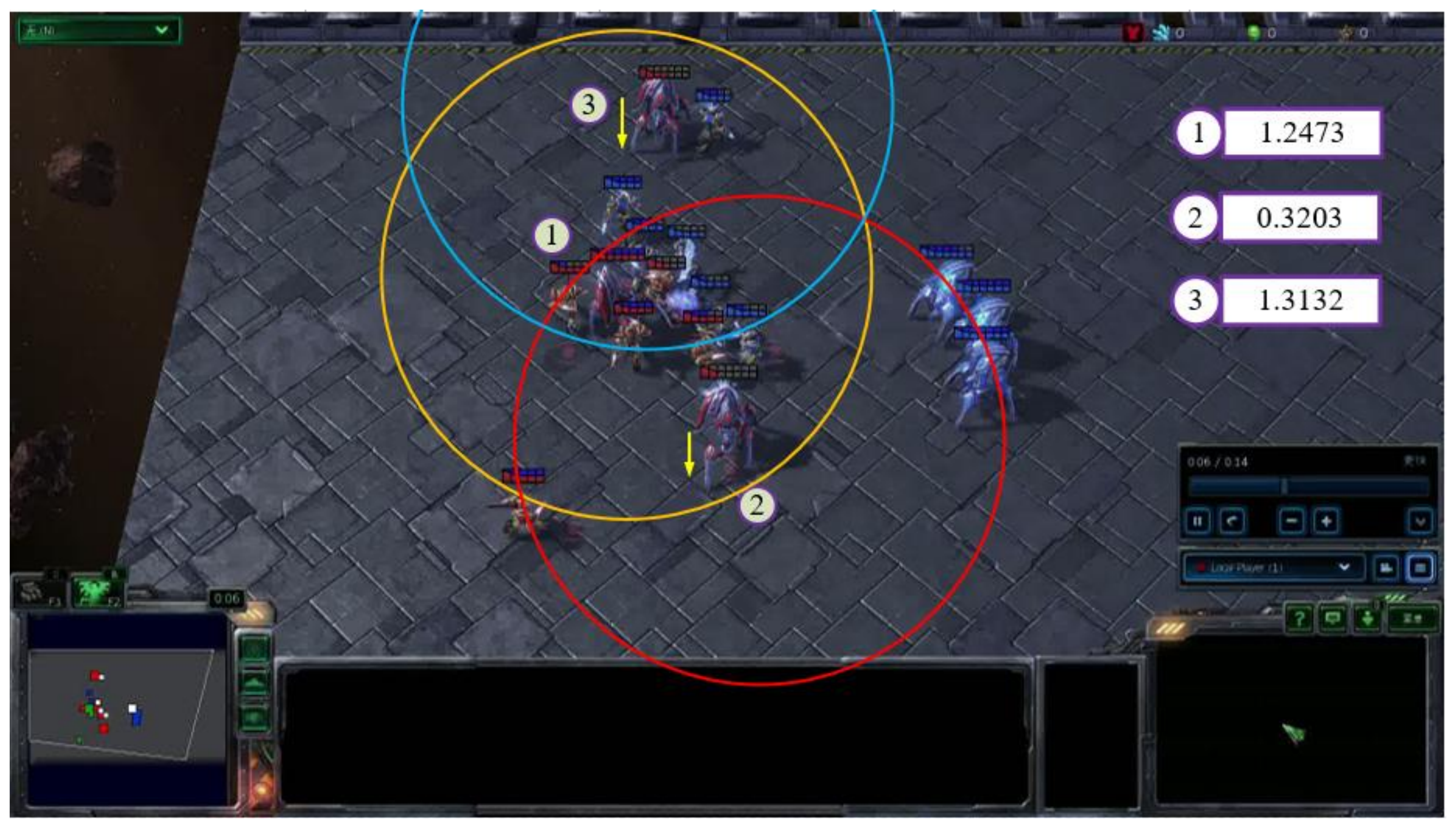}
        	    \caption{SHAQ: $\epsilon$-greedy.}
        	\label{fig:shaq_epsilon_greedy_appendix}
            \end{subfigure}
            ~
            \begin{subfigure}[b]{0.23\textwidth}
                \centering                \includegraphics[width=\textwidth]{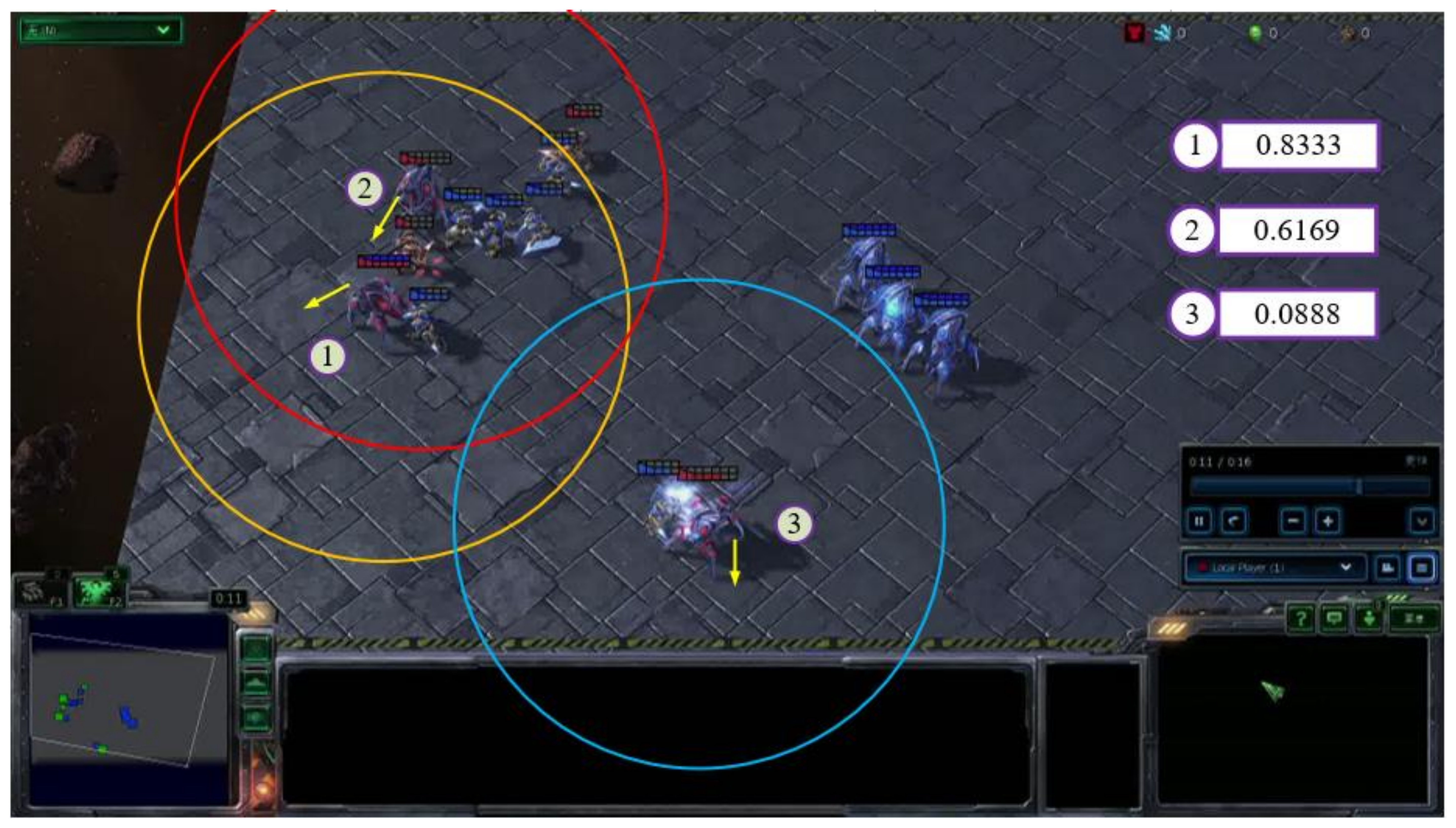}
                \caption{VDN: $\epsilon$-greedy.}
            \label{fig:vdn_epsilon_greedy_appendix}
            \end{subfigure}
            ~
            \begin{subfigure}[b]{0.23\textwidth}
                \centering                \includegraphics[width=\textwidth]{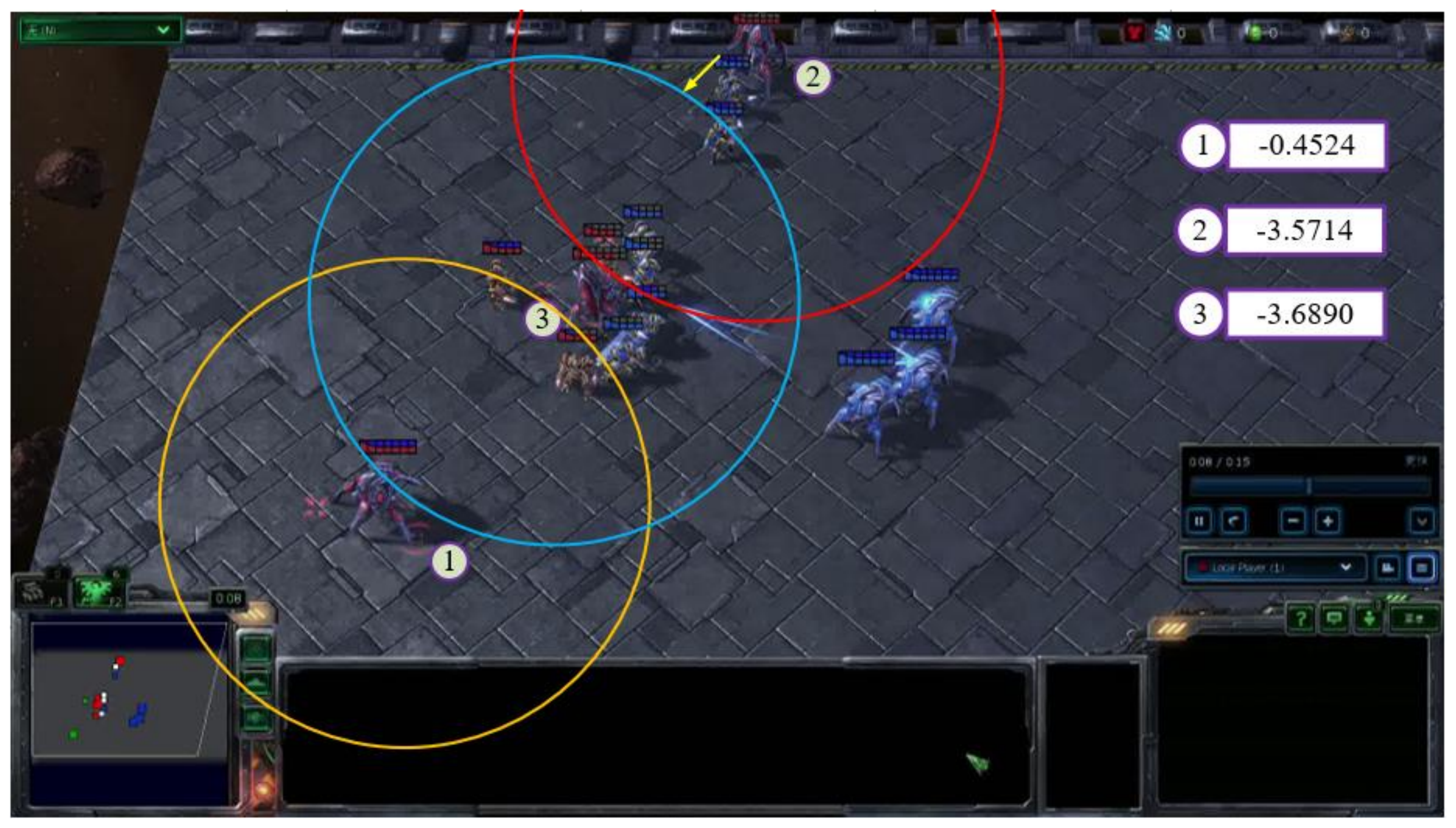}
                \caption{QMIX: $\epsilon$-greedy.}
            \label{fig:qmix_epsilon_greedy_appendix}
            \end{subfigure}
            ~
            \begin{subfigure}[b]{0.23\textwidth}
        	    \centering        	    \includegraphics[width=\textwidth]{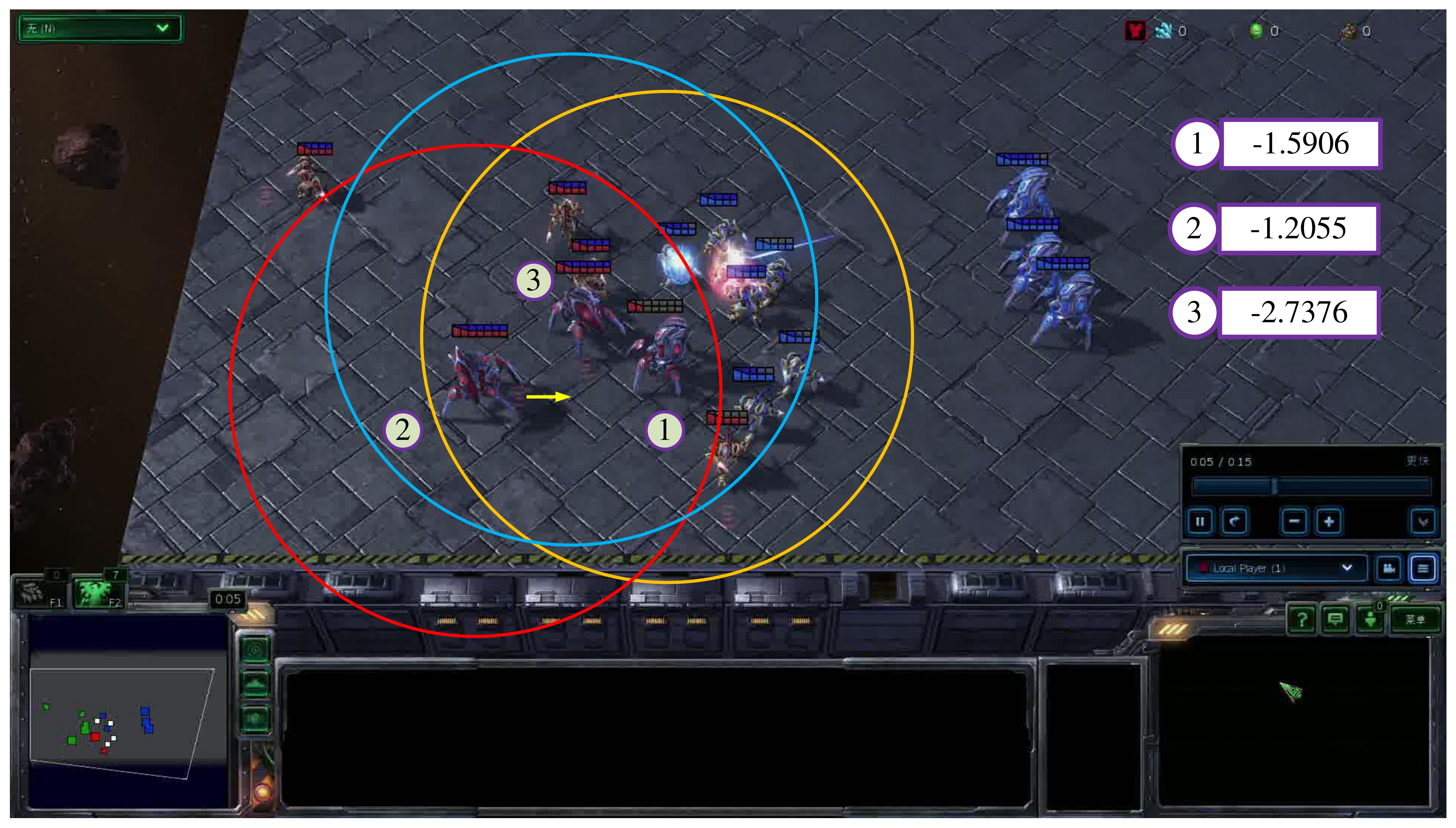}
        	    \caption{QPLEX: $\epsilon$-greedy.}
        	\label{fig:qplex_epsilon_greedy_appendix}
            \end{subfigure}
            
            \begin{subfigure}[b]{0.23\textwidth}
        	    \centering        	    \includegraphics[width=\textwidth]{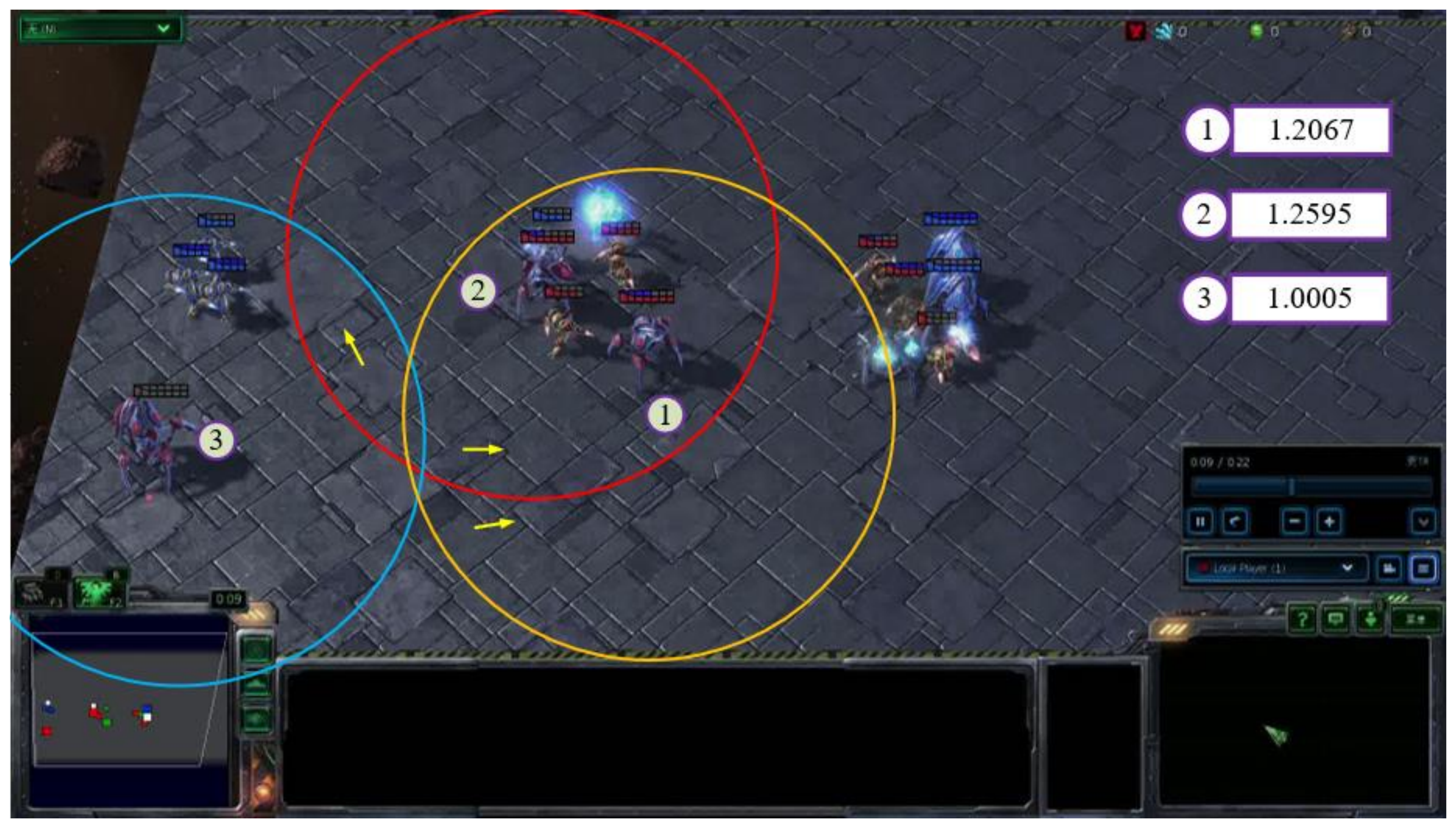}
        	    \caption{SHAQ: greedy.}
        	\label{fig:shaq_greedy_appendix}
            \end{subfigure}
            ~
            \begin{subfigure}[b]{0.23\textwidth}
            	\centering        	    \includegraphics[width=\textwidth]{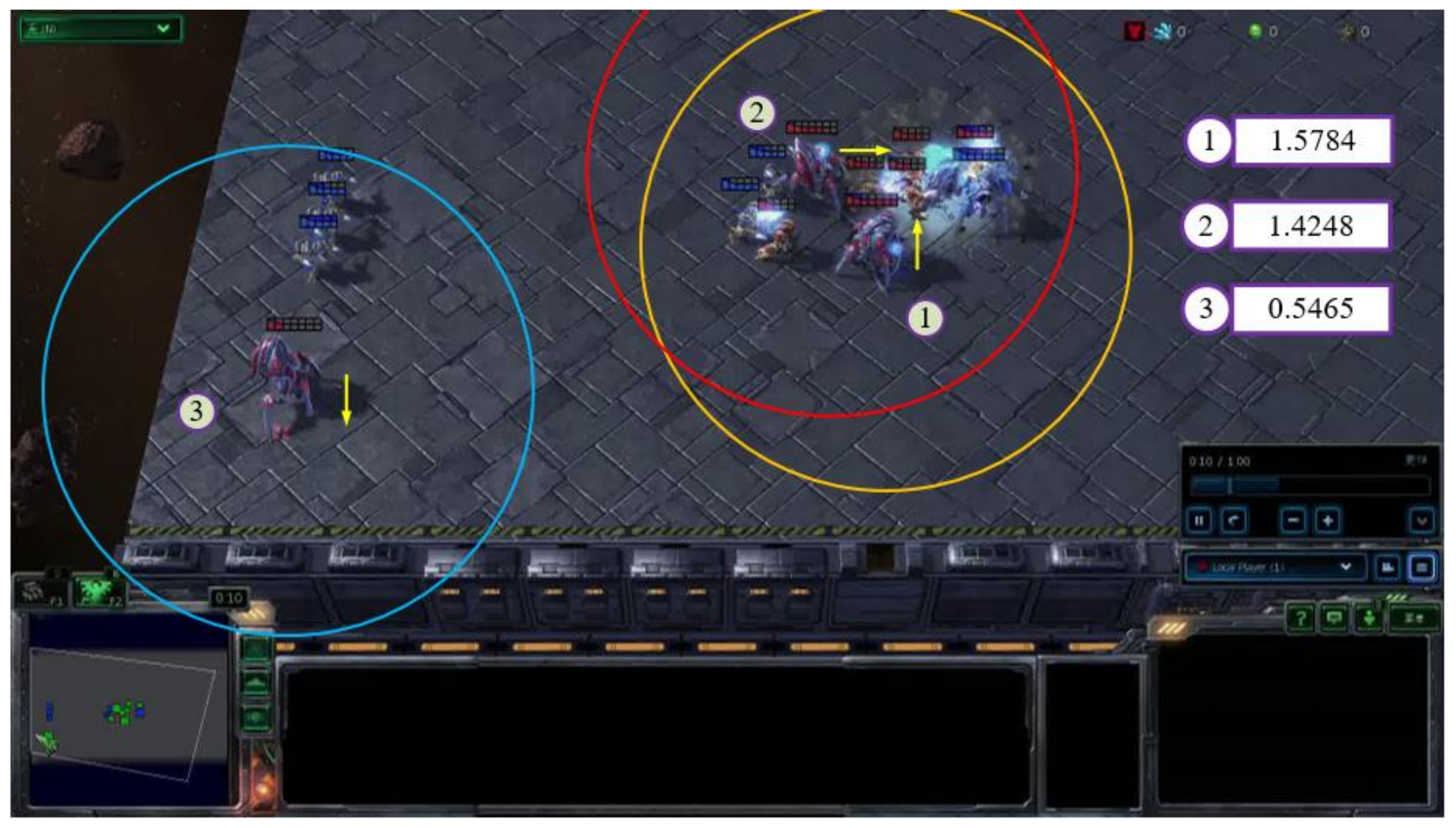}
        	    \caption{VDN: greedy.}
        	\label{fig:vdn_greedy_appendix}
            \end{subfigure}
            ~
            \begin{subfigure}[b]{0.23\textwidth}
        	    \centering        	    \includegraphics[width=\textwidth]{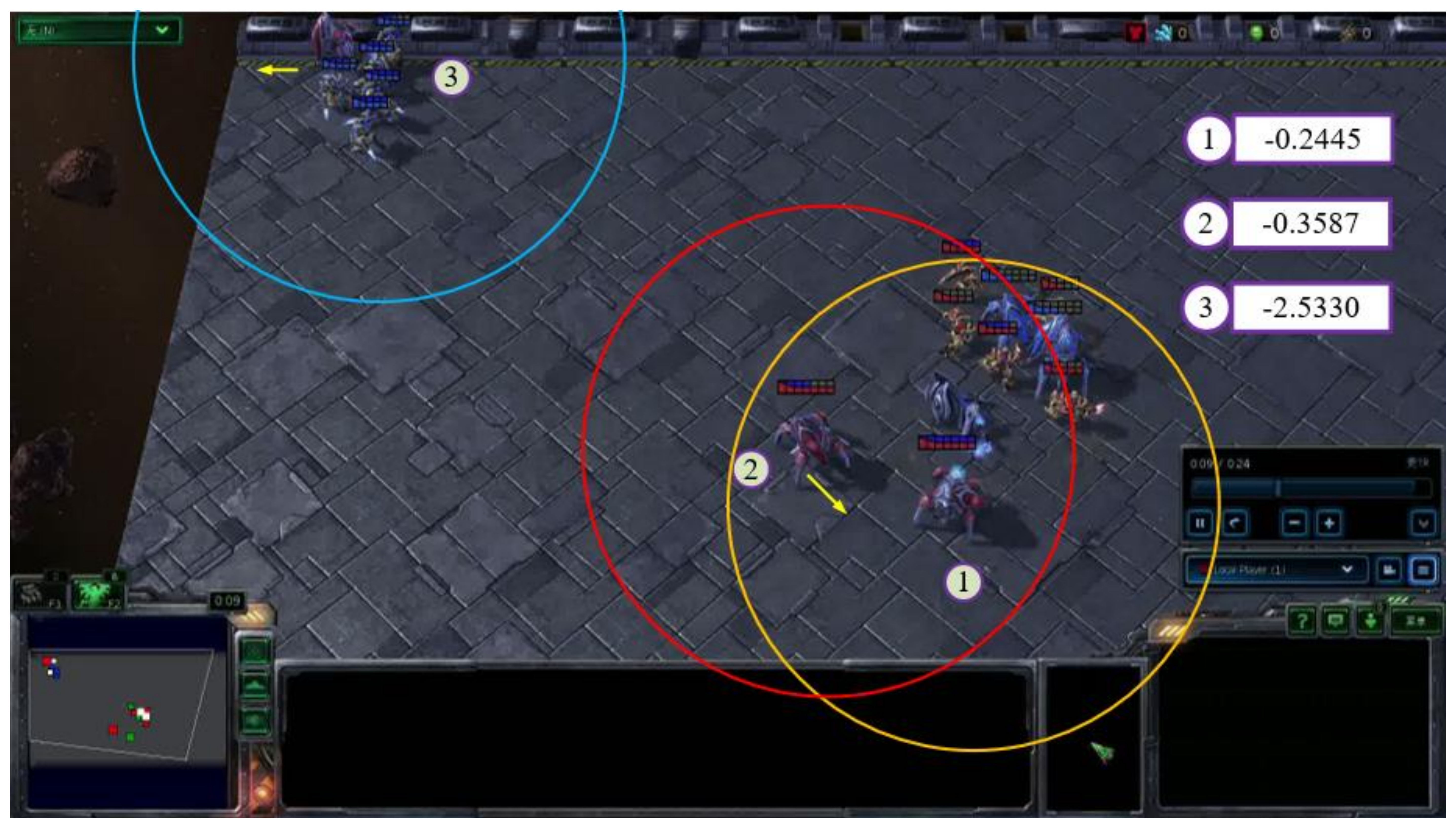}
        	    \caption{QMIX: greedy.}
        	\label{fig:qmix_greedy_appendix}
            \end{subfigure}
            ~
            \begin{subfigure}[b]{0.23\textwidth}
        	    \centering        	    \includegraphics[width=\textwidth]{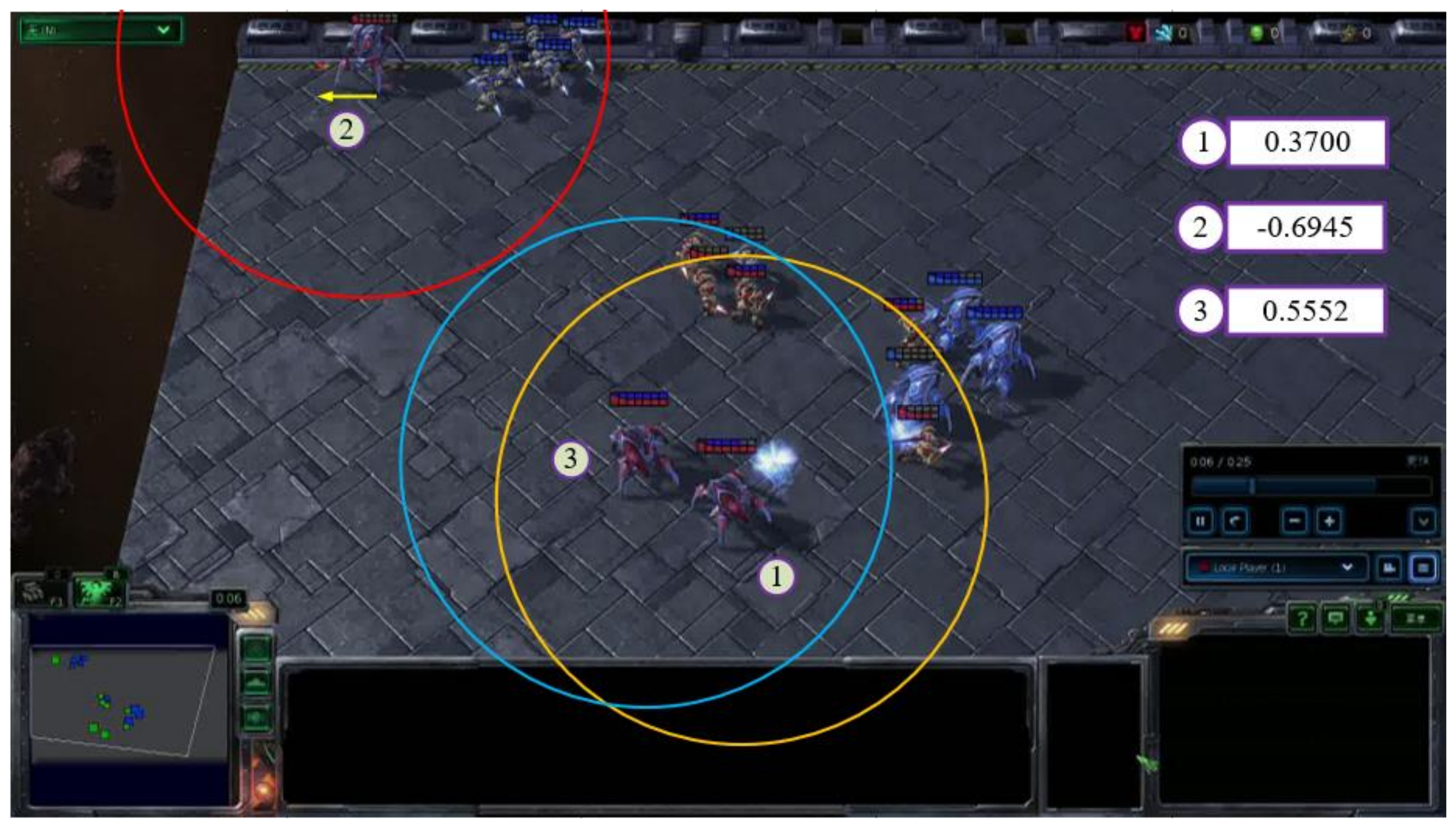}
        	    \caption{QPLEX: greedy.}
        	\label{fig:qplex_greedy_appendix}
            \end{subfigure}
            \caption{Visualisation of the evaluation for SHAQ and baselines on 3s5z\_vs\_3s6z in SMAC: each colored circle is the centered attacking range of a controllable agent (in red), and each agent's factorised Q-value is reported on the right. We mark the direction that each moving agent face by an arrow.}
        \label{fig:study_on_factorised_q_values_appendix}
        \end{figure*}
        To verify our theoretical results more firmly, we show the Q-values on a more complicated scenario in SMAC, i.e. 3s5z\_vs\_3s6z during test in Figure \ref{fig:study_on_factorised_q_values_appendix}. First, we take a look into the optimal actions. SHAQ can still demonstrate the equal credit assignment as we claimed before. Unfortunately, VDN does not explicitly show equal credit assignment. The possible reason is that part of parameters of Q-value are shared between optimal actions and suboptimal actions. Therefore, the parametric effects of the mistakes conducted on suboptimal actions to the optimal actions by VDN during learning may be exaggerated when the number of agents increases. About QMIX and QPLEX, the Q-values of optimal actions are difficult to be interpreted in this complicated scenario. For both algorithms, the agent who is responsible for kiting \footnote{\url{https://en.wikipedia.org/wiki/Glossary_of_video_game_terms}.} (i.e. Agent 3 for QMIX and Agent 2 for QPLEX) receives the lowest credit, however, it is an important role to the team in a combat tactic. Next, we focus on the demonstration of the suboptimal actions. As for SHAQ, Agent 1 and Agent 3 are participating into the battle, so deserving almost the equal credit assignment. However, Agent 2 drops teammates and escapes from the center of battle, so it contributes almost nothing to the team. As a result, it can be seen as a dummy agent and thus obtains the credit near 0. This again agrees with our theoretical analysis. About VDN, it coincidentally receives near 0 for the dummy agent (i.e. Agent 3) in this scenario. Nevertheless, the low credit assignments to the other 2 agents who participate in the battle are difficult to be interpreted. About QMIX, the agents who participate in the battle (i.e. Agent 2 and Agent 3) receive the lowest credits, while the agent (i.e. Agent 1) who escapes from the battle receives the highest credit. For QPLEX, the agents' behaviours are difficult to be interpreted.
        
    \subsection{Extra Experimental Results of Predator-Prey}
    \label{subsec:extra_experimental_results_for_predator_prey}
        In Figure \ref{fig:sqddpg_wqmix_pp_b} and Figure \ref{fig:sqddpg_wqmix_pp_c}, we show the results of W-QMIX with the annealing steps as 50k to support that the poor performance of W-QMIX on Predator-Prey is due to its poor robustness to the increased explorations.
        
        \begin{figure*}[ht!]
            \centering
            \begin{subfigure}[b]{0.45\linewidth}
                \centering
                \includegraphics[width=\textwidth]{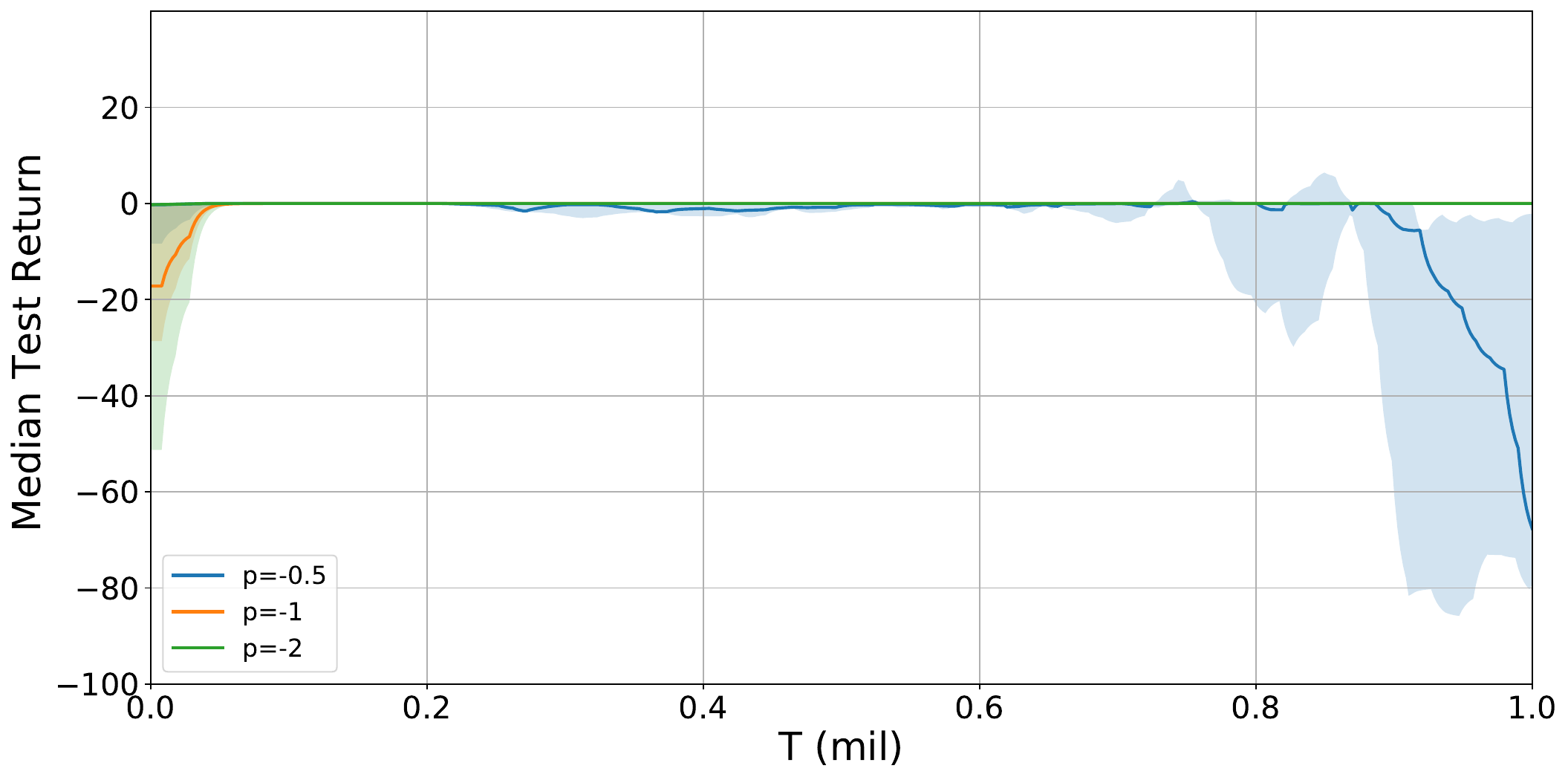}
                \caption{SQDDPG.}
            \label{fig:sqddpg_wqmix_pp_a}
            \end{subfigure}
            ~
            \begin{subfigure}[b]{0.45\linewidth}
                \centering
                \includegraphics[width=\textwidth]{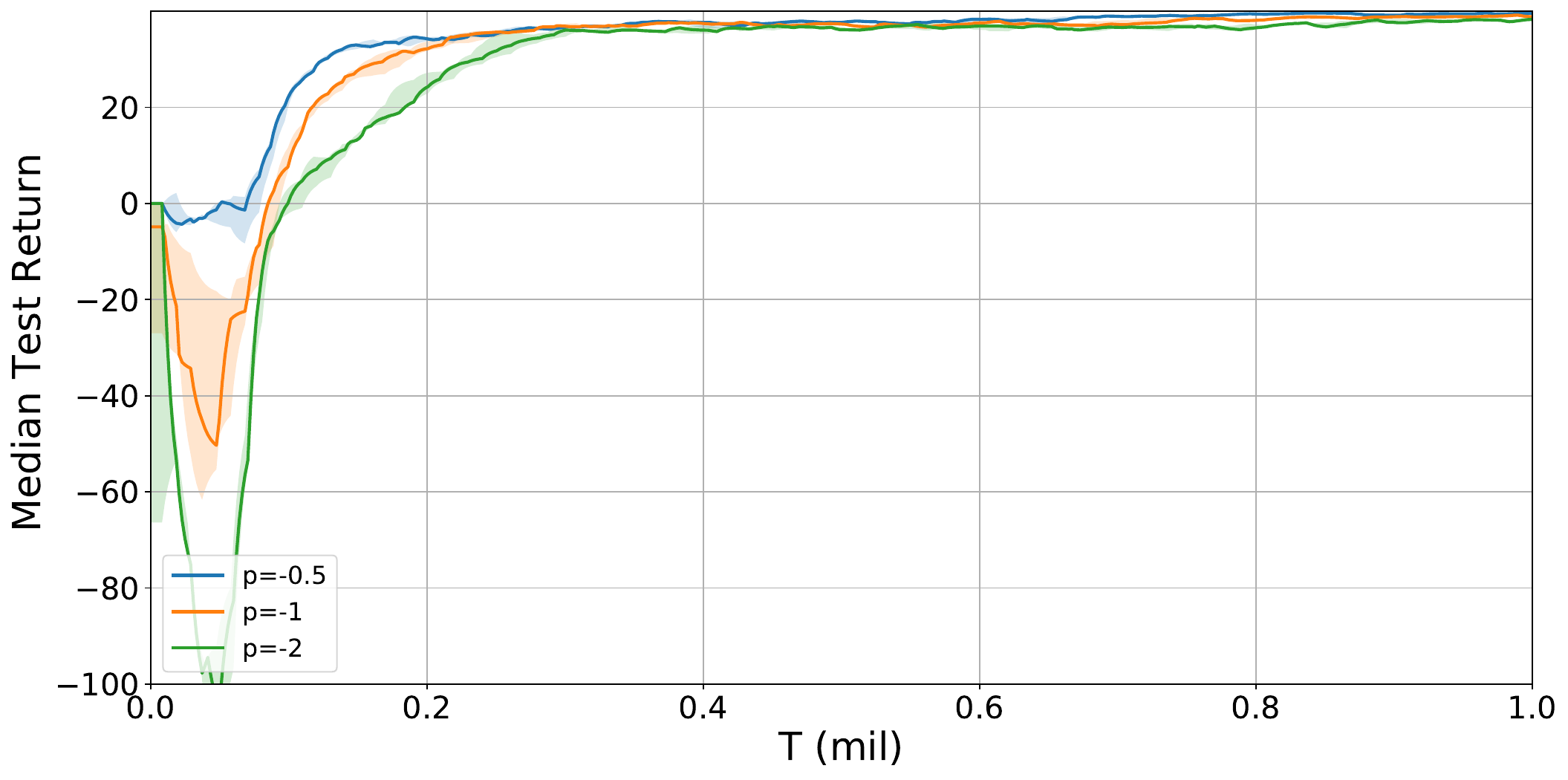}
                \caption{CW-QMIX.}
            \label{fig:sqddpg_wqmix_pp_b}
            \end{subfigure}
            ~
            \begin{subfigure}[b]{0.45\linewidth}
                \centering
                \includegraphics[width=\textwidth]{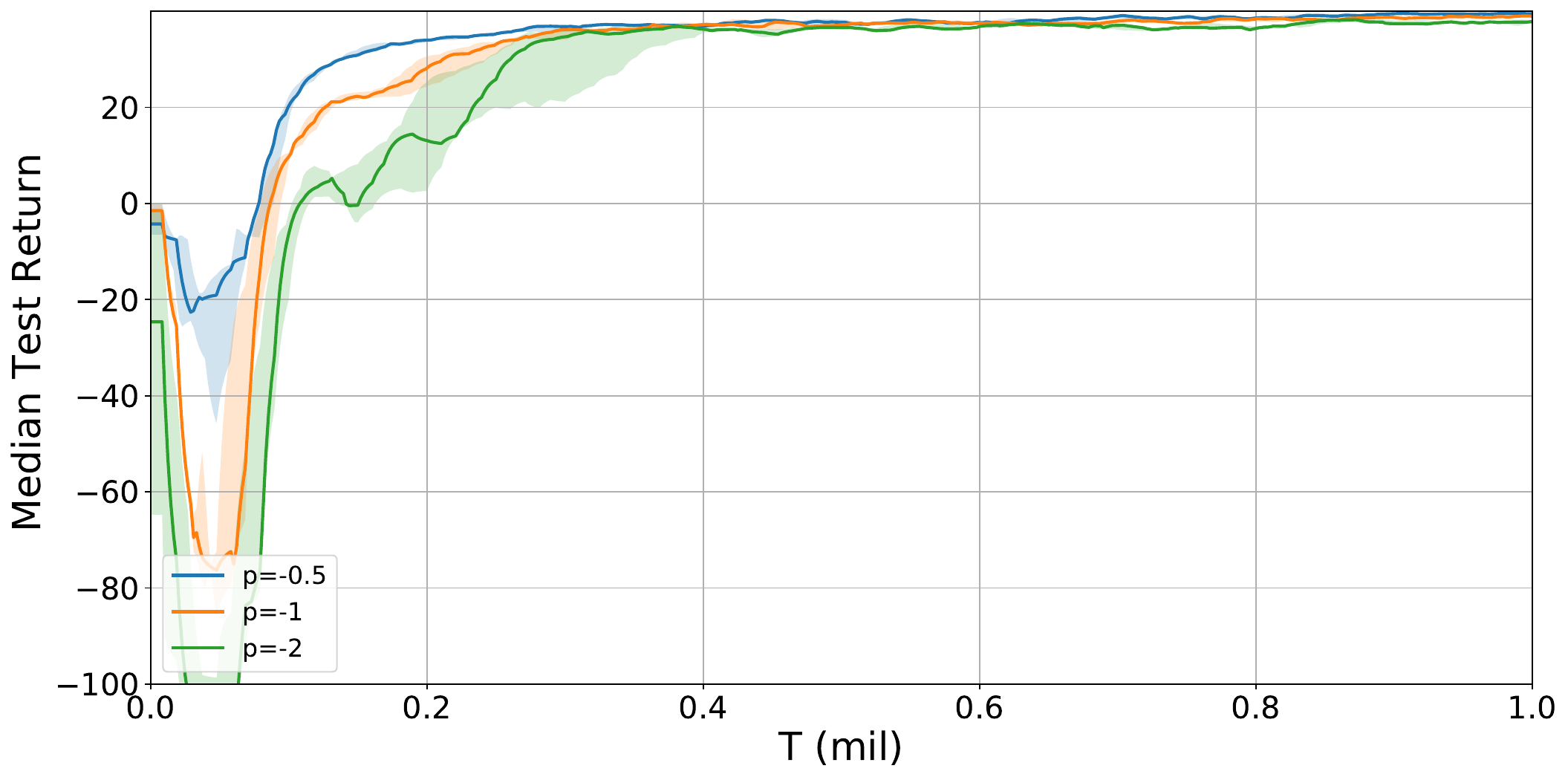}
                \caption{OW-QMIX.}
            \label{fig:sqddpg_wqmix_pp_c}
            \end{subfigure}
            \caption{Median test return for SQDDPG and W-QMIX (including OW-QMIX and CW-QMIX) on Predator-Prey.}
        \label{fig:sqddpg_wqmix_pp}
        \end{figure*}
        
\section{Additional Background}
\label{sec:additional background}
    \subsection{Value Factorisation in MARL}
    \label{subsec:value_factorisation_in_marl}
        Although there are lots of works on value factorisation in MARL, most of them are based on an assumption called Individual-Global-Max (IGM) \citep{SonKKHY19} that is defined in Definition \ref{def:igm}.
        \begin{definition}
            \label{def:igm}
            For a joint Q-value $Q^{\pi}(\mathbf{s}, \mathbf{a})$ with a deterministic policy, if the following equation is assumed to hold such that
            \begin{equation}
                \arg \max_{\mathbf{a}} Q^{\pi}(\mathbf{s}, \mathbf{a}) = \left(\arg\max_{a_{i}} Q_{i}(\mathbf{s}, a_{i}) \right)_{i=1, 2, ..., |\mathcal{N}|},
            \label{eq:igm}
            \end{equation}
            then we say that $\big(\mathit{Q}_{i}(\mathbf{s}, a_{i})\big)_{i=1, 2, ..., |\mathcal{N}|}$ satisfies Individual-Global-Max (IGM) and $\mathit{Q}^{\pi}(\mathbf{s}, \mathbf{a})$ can be factorised by $\big(\mathit{Q}_{i}(\mathbf{s}, a_{i})\big)_{i=1, 2, ..., |\mathcal{N}|}$.
        \end{definition}
        
        There are 3 popular frameworks that are followed by most of works implementing the IGM, called VDN \citep{SunehagLGCZJLSL18}, QMIX \citep{RashidSWFFW18} and QTRAN \citep{SonKKHY19}. 
        
        \textbf{VDN. } VDN linearly factorises a global value function such that
        \begin{equation}
            Q^{\pi}(\mathbf{s}, \mathbf{a}) = \sum_{i \in \mathcal{N}} Q_{i}(\mathbf{s}, a_{i}),
        \label{eq:vdn}
        \end{equation}
        so that Eq.\ref{eq:igm} holds. 
        
        \textbf{QMIX. } QMIX learns a monotonic mixing function $\mathit{f}_{\mathbf{s}}: \mathlarger{\mathlarger{\times}}_{i \in \mathcal{N}} Q_{i}(\mathbf{s}, a_{i}) \times \mathbf{s} \mapsto \mathbb{R}$ to implement the factorisation such that
        \begin{equation}
            Q^{\pi}(\mathbf{s}, \mathbf{a}) = f_{\mathbf{s}} \left( Q_{1}(\mathbf{s}, a_{1}), ..., Q_{\scriptscriptstyle|\mathcal{N}|}(\mathbf{s}, a_{\scriptscriptstyle|\mathcal{N}|}) \right),
        \label{eq:qmix}
        \end{equation}
        so that Eq.\ref{eq:igm} holds. Although QMIX has a richer functional class of factorisation than that of VDN, it meets a problem that $\max_{\mathbf{a}} Q^{\pi}(\mathbf{s}, \mathbf{a}) = \sum_{i \in \mathcal{N}} \max_{a_{i}} Q_{i}(\mathbf{s}, a_{i})$ does not necessarily hold, which may lead to the bias on Q-value estimation \citep{SonKKHY19} and affect the learning process to achieve the optimal joint policy. Theoretically, VDN does not possess the problem discussed above, however, the functional class of the simply additive factorisation is so restrictive \citep{RashidSWFFW18}.
        
        \textbf{QTRAN. } QTRAN gives a sufficient condition for value factorisation that satisfies IGM such that
        \begin{equation}
            \sum_{i \in \mathcal{N}} Q_{i}(\mathbf{s}, a_{i}) - Q^{\pi}(\mathbf{s}, \mathbf{a}) + V^{\pi}(\mathbf{s}) = 
            \begin{cases}
                 0 & \mathbf{a} = \mathbf{\bar{a}}, \\
                 \geq 0 & \mathbf{a} \neq \mathbf{\bar{a}},
            \end{cases}
        \label{eq:qtran}
        \end{equation}
        wherein
        \begin{equation*}
            V^{\pi}(\mathbf{s}) = \max_{\mathbf{a}} Q^{\pi}(\mathbf{s}, \mathbf{a}) - \sum_{i \in \mathcal{N}} Q_{i}(\mathbf{s}, \bar{a}_{i}).
        \end{equation*}
        In Eq.\ref{eq:qtran}, $\mathbf{a} = \mathlarger{\mathlarger{\times}}_{\scriptscriptstyle i \in \mathcal{N}} a_{i}$; and $\mathbf{\bar{a}} = \mathlarger{\mathlarger{\times}}_{\scriptscriptstyle i \in \mathcal{N}} \bar{a}_{i}$ where $\bar{a}_{i} = \arg\max_{a_{i}} Q_{i}(\mathbf{s}, a_{i})$ because of IGM. Additionally, \citet{SonKKHY19} showed that the above condition also holds for affine transformation on $Q_{i}, \forall i \in \mathcal{N}$ such that $w_{i} Q_{i} + b_{i}$. For this reason, an additional transformed global Q-value such that $Q^{\pi'}(\mathbf{s}, \mathbf{a}) = \sum_{i \in \mathcal{N}} Q_{i}(\mathbf{s}, a_{i})$ by setting $w_{i} = 1$ and $\sum_{i \in \mathcal{N}} b_{i} = 0$ is used to represent the value factorisation. It is forced to fit the above condition with a learned global Q-value $Q^{\pi}(\mathbf{s}, \mathbf{a})$ and $V^{\pi}(\mathbf{s})$. \citet{SonKKHY19} argued that finding the factorisation of $Q^{\pi'}(\mathbf{s}, \mathbf{a})$ is equivalent to finding $[Q_{i}]_{i \in \mathcal{N}}$ to satisfy IGM. Therefore, a value factorisation for obtaining decentralised Q-values that satisfies IGM is found.
    
    \subsection{Interpretation of Definitions in Markov Convex Game}
    \label{subsec:interpretation_of_mag_definition}
        \subsubsection{Condition of Markov Convex Game}
        \label{subsubsec:condition_of_markov_convex_game}
        Eq.\ref{eq:mcg_assumption} implies a fact existing in most real-life scenarios that a larger coalition results in the greater payoff distributions (see Remark \ref{rmk:credit_assignment}) and therefore the greater optimal global value in cooperation, which directly increases the agents' incentives for joining the grand coalition. This can be seen as an insight into the global reward game with value factorisation. This interpretation for the dynamic scenario in this paper is consistent with the static scenario given by \cite{shapley1971cores}, also known as the snowball effect.
        
        \begin{remark}
        \label{rmk:credit_assignment}
            Suppose there are two coalitions $\mathcal{T}, \mathcal{S}$ such that $\mathcal{T} \ \mathlarger{\mathlarger{\mathlarger{\subset}}} \ \mathcal{S} \ \mathlarger{\mathlarger{\mathlarger{\subset}}} \ \mathcal{N}$ and an agent $i \in \mathcal{N} \backslash \mathcal{S}$. For convenience, we denote $\mathcal{C}_{1} = \mathcal{T} \ \mathlarger{\mathlarger{\mathlarger{\cup}}} \ \{i\}$ and $\mathcal{C}_{2} = \mathcal{S}$, and thus $\mathcal{C}_{\scriptscriptstyle\cap} = \mathcal{C}_{1} \ \mathlarger{\mathlarger{\mathlarger{\cap}}} \ \mathcal{C}_{2} = (\mathcal{T} \ \mathlarger{\mathlarger{\mathlarger{\cup}}} \ \{i\}) \ \mathlarger{\mathlarger{\mathlarger{\cap}}} \ \mathcal{S} = \mathcal{T}$ and $\mathcal{C}_{\scriptscriptstyle\cup} = \mathcal{C}_{1} \ \mathlarger{\mathlarger{\mathlarger{\cup}}} \ \mathcal{C}_{2} = (\mathcal{T} \ \mathlarger{\mathlarger{\mathlarger{\cup}}} \ \{i\}) \ \mathlarger{\mathlarger{\mathlarger{\cup}}} \ \mathcal{S} = \mathcal{S} \ \mathlarger{\mathlarger{\mathlarger{\cup}}} \ \{i\}$. By Eq.\ref{eq:mcg_assumption}, we can write the following inequalities such that
            \begin{equation}
                \begin{split}
                    \max_{\pi_{\mathcal{S} \cup \{i\}}} V^{\pi_{\mathcal{S} \cup \{i\}}}(\mathbf{s}) - \max_{\pi_{\mathcal{S}}} V^{\pi_{\mathcal{S}}}(\mathbf{s}) &= \max_{\pi_{\mathcal{\mathcal{C}_{\cup}}}} V^{\pi_{\mathcal{C}_{\cup}}}(\mathbf{s}) - \max_{\pi_{\mathcal{C}_{2}}} V^{\pi_{\mathcal{C}_{2}}}(\mathbf{s}) \\
                    &\geq \max_{\pi_{\mathcal{C}_{1}}} V^{\pi_{\mathcal{C}_{1}}}(\mathbf{s}) - \max_{\pi_{\mathcal{\mathcal{C}_{\cap}}}} V^{\pi_{\mathcal{C}_{\cap}}}(\mathbf{s}) \\
                    &= \max_{\pi_{\mathcal{T} \cup \{i\}}} V^{\pi_{\mathcal{T} \cup \{i\}}}(\mathbf{s}) - \max_{\pi_{\mathcal{T}}} V^{\pi_{\mathcal{T}}}(\mathbf{s}).
                \end{split}
            \end{equation}
        It is intuitive to see that each agent can gain more payoffs if the size of the coalition grows.
        \end{remark}
        
        \subsubsection{Insight into Markov Core}
        \label{subsubsec:insight_into_the_core_appendix}
        
            In Eq.\ref{eq:epsilon_core}, $\big( \max_{\pi_{i}} x_{i}(\mathbf{s}) \big)_{i \in \mathcal{N}}$ indicates the payoff distribution scheme for the grand coalition. $\max_{\pi_{\mathcal{C}}} x(\mathbf{s}|\mathcal{C}) = \sum_{i \in \mathcal{C}} \max_{\pi_{i}} x_{i}(\mathbf{s})$ indicates the sum of payoff distributions (for the grand coalition) of the agents who is under evaluation within coalition $\mathcal{C}$. By Remark \ref{rmk:optimal_global_value} and \ref{rmk:example_core}, it is obvious that Eq.\ref{eq:epsilon_core} indicates that the optimal global value obtained by the payoff distribution scheme in the Markov core (under the grand coalition) is no less than that they can achieve with other coalition structures, which is called the maximal social welfare in the prior work \citep{Wang_2020}. It can be regarded as an intuitive interpretation of Markov core (under the grand coalition). 

            \begin{remark}
            \label{rmk:optimal_global_value}
                Suppose that a coalition structure is written as $\mathcal{CS} = \{ \mathcal{C}_{1}, \mathcal{C}_{2}, ..., \mathcal{C}_{n}\}$, \text{where} $\bigcup_{k = 1}^{n} \mathcal{C}_{k} = \mathcal{N}$ \text{and each $\mathcal{C}_{k}$ is mutually exclusive (i.e., $\mathcal{C}_{m} \ \mathlarger{\mathlarger{\mathlarger{\cap}}} \ \mathcal{C}_{n} = \emptyset, \text{if } m \neq n$)}, the optimal global value with respect to $\mathcal{CS}$ is represented as $\max_{\pi} V^{\pi}(\mathbf{s}) = \sum_{k = 1}^{n} \max_{\pi_{\scriptscriptstyle\mathcal{C}_{k}}} V^{\pi_{\scriptscriptstyle\mathcal{C}_{k}}}(\textbf{s})$.
            \end{remark}
            
            \begin{remark}
            \label{rmk:example_core}
                Suppose that the condition of Markov core holds for the grand coalition (i.e., $\mathcal{N}$) with some payoff distribution scheme $\big( \max_{\pi_{i}} x_{i}(\mathbf{s}) \big)_{i \in \mathcal{N}}$. For an arbitrary coalition structure $\mathcal{CS} = \{ \mathcal{C}_{1}, \mathcal{C}_{2}, ..., \mathcal{C}_{n}\}$ other than $\{ \mathcal{N} \}$, where $\bigcup_{k = 1}^{n} \mathcal{C}_{k} = \mathcal{N}$ and each $\mathcal{C}_{k}$ is mutually exclusive, we can write down the equation such that 
                \begin{equation}
                    \max_{\pi_{\mathcal{C}_{k}}} x(\mathbf{s} | \mathcal{C}_{k}) \geq \max_{\pi_{\mathcal{C}_{k}}} V^{\pi_{\mathcal{C}_{k}}}(\mathbf{s}), \ \ \forall \mathcal{C}_{k} \in \mathcal{CS}.
                \label{eq:example_core}
                \end{equation}
                If we sum up Eq.\ref{eq:example_core} for all coalitions in $\mathcal{CS}$, we can get the following equation such that
                \begin{equation}
                    \sum_{\mathcal{C}_{k} \in \mathcal{CS}} \max_{\pi_{\mathcal{C}_{k}}} x(\mathbf{s} | \mathcal{C}_{k}) \geq \sum_{\mathcal{C}_{k} \in \mathcal{CS}} \max_{\pi_{\mathcal{C}_{k}}} V^{\pi_{\mathcal{C}_{k}}}.
                \label{eq:example_core_sum}
                \end{equation}
                Recall that $\max_{\pi_{\mathcal{C}_{k}}} x(\mathbf{s} | \mathcal{C}_{k}) = \sum_{j \in \mathcal{C}_{k}} \max_{\pi_{i}} x_{i}(\mathbf{s})$. The LHS of Eq.\ref{eq:example_core_sum} can be written as follows:
                \begin{equation}
                    \sum_{\mathcal{C}_{k} \in \mathcal{CS}} \max_{\pi_{\mathcal{C}_{k}}} x(\mathbf{s} | \mathcal{C}_{k}) = \sum_{\mathcal{C}_{k} \in \mathcal{CS}} \sum_{j \in \mathcal{C}_{k}} \max_{\pi_{j}} x_{j}(\mathbf{s}) = \sum_{j \in \mathcal{N}} \max_{\pi_{j}} x_{j}(\mathbf{s}) = \max_{\pi} \hat{V}^{\pi}(\mathbf{s}),
                \label{eq:lhs}
                \end{equation}
                wherein $\max_{\pi} \hat{V}^{\pi}(\mathbf{s})$ is denoted as the optimal global value obtained by the payoff distribution scheme in the Markov core. By the result in Remark \ref{rmk:optimal_global_value}, the RHS of Eq.\ref{eq:example_core_sum} can be written as follows:
                \begin{equation}
                    \sum_{\mathcal{C}_{k} \in \mathcal{CS}} \max_{\pi_{\mathcal{C}_{k}}} V^{\pi_{\mathcal{C}_{k}}} = \max_{\pi} V^{\pi}(\mathbf{s}),
                \label{eq:rhs}
                \end{equation}
                where $\max_{\pi} V^{\pi}(\mathbf{s})$ is the optimal global value obtained by an arbitrary coalition structure other than $\{\mathcal{N}\}$. By inserting Eq.\ref{eq:lhs} and \ref{eq:rhs} into Eq.\ref{eq:example_core_sum}, we can get that
                \begin{equation*}
                    \max_{\pi} \hat{V}^{\pi}(\mathbf{s}) \geq \max_{\pi} V^{\pi}(\mathbf{s}).
                \end{equation*}
                Therefore, we have shown that the solution in the Markov core under the grand coalition is equivalent to the optimal global value.
            \end{remark}

\section{Complete Mathematical Proofs}
\label{sec:complete_mathematical_proofs}
    
    \subsection{Assumptions}
    \label{subsec:assumptions}
        
        \begin{assumption}
        \label{assm:finite_markov_convex_game}
            In this paper, we consider a finite Markov convex game, wherein both the state space and the joint action space are finite.
        \end{assumption}
        
        \begin{assumption}
        \label{assm:agent_policy_assumption}
            For the ease of analysis, in this paper we assume that each agent's policy will not be affected by the coalition formation. In other words, each agent's policy is regarded as its inherent feature, invariant throughout the interaction with other agents (e.g. joining a coalition).
        \end{assumption}
        
        \begin{assumption}
        \label{assm:assumption_for_joint_policy_factorisation}
            Any coalition policy can be factorised to a permutation of decentralised (i.e. disjoint) policies, i.e., $\pi_{\scriptscriptstyle\mathcal{C}} = \mathlarger{\mathlarger{\times}}_{\scriptscriptstyle i \in \mathcal{C}} \pi_{i}$, where $\pi_{i}$ is agent $i$'s policy. Each $\pi_{\scriptscriptstyle\mathcal{C}}$ uniquely corresponds to a $V^{\pi_{\mathcal{C}}}(\mathbf{s})$ as a characteristic function (i.e. a set-valued function).
        \end{assumption}
        
        \begin{assumption}
        \label{assm:dummy_agent}
            If an agent $\mathit{i}$ is a dummy for an arbitrary state $\mathbf{s} \in \mathcal{S}$, it will not provide any contribution to any coalition $\mathcal{C}_{i} \ \mathlarger{\mathlarger{\mathlarger{\subseteq}}} \ \mathcal{N} \backslash \{i\}$ such that $V^{\pi_{\mathcal{C}}}(\mathbf{s}) = V^{\pi_{\mathcal{C} \cup \{i\}}}(\mathbf{s})$. Additionally, no members in coalition $\mathcal{C}_{i}$ will react in different manners after agent $\mathit{i}$ joins.
        \end{assumption}
        
        \begin{assumption}
            If agents $i$ and $j$ are symmetric for an arbitrary state $\mathbf{s} \in \mathcal{S}$, $V^{\pi_{\mathcal{C} \cup \{i\}}}(\mathbf{s}) = V^{\pi_{\mathcal{C} \cup \{j\}}}(\mathbf{s})$ to any coalitions $\mathcal{C} \ \mathlarger{\mathlarger{\mathlarger{\subseteq}}} \ \mathcal{N} \backslash \{i, j\}$. Literally, the contributions of $i$ and $j$ are equal to any coalition $\mathcal{C}$.
        \label{assm:symmetry}
        \end{assumption}
        
        \begin{assumption}
        \label{assm:max_shapley_value}
            For any agent $\mathit{i} \in \mathcal{N}$ and any $\mathbf{s} \in \mathcal{S}$, its optimal Markov Shapley value denoted as $\max_{\pi_{i}} V_{i}^{\phi}(\mathbf{s})$ satisfies the following equation such that
            \begin{equation*}
                \max_{\pi_{i}} V_{i}^{\phi}(\mathbf{s}) = \sum_{\mathcal{C}_{i} \ \mathlarger{\mathlarger{\subseteq}} \ \mathcal{N} \backslash \{i\} } \frac{|\mathcal{C}_{i}|!(|\mathcal{N}|-|\mathcal{C}_{i}|-1)!}{|\mathcal{N}|!} \cdot \max_{\pi_{i}} \Phi_{i}(\mathbf{s} | \mathcal{C}_{i}),
            \end{equation*}
            where $\pi_{i}$ is agent $\mathit{i}$'s policy.
        \end{assumption}
        
        Assumption \ref{assm:finite_markov_convex_game} is the common assumption in the Markov decision process for the ease of analysis. Assumption \ref{assm:agent_policy_assumption} is a technical assumption for the ease of analysis. Assumption \ref{assm:assumption_for_joint_policy_factorisation} is natural to hold given the chain rule in probability theory, the independence of each agent's policy and the definition of value function in reinforcement learning. Assumption \ref{assm:dummy_agent} and \ref{assm:symmetry} directly inherit the definitions from cooperative game theory \citep{chalkiadakis2011computational}. Assumption \ref{assm:max_shapley_value} inherits the definition from Shapley value \citep{shapley1953value} with extra consideration of agent $i$'s policy, an underlying condition of which is that the maximizer (i.e., $\pi_{i}$) of each $\Phi_{i}(\mathbf{s} \mid \mathcal{C}_{i}) \in \left\{ \Phi_{i}(\mathbf{s} | \mathcal{C}_{i}) | \mathcal{C}_{i} \ \mathlarger{\mathlarger{\mathlarger{\subseteq}}} \ \mathcal{N} \backslash \{i\} \right\}$ needs to be identical, for any $\mathbf{s} \in \mathcal{S}$. In other words, it implies that different permutations correspond to different long-term rewards probably encoding some unexpected events (i.e., each permutation maps to a marginal contribution of agent $i$), but with the same optimal policy as solutions, which is a sufficient condition for Assumption \ref{assm:agent_policy_assumption}. Thereby, learning through Markov Shapley value is primarily for fair credit assignments, with no changes to each agent's optimal policy. We would argue for the existence of this condition by Example \ref{exp:max_shapley_value}.
        
        \begin{example}
        \label{exp:max_shapley_value}
            Suppose that there are two agents in total (i.e., $|\mathcal{N}|=2$), and we consider an arbitrary agent $i$ belonging to $\mathcal{N}$ whose action set is defined as $\mathcal{A}_{i} = \{ 0, 0.15, 0.25 \}$. Therefore, there are only two intermediate coalitions for agent $i$ to join and therefore two marginal contributions. To ease life, we only discuss a two-stage scenario and the result can be naturally extended to long-horizon scenarios. Agent $i$'s policy can be expressed as a sequence of actions such that $\pi_{i} = \left\langle a_{i}^{0}, a_{i}^{1} \right\rangle$. The set of marginal contributions of agent $i$ is supposed to be $\left\{ \Phi_{i}(\mathbf{s} | \{-i\}) := - (a_{i}^{0} + a_{i}^{1} - 0.5)^{2} + 1 + ||\mathbf{s}||_{2}, \Phi_{i}(\mathbf{s} | \emptyset) := \sin (a_{i}^{0} + a_{i}^{1}) + ||\mathbf{s}||_{2} \right\}$. Since $V_{i}^{\phi}(\mathbf{s}) = \frac{1}{2} \left( \Phi_{i}(\mathbf{s} | \{-i\}) + \Phi_{i}(\mathbf{s} | \emptyset) \right)$, it is easy to observe that Assumption \ref{assm:max_shapley_value} holds.
        \end{example}
        
    \subsection{Mathematical Proofs of The Marginal Contribution}
    \label{subsec:preliminary_theoretical_results}
        \setcounter{proposition}{3}
        \begin{proposition}
        \label{prop:condition_coalition_marginal_contribution}
            $\forall \mathcal{C}_{i} \ \mathlarger{\mathlarger{\mathlarger{\subseteq}}} \ \mathcal{N}$ and $\forall \mathbf{s} \in \mathcal{S}$, Eq.\ref{eq:mcg_assumption} is satisfied if and only if $\max_{\pi_{i}} \Phi_{i}(\mathbf{s}|\mathcal{C}_{i}) \geq 0$.
        \end{proposition}
        
        \begin{proof}
            $\forall \mathcal{C}_{i} \ \mathlarger{\mathlarger{\mathlarger{\subseteq}}} \ \mathcal{N}$ and $\forall \mathbf{s} \in \mathcal{S}$, given that Eq.\ref{eq:mcg_assumption} is satisfied, with the fact that $\mathcal{C}_{i} \ \mathlarger{\mathlarger{\mathlarger{\cap}}} \ \{i\} = \emptyset$ we can get the equation such that
            \begin{equation}
                \begin{split}
                    \max_{\pi_{\mathcal{C}_{i} \cup \{i\}}} V^{\pi_{\mathcal{C}_{i} \cup \{i\}}}(\mathbf{s}) \geq \max_{\pi_{\mathcal{C}_{i}}} V^{\pi_{\mathcal{C}_{i}}}(\mathbf{s})
                    + \max_{\pi_{i}} V^{\pi_{i}}(\mathbf{s}).
                \end{split}
            \label{eq:condition_1}
            \end{equation}
            Since $\max_{\pi_{i}} V^{\pi_{i}}(\mathbf{s}) \geq 0$ by the definition in Markov convex game, we can easily get the equation such that
            \begin{equation}
                \begin{split}
                    \max_{\pi_{\mathcal{C}_{i} \cup \{i\}}} V^{\pi_{\mathcal{C}_{i} \cup \{i\}}}(\mathbf{s}) - \max_{\pi_{\mathcal{C}_{i}}} V^{\pi_{\mathcal{C}_{i}}}(\mathbf{s}) \geq 0.
                \end{split}
            \label{eq:condition_2}
            \end{equation}
            Therefore, we can get the equation such that
            \begin{equation}
                \begin{split}
                    \max_{\pi_{i}} \Phi_{i}(\mathbf{s}|\mathcal{C}_{i}) \geq 0.
                \end{split}
            \label{eq:condition_3}
            \end{equation}
            With the same conditions, the reverse direction of proof apparently holds by going through from Eq.\ref{eq:condition_3} to \ref{eq:condition_1}. By Definition \ref{def:shapley_value}, Eq.\ref{eq:condition_3} determines the range of Markov Shapley value, which is consistent with the range of the coalition value defined in Section \ref{sec:markov_convex_game}.
        \end{proof}
        
        \setcounter{proposition}{4}
        \begin{proposition}
        \label{prop:marginal_contribution_equal_value_factorisation}
            In Markov convex game with the grand coalition, marginal contribution satisfies the efficiency property: $\max_{\pi} V^{\pi}(\mathbf{s}) = \sum_{i \in \mathcal{N}} \max_{\pi_{i}} \Phi_{i}(\mathbf{s}|\mathcal{C}_{i})$.
        \end{proposition}
        
        \begin{proof}
            For any $\mathcal{C}_{i} \ \mathlarger{\mathlarger{\mathlarger{\subseteq}}} \ \mathcal{N} \backslash \{i\}$ and $\mathit{i} \in \mathcal{N}$, according to Eq.\ref{eq:marginal_contribution_v} we can get the equation such that
            \begin{equation}
                \begin{split}
                    \max_{\pi_{i}} \Phi_{i}(\mathbf{s} | \mathcal{C}_{i}) = \max_{\pi_{\mathcal{C}_{i} \cup \{i\}}} V^{\pi_{\mathcal{C}_{i} \cup \{i\}}}(\mathbf{s})
                    - \max_{\pi_{\mathcal{C}_{i}}} V^{\pi_{\mathcal{C}_{i}}}(\mathbf{s}),
                \end{split}
            \label{eq:coalition_marginal_contribution_max}
            \end{equation}
            where $\max_{\pi_{\mathcal{C}_{i} \cup \{i\}}} V^{\pi_{\mathcal{C}_{i}}}(\mathbf{s})=\max_{\pi_{\mathcal{C}_{i}}} V^{\pi_{\mathcal{C}_{i}}}(\mathbf{s})$, since the decision of agent $\mathit{i}$ will not affect the value of $\mathcal{C}_{i}$ (i.e., the coalition excluding agent $\mathit{i}$).
            Given the definition that $V^{\pi_{\emptyset}}(\mathbf{s})=0$ and the result from Eq.\ref{eq:coalition_marginal_contribution_max}, by Assumption \ref{assm:assumption_for_joint_policy_factorisation} we can get the equations such that
            \begin{align}
                &\quad \ \max_{\pi} V^{\pi}(\mathbf{s}) \nonumber \\ 
                &= \max_{\pi_{\{j_{1}\}}} V^{\pi_{\{j_{1}\}}}(\mathbf{s}) - \max_{\pi_{\emptyset}} V^{\pi_{\emptyset}}(\mathbf{s}) \nonumber \\
                &+ \max_{\pi_{\{j_{1}, j_{2}\}}} V^{\pi_{\{j_{1}\}}}(\mathbf{s}) - \max_{\pi_{\{j_{1}\}}} V^{\pi_{\{j_{1}\}}}(\mathbf{s}) \nonumber \\
                &+ \qquad \qquad \qquad \qquad \vdots \nonumber \\
                &+ \max_{\pi} V^{\pi}(\mathbf{s}) - \max_{\pi_{\mathcal{N} \backslash \{j_{n}\}}} V^{\pi_{\mathcal{N} \backslash \{j_{n}\}}}(\mathbf{s}) \nonumber \\
                &= \sum_{i \in \mathcal{N}} \max_{\pi_{i}} \Phi_{i}(\mathbf{s}|\mathcal{C}_{i}).
            \end{align}
        \end{proof}

        \begin{lemma}
        \label{lemm:condition_coalition_marginal_contribution}
            The optimal marginal contribution is a solution in the Markov core under Markov convex game with the grand coalition.
        \end{lemma}
        
        \begin{proof}
            The complete proof is as follows.
            
            Firstly, if we would like to prove that the optimal marginal contribution is a payoff distribution scheme in the Markov core (with the grand coalition), we just need to prove that for any intermediate coalition $\mathcal{C} \ \mathlarger{\mathlarger{\mathlarger{\subseteq}}} \ \mathcal{N}$, the following condition is satisfied such that
            \begin{equation}
                \max_{\pi_{\mathcal{C}}} \Phi(\mathbf{s}|\mathcal{C}) \geq \max_{\pi_{\mathcal{C}}} V^{\pi_{\mathcal{C}}}(\mathbf{s}), \ \forall \mathbf{s} \in \mathcal{S},
            \label{eq:appendix_core}
            \end{equation}
            where $\max_{\pi_{\mathcal{C}}} \Phi(\mathbf{s}|\mathcal{C}) = \sum_{i \in \mathcal{C}} \max_{\pi_{i}} \Phi_{i}(\mathbf{s}|\mathcal{C}_{i})$. 
            
            Suppose for the sake of contradiction that we have $\max_{\pi_{\mathcal{C}}} \Phi(\mathbf{s}|\mathcal{C}) < \max_{\pi_{\mathcal{C}}} V^{\pi_{\mathcal{C}}}(\mathbf{s})$ for some $\mathbf{s} \in \mathcal{S}$ and some coalition $\mathcal{C} = \{ j_{1}, j_{2}, ..., j_{|\mathcal{C}|} \} \ \mathlarger{\mathlarger{\mathlarger{\subseteq}}} \ \mathcal{N}$, where $\mathit{j}_{n} \in \mathcal{C}$ and $n \in \{1, 2, ..., |\mathcal{C}|\}$. 
            We can assume without the loss of generality that the coalition $\mathcal{C}$ is generated by the permutation $\langle j_{1}, j_{2}, ..., j_{|\mathcal{C}|} \rangle$, i.e., the agents joins in $\mathcal{C}$ following the order $j_{1}, j_{2}, ..., j_{|\mathcal{C}|}$. Now, for each $n \in \{1, 2, ..., |\mathcal{C}|\}$, we have $\left\{ j_{1}, j_{2}, ..., j_{n-1} \right\} \ \mathlarger{\mathlarger{\mathlarger{\subseteq}}} \ \left\{ 1, 2, ..., j_{n}-1 \right\}$. Following Eq.\ref{eq:mcg_assumption}, we can write out the inequality as follows:
            \begin{equation}
                \begin{split}
                    \max_{\pi_{\mathcal{C}_{\cup}^{n}}} V^{\pi_{\mathcal{C}_{\cup}^{n}}}(\mathbf{s}) +
                    \max_{\pi_{\mathcal{C}_{\cap}^{n}}} V^{\pi_{\mathcal{C}_{\cap}^{n}}}(\mathbf{s}) \geq
                    \max_{\pi_{\mathcal{C}_{m}^{n}}} V^{\pi_{\mathcal{C}_{m}^{n}}}(\mathbf{s}) + \max_{\pi_{\mathcal{C}_{k}^{n}}} V^{\pi_{\mathcal{C}_{k}^{n}}}(\mathbf{s}),\\
                    \mathcal{C}_{k}^{n} = \{ 1, 2, ..., j_{n}-1 \}, \\
                    \mathcal{C}_{m}^{n} = \{ j_{1}, j_{2}, ..., j_{n} \}, \\
                    \mathcal{C}_{\cap}^{n} = \mathcal{C}_{m}^{n} \ \mathlarger{\mathlarger{\cap}} \ \mathcal{C}_{k}^{n} = \{ j_{1}, j_{2}, ..., j_{n-1} \}, \\
                    \mathcal{C}_{\cup}^{n} = \mathcal{C}_{m}^{n} \ \mathlarger{\mathlarger{\cup}} \ \mathcal{C}_{k}^{n} = \{ 1, 2, ..., j_{n} \}. \\
                \end{split}
            \label{eq:appendix_ecg}
            \end{equation}
            
            Next, we rearrange Eq.\ref{eq:appendix_ecg} and the following inequality is obtained such that
            \begin{equation}
                \begin{split}
                    \max_{\pi_{\mathcal{C}_{\cup}^{n}}} V^{\pi_{\mathcal{C}_{\cup}^{n}}}(\mathbf{s}) - \max_{\pi_{\mathcal{C}_{k}^{n}}} V^{\pi_{\mathcal{C}_{k}^{n}}}(\mathbf{s}) \geq
                    \max_{\pi_{\mathcal{C}_{m}^{n}}} V^{\pi_{\mathcal{C}_{m}^{n}}}(\mathbf{s}) - \max_{\pi_{\mathcal{C}_{\cap}^{n}}} V^{\pi_{\mathcal{C}_{\cap}^{n}}}(\mathbf{s}),\\
                \end{split}
            \label{eq:appendix_rearrange_ecg}
            \end{equation}
            
            Since we can express $\max_{\pi_{\mathcal{C}}} V^{\pi_{\mathcal{C}}}(\mathbf{s})$ as follows:
            \begin{align}
                \max_{\pi_{\mathcal{C}}} V^{\pi_{\mathcal{C}}}(\mathbf{s}) &= \max_{\pi_{j_{1}}} V^{\pi_{j_{1}}}(\mathbf{s}) - \max_{\pi_{\emptyset}} V^{\pi_{\emptyset}}(\mathbf{s}) \nonumber \\
                &+ \max_{\pi_{\{j_{1}, j_{2}\}}} V^{\pi_{\{j_{1}, j_{2}\}}}(\mathbf{s}) - \max_{\pi_{j_{1}}} V^{\pi_{j_{1}}}(\mathbf{s}) \nonumber \\
                &+ \qquad \qquad \qquad \qquad \vdots \nonumber \\
                &+ \max_{\pi_{\mathcal{C}}} V^{\pi_{\mathcal{C}}}(\mathbf{s}) - \max_{\pi_{\mathcal{C} \backslash \{j_{n}\}}} V^{\pi_{\mathcal{C} \backslash \{j_{n}\}}}(\mathbf{s}).
            \label{eq:appendix_prop1_contradiction_-1}
            \end{align}
            By Definition \ref{def:marginal_contribution} we can obviously get the following equations such that
            \begin{align}
                \Phi_{i}(\mathbf{s}|\mathcal{C}_{i}) = \Phi_{i}(\mathbf{s}|\mathcal{C}_{k}^{n}) &= \max_{\pi_{\mathcal{C}_{k}^{n}}} V^{\pi_{\mathcal{C}_{\cup}^{n}}}(\mathbf{s}) - \max_{\pi_{\mathcal{C}_{k}^{n}}} V^{\pi_{\mathcal{C}_{k}^{n}}}(\mathbf{s}).
            \label{eq:appendix_prop1_contradiction_0}
            \end{align}
            By taking the maximum operator over $\pi_{i}$ to Eq.\ref{eq:appendix_prop1_contradiction_0}, we can get that
            \begin{align}
                \max_{\pi_{i}} \Phi_{i}(\mathbf{s}|\mathcal{C}_{i}) = \max_{\pi_{i}} \Phi_{i}(\mathbf{s}|\mathcal{C}_{k}^{n}) = \max_{\pi_{\mathcal{C}_{\cup}^{n}}} V^{\pi_{\mathcal{C}_{\cup}^{n}}}(\mathbf{s}) - \max_{\pi_{\mathcal{C}_{k}^{n}}} V^{\pi_{\mathcal{C}_{k}^{n}}}(\mathbf{s}).
            \label{eq:appendix_prop1_contradiction_1}
            \end{align}
            By adding up these inequalities in Eq.\ref{eq:appendix_rearrange_ecg} for all $\mathcal{C} \ \mathlarger{\mathlarger{\mathlarger{\subseteq}}} \ \mathcal{N}$ and inserting the results from Eq.\ref{eq:appendix_prop1_contradiction_-1} and \ref{eq:appendix_prop1_contradiction_1}, we can directly obtain a new inequality such that
            \begin{equation}
                \sum_{i \in \mathcal{C}} \max_{\pi_{i}} \Phi_{i}(\mathbf{s}|\mathcal{C}_{i}) = \max_{\pi_{\mathcal{C}}} \Phi(\mathbf{s}|\mathcal{C}) \geq \max_{\pi_{\mathcal{C}}} V^{\pi_{\mathcal{C}}}(\mathbf{s}).
            \label{eq:appendix_prop1_contradiction}
            \end{equation}
            It is obvious that Eq.\ref{eq:appendix_prop1_contradiction} contradicts the suppose, so we have showed that Eq.\ref{eq:appendix_core} always holds for any coalition $\mathcal{C} \ \mathlarger{\mathlarger{\mathlarger{\subseteq}}} \ \mathcal{N}$. For this reason, we can get the conclusion that marginal contribution is a solution in Markov core of Markov convex game with the grand coalition.
        \end{proof}
    
    \subsection{Mathematical Proofs of The Markov Shapley Value}
    \label{subsec:math_proofs_for_section_generalised_shapley_value}
        \begingroup
            \def\theproposition{\ref{prop:optimal_action_coalition_marginal_contribution}}
            \begin{proposition}
                Agent $i$'s action marginal contribution can be derived as follows:
                \begin{equation}
                    \begin{split}
                        \Upphi_{i}(\mathbf{s}, a_{i} | \mathcal{C}_{i}) 
                        = \max_{\mathbf{a}_{\mathcal{C}_{i}}} Q^{\pi_{\mathcal{C}_{i}}^{*}}(\mathbf{s}, \mathbf{a}_{\scriptscriptstyle \mathcal{C}_{i} \cup \{i\}})
                        - \max_{\mathbf{a}_{\mathcal{C}_{i}}} Q^{\pi_{\mathcal{C}_{i}}^{*}}(\mathbf{s}, \mathbf{a}_{\scriptscriptstyle \mathcal{C}_{i}}).
                    \end{split}
                \end{equation}
            \end{proposition}
        \endgroup
        
        \begin{proof}
            The complete proof is as follows.
            
            We now rewrite $\max_{\pi_{\mathcal{C}_{i}}} V^{\pi_{\mathcal{C}_{i} \cup \{i\}}}(\mathbf{s})$ as follows:
            \begin{align}
                \max_{\pi_{\mathcal{C}_{i}}} V^{\pi_{\mathcal{C}_{i} \cup \{i\}}}(\mathbf{s}) &= \max_{\pi_{\mathcal{C}_{i}}} \sum_{\mathbf{a}_{\scriptscriptstyle\mathcal{C}_{i} \cup \{i\}}} \pi_{\scriptscriptstyle\mathcal{C}_{i} \cup \{i\}}(\mathbf{a}_{\scriptscriptstyle\mathcal{C}_{i} \cup \{i\}} | \mathbf{s}) \ Q^{\pi_{\mathcal{C}_{i} \cup \{i\}}}(\mathbf{s}, \mathbf{a}_{\scriptscriptstyle\mathcal{C}_{i} \cup \{i\}}) \nonumber \\
                &\quad \big( \text{Since $\pi_{\scriptscriptstyle\mathcal{C}_{i} \cup \{i\}}$ is a deterministic joint policy, we can have the following equation.} \big) \nonumber \\
                &= \max_{\mathbf{a}_{\mathcal{C}_{i}}} \max_{\pi_{\mathcal{C}_{i}}} Q^{\pi_{\mathcal{C}_{i} \cup \{i\}}}(\mathbf{s}, \mathbf{a}_{\scriptscriptstyle\mathcal{C}_{i} \cup \{i\}}) \nonumber \\
                &\quad \big( \ \text{We write $\max_{\pi_{\mathcal{C}_{i}}} Q^{\pi_{\mathcal{C}_{i} \cup \{i\}}}(\mathbf{s}, \mathbf{a}_{\scriptscriptstyle\mathcal{C}_{i} \cup \{i\}})$ as $Q^{\pi_{\mathcal{C}_{i}}^{*}}(\mathbf{s}, \mathbf{a}_{\scriptscriptstyle\mathcal{C}_{i} \cup \{i\}})$} \ \big) \nonumber \\
                &= \max_{\mathbf{a}_{\mathcal{C}_{i}}} Q^{\pi_{\mathcal{C}_{i}}^{*}}(\mathbf{s}, \mathbf{a}_{\scriptscriptstyle\mathcal{C}_{i} \cup \{i\}}).
                \label{eq:V_rewirte_1}
            \end{align}
            
            Similarly, we rewrite $\max_{\pi_{\mathcal{C}_{i}}} V^{\pi_{\mathcal{C}_{i}}}(\mathbf{s})$ as follows:
            \begin{align}
                \max_{\pi_{\mathcal{C}_{i}}} V^{\pi_{\mathcal{C}_{i}}}(\mathbf{s}) = \max_{\mathbf{a}_{\mathcal{C}_{i}}} \max_{\pi_{\mathcal{C}_{i}}} Q^{\pi_{\mathcal{C}_{i}}}(\mathbf{s}, \mathbf{a}_{\scriptscriptstyle\mathcal{C}_{i}}) = \max_{\mathbf{a}_{\mathcal{C}_{i}}} Q^{\pi_{\mathcal{C}_{i}}^{*}}(\mathbf{s}, \mathbf{a}_{\scriptscriptstyle\mathcal{C}_{i}}).
                \label{eq:V_rewrite_2}
            \end{align}
            
            Since $\max_{\pi_{\mathcal{C}_{i}}} V^{\pi_{\mathcal{C}_{i}}}(\mathbf{s})$ is irrelevant to $\mathit{a}_{i}$, by Eq.\ref{eq:V_rewirte_1} and \ref{eq:V_rewrite_2} we can get that
            \begin{equation}
                \Upphi_{i}(\mathbf{s}, a_{i}|\mathcal{C}_{i}) = \max_{\mathbf{a}_{\mathcal{C}_{i}}} Q^{\pi_{\mathcal{C}_{i}}^{*}}(\mathbf{s}, \mathbf{a}_{\scriptscriptstyle\mathcal{C}_{i} \cup \{i\}})
                - \max_{\mathbf{a}_{\mathcal{C}_{i}}} Q^{\pi_{\mathcal{C}_{i}}^{*}}(\mathbf{s}, \mathbf{a}_{\scriptscriptstyle\mathcal{C}_{i}}).
            \label{eq:coalition_marginal_contribution_action}
            \end{equation}
            
            By Eq.\ref{eq:coalition_marginal_contribution_action}, we can also get Agent $i$'s optimal action marginal contribution such that
            \begin{align}
                \Upphi_{i}^{*}(\mathbf{s}, a_{i}|\mathcal{C}_{i}) &= \max_{\pi_{i}} \Upphi_{i}(\mathbf{s}, a_{i}|\mathcal{C}_{i}) \nonumber \\
                &= \max_{\pi_{i}} \bigg\{ \max_{\mathbf{a}_{\mathcal{C}_{i}}} Q^{\pi_{\mathcal{C}_{i}}^{*}}(\mathbf{s}, \mathbf{a}_{\scriptscriptstyle\mathcal{C}_{i} \cup \{i\}})
                - \max_{\mathbf{a}_{\mathcal{C}_{i}}} Q^{\pi_{\mathcal{C}_{i}}^{*}}(\mathbf{s}, \mathbf{a}_{\scriptscriptstyle\mathcal{C}_{i}}) \bigg\} \nonumber \\
                &= \max_{\pi_{i}} \bigg\{ \max_{\mathbf{a}_{\mathcal{C}_{i}}} \max_{\pi_{\mathcal{C}_{i}}} Q^{\pi_{\mathcal{C}_{i} \cup \{i\}}}(\mathbf{s}, \mathbf{a}_{\scriptscriptstyle\mathcal{C}_{i} \cup \{i\}}) -  \max_{\mathbf{a}_{\mathcal{C}_{i}}} \max_{\pi_{\mathcal{C}_{i}}} Q^{\pi_{\mathcal{C}_{i}}}(\mathbf{s}, \mathbf{a}_{\scriptscriptstyle\mathcal{C}_{i}}) \bigg\} \nonumber \\
                &= \max_{\pi_{i}} \max_{\mathbf{a}_{\mathcal{C}_{i}}} \max_{\pi_{\mathcal{C}_{i}}} Q^{\pi_{\mathcal{C}_{i} \cup \{i\}}}(\mathbf{s}, \mathbf{a}_{\scriptscriptstyle\mathcal{C}_{i} \cup \{i\}}) - \max_{\mathbf{a}_{\mathcal{C}_{i}}} \max_{\pi_{\mathcal{C}_{i}}} Q^{\pi_{\mathcal{C}_{i}}}(\mathbf{s}, \mathbf{a}_{\scriptscriptstyle\mathcal{C}_{i}}) \nonumber \\
                &= \max_{\mathbf{a}_{\mathcal{C}_{i}}} \max_{\pi_{\mathcal{C}_{i} \cup \{i\}}} Q^{\pi_{\mathcal{C}_{i} \cup \{i\}}}(\mathbf{s}, \mathbf{a}_{\scriptscriptstyle\mathcal{C}_{i} \cup \{i\}}) - \max_{\mathbf{a}_{\mathcal{C}_{i}}} \max_{\pi_{\mathcal{C}_{i}}} Q^{\pi_{\mathcal{C}_{i}}}(\mathbf{s}, \mathbf{a}_{\scriptscriptstyle\mathcal{C}_{i}}) \nonumber \\
                &= \max_{\mathbf{a}_{\mathcal{C}_{i}}} Q^{\pi_{\mathcal{C}_{i} \cup \{i\}}^{*}}(\mathbf{s}, \mathbf{a}_{\scriptscriptstyle\mathcal{C}_{i} \cup \{i\}})
                - \max_{\mathbf{a}_{\mathcal{C}_{i}}} Q^{\pi_{\mathcal{C}_{i}}^{{*}}}(\mathbf{s}, \mathbf{a}_{\scriptscriptstyle\mathcal{C}_{i}}).
            \label{eq:optimal_marginal_q}
            \end{align}
        The proof is completed.
        \end{proof}
        
        \begingroup
            \def\theproposition{\ref{prop:shapley_value_properties}}
            \begin{proposition}
                Markov Shapley value possesses properties as follows: (i) identifiability of dummy agents: $V_{i}^{\phi}(\mathbf{s}) = 0$; (ii) efficiency: $\max_{\pi} V^{\pi}(\mathbf{s}) = \sum_{i \in \mathcal{N}} \max_{\pi_{i}} V_{i}^{\phi}(\mathbf{s})$; (iii) reflecting the contribution; and (iv) symmetry.
            \end{proposition}
        \endgroup
        
        \begin{proof}
            The complete proof is as follows.
            
            The marginal contribution is an implementation reflecting an agent's contribution and Markov Shapley value is defined as the weighted average of all marginal contributions. Therefore, this definition can still reflect an agent's contribution to the grand coalition by considering all permutations of agents to form the grand coalition and (iii) holds. We will next prove the (i), followed by (ii) and (iv). For any agent $\mathit{i} \in \mathcal{N}$ and any state $\mathbf{s} \in \mathcal{S}$, its Markov Shapley value denoted as $V_{i}^{\phi}(\mathbf{s})$.
            
            \textbf{Proof of (i):} Let us define $\Pi(\mathcal{N})$ as the set of all permutations of agents. Suppose that an arbitrary agent $i$ is a dummy agent for an arbitrary state $\mathbf{s} \in \mathcal{S}$. For any permutation $\mathit{m} \in \Pi(\mathcal{N})$ of agents to form the grand coalition, by Assumption \ref{assm:dummy_agent} we have $\max_{\pi_{\mathcal{C}_{i}^{m}}} V^{\pi_{\mathcal{C}_{i}^{m}}}(\mathbf{s})=\max_{\pi_{\mathcal{C}_{i}^{m}}} V^{\pi_{\mathcal{C}_{i}^{m} \cup \{i\}}}(\mathbf{s})$, thereby $\Phi_{i}(\mathbf{s}|\mathcal{C}_{i}^{m})=0$, where $\mathcal{C}_{i}^{m}$ denotes the intermediate coalition generated from permutation $m$ that agent $i$ would join. Also, the above analysis is valid for all permutations of agents to form the grand coalition. By Definition \ref{def:shapley_value}, it is not difficult to see that the dummy agent's Markov Shapley value will be 0 such that $V_{i}^{\phi}(\mathbf{s}) = 0$. The proof of (i) completes.
            
            \textbf{Proof of (ii):} The objective is proving that Markov Shapley value satisfies the following equation such that
            \begin{equation*}
                \max_{\pi} V^{\pi}(\mathbf{s}) = \sum_{i \in \mathcal{N}} \max_{\pi_{i}} V_{i}^{\phi}(\mathbf{s}), \quad \forall \mathbf{s} \in \mathcal{S}.
            \end{equation*}
            By the result from Proposition \ref{prop:marginal_contribution_equal_value_factorisation} and Assumption \ref{assm:assumption_for_joint_policy_factorisation}, for an arbitrary permutation $\mathit{m} \in \Pi(\mathcal{N})$ we can get the equation such that
            \begin{equation*}
                \max_{\pi} V^{\pi}(\mathbf{s}) = \sum_{i \in \mathcal{N}} \max_{\pi_{i}} \Phi_{i}(\mathbf{s}|\mathcal{C}_{i}^{m}), \quad \forall \mathbf{s} \in \mathcal{S},
            \end{equation*}
            where $\mathcal{C}_{i}^{m}$ denotes the intermediate coalition generated from permutation $m$ that agent $i$ would join and $\Phi_{i}(\mathbf{s}|\mathcal{C}_{i}^{m})$ is the corresponding marginal contribution. If we consider all possible permutations of agents to form the grand coalition and add all these inequalities, we can get the following equation such that
            \begin{equation*}
                \sum_{m \in \Pi(\mathcal{N})} \max_{\pi} V^{\pi}(\mathbf{s}) = \sum_{m \in \Pi(\mathcal{N})} \sum_{i \in \mathcal{N}} \max_{\pi_{i}} \Phi_{i}(\mathbf{s}|\mathcal{C}_{i}^{m}), \quad \forall \mathbf{s} \in \mathcal{S}.
            \end{equation*}
            
            By dividing $|\mathcal{N}|!$ on the both sides, we can get that
            \begin{equation}
                \frac{1}{|\mathcal{N}|!} \sum_{m \in \Pi(\mathcal{N})} \max_{\pi} V^{\pi}(\mathbf{s}) = \frac{1}{|\mathcal{N}|!} \sum_{i \in \mathcal{N}} \sum_{m \in \Pi(\mathcal{N})} \max_{\pi_{i}} \Phi_{i}(\mathbf{s}|\mathcal{C}_{i}^{m}), \quad \forall \mathbf{s} \in \mathcal{S}.
            \label{eq:shapley_property_3}
            \end{equation}
            
            Next, to ease life we start from the LHS of Eq.\ref{eq:shapley_property_3}. We directly get the following equation such that
            \begin{equation}
                \frac{1}{|\mathcal{N}|!} \sum_{m \in \Pi(\mathcal{N})} \max_{\pi} V^{\pi}(\mathbf{s}) = \frac{1}{|\mathcal{N}|!} \cdot |\mathcal{N}|! \cdot \max_{\pi} V^{\pi}(\mathbf{s}) = \max_{\pi} V^{\pi}(\mathbf{s}).
            \label{eq:shapley_property_3_left}
            \end{equation}
            
            Now, we start processing the RHS of Eq.\ref{eq:shapley_property_3}. By rearranging it, we can get the equations such that
            \begin{align}
                \frac{1}{|\mathcal{N}|!} \sum_{i \in \mathcal{N}} \sum_{m \in \Pi(\mathcal{N})} \max_{\pi_{i}} \Phi_{i}(\mathbf{s}|\mathcal{C}_{i}^{m}) &= \sum_{i \in \mathcal{N}} \frac{1}{|\mathcal{N}|!} \sum_{m \in \Pi(\mathcal{N})} \max_{\pi_{i}} \Phi_{i}(\mathbf{s}|\mathcal{C}_{i}^{m}) \nonumber \\
                &\quad (\text{The identical $\mathcal{C}_{i}^{m}$ in different permutations is written as $\mathcal{C}_{i}$} \nonumber \\
                &\quad \ \ \text{and we can rearrange the equation as follows.}) \nonumber \\
                &= \sum_{i \in \mathcal{C}} \frac{1}{|\mathcal{N}|!} \sum_{\mathcal{C}_{i} \subseteq \mathcal{N} \backslash \{i\}} |\mathcal{C}_{i}|!(|\mathcal{N}|-|\mathcal{C}_{i}|-1)! \cdot \max_{\pi_{i}} \Phi_{i}(\mathbf{s}|\mathcal{C}_{i}) \nonumber \\
                &= \sum_{i \in \mathcal{N}} \sum_{\mathcal{C}_{i} \subseteq \mathcal{N} \backslash \{i\}} \frac{|\mathcal{C}_{i}|!(|\mathcal{N}|-|\mathcal{C}_{i}|-1)!}{|\mathcal{N}|!} \cdot \max_{\pi_{i}} \Phi_{i}(\mathbf{s}|\mathcal{C}_{i}).
            \label{eq:shapley_property_3_right_1}
            \end{align}
            By Assumption \ref{assm:max_shapley_value}, we can get the following equations such that
            \begin{align}
                \sum_{i \in \mathcal{N}} \sum_{\mathcal{C}_{i} \subseteq \mathcal{N} \backslash \{i\}} \frac{|\mathcal{C}_{i}|!(|\mathcal{N}|-|\mathcal{C}_{i}|-1)!}{|\mathcal{N}|!} \cdot \max_{\pi_{i}} \Phi_{i}(\mathbf{s}|\mathcal{C}_{i})
                = \sum_{i \in \mathcal{N}} \max_{\pi_{i}} V_{i}^{\phi}(\mathbf{s})
            \label{eq:shapley_property_3_right_2}
            \end{align}
            Inserting the results from Eq.\ref{eq:shapley_property_3_left} and \ref{eq:shapley_property_3_right_2} to Eq.\ref{eq:shapley_property_3}, we can get the equation such that
            \begin{equation*}
                \max_{\pi} V^{\pi}(\mathbf{s}) = \sum_{i \in \mathcal{N}} \max_{\pi_{i}} V_{i}^{\phi}(\mathbf{s}), \quad \forall \mathbf{s} \in \mathcal{S}.
            \end{equation*}
            Therefore, the proof for (ii) completes.
            
            \textbf{Proof of (iv):} We would like to prove that if two agents are symmetric for an arbitrary state $\mathbf{s} \in \mathcal{S}$, then their optimal Markov Shapley values should be equal. As Assumption \ref{assm:symmetry} illustrates, suppose that agents $i$ and $j$ are symmetric for an arbitrary state $\mathbf{s} \in \mathcal{S}$, $V^{\pi_{\mathcal{C} \cup \{i\}}}(\mathbf{s}) = V^{\pi_{\mathcal{C} \cup \{j\}}}(\mathbf{s})$ for any coalitions $\mathcal{C} \ \mathlarger{\mathlarger{\mathlarger{\subseteq}}} \ \mathcal{N} \backslash \{i, j\}$. Given an arbitrary permutation $m \in \Pi(\mathcal{N})$, let $m'$ denote the permutation obtained by exchanging $i$ and $j$ such that $\mathcal{C}_{i}^{m} = \mathcal{C}_{j}^{m'}$, $\mathcal{C}_{i}^{m'} = \mathcal{C}_{j}^{m}$ and $\mathcal{C}_{l}^{m'} = \mathcal{C}_{l}^{m}, \forall l \neq i, j$. Next, we aim to prove that $\max_{\pi_{i}} \Phi_{i}(\mathbf{s} | \mathcal{C}_{i}^{m}) = \max_{\pi_{j}} \Phi_{j}(\mathbf{s} | \mathcal{C}_{j}^{m'})$, for the state $\mathbf{s}$.
            
            We first suppose that $i$ precedes $j$ in $m$. Then we have $\mathcal{C}_{i}^{m} = \mathcal{C}_{j}^{m'}$. Setting $\mathcal{C} = \mathcal{C}_{i}^{m} = \mathcal{C}_{j}^{m'}$, for the state $\mathbf{s}$ we can obtain that
            \begin{equation*}
                \begin{split}
                    \max_{\pi_{i}} \Phi_{i}(\mathbf{s} | \mathcal{C}_{i}^{m}) = \max_{\pi_{\mathcal{C} \cup \{i\}}} V^{\pi_{\mathcal{C} \cup \{i\}}}(\mathbf{s}) - \max_{\pi_{\mathcal{C}}} V^{\pi_{\mathcal{C}}}(\mathbf{s}), \\
                    \max_{\pi_{j}} \Phi_{j}(\mathbf{s} | \mathcal{C}_{j}^{m'}) = \max_{\pi_{\mathcal{C} \cup \{j\}}} V^{\pi_{\mathcal{C} \cup \{j\}}}(\mathbf{s}) - \max_{\pi_{\mathcal{C}}} V^{\pi_{\mathcal{C}}}(\mathbf{s}).
                \end{split}
            \end{equation*}
            By symmetry, we have $V^{\pi_{\mathcal{C} \cup \{i\}}}(\mathbf{s}) = V^{\pi_{\mathcal{C} \cup \{j\}}}(\mathbf{s})$, which directly implies that $\max_{\pi_{i}} \Phi_{i}(\mathbf{s} | \mathcal{C}_{i}^{m}) = \max_{\pi_{j}} \Phi_{j}(\mathbf{s} | \mathcal{C}_{j}^{m'})$.
            
            Second, we suppose that $j$ precedes $i$ in $m$. Setting $\mathcal{C} = \mathcal{C}_{i}^{m} \backslash \{j\}$, for the state $\mathbf{s}$ we have 
            \begin{equation*}
                \begin{split}
                    \max_{\pi_{i}} \Phi_{i}(\mathbf{s} | \mathcal{C}_{i}^{m}) = \max_{\pi_{\mathcal{C} \cup \{j\} \cup \{i\}}} V^{\pi_{\mathcal{C} \cup \{j\} \cup \{i\}}}(\mathbf{s}) - \max_{\pi_{\mathcal{C} \cup \{j\}}} V^{\pi_{\mathcal{C} \cup \{j\}}}(\mathbf{s}), \\
                    \max_{\pi_{j}} \Phi_{j}(\mathbf{s} | \mathcal{C}_{j}^{m'}) = \max_{\pi_{\mathcal{C} \cup \{j\} \cup \{i\}}} V^{\pi_{\mathcal{C} \cup \{j\} \cup \{i\}}}(\mathbf{s}) - \max_{\pi_{\mathcal{C} \cup \{i\}}} V^{\pi_{\mathcal{C} \cup \{i\}}}(\mathbf{s}).
                \end{split}
            \end{equation*}
            
            Since $\mathcal{C} \ \mathlarger{\mathlarger{\mathlarger{\subseteq}}} \ \mathcal{N} \backslash \{i, j\}$, by symmetry we have $V^{\pi_{\mathcal{C} \cup \{j\}}}(\mathbf{s}) = V^{\pi_{\mathcal{C} \cup \{i\}}}(\mathbf{s})$ and thus $\max_{\pi_{i}} \Phi_{i}(\mathbf{s} | \mathcal{C}_{i}^{m}) = \max_{\pi_{j}} \Phi_{j}(\mathbf{s} | \mathcal{C}_{j}^{m'})$. Therefore, we have proved that $\max_{\pi_{i}} \Phi_{i}(\mathbf{s} | \mathcal{C}_{i}^{m}) = \max_{\pi_{j}} \Phi_{j}(\mathbf{s} | \mathcal{C}_{j}^{m'})$ for any $m \in \Pi(\mathcal{N})$. It is not difficult to observe that $m \mapsto m'$ is a one-to-one mapping, so $\Pi(\mathcal{N}) = \left\{ m' \mid m \in \Pi(\mathcal{N}) \right\}$. 
            
            By Assumption \ref{assm:max_shapley_value}, for an arbitrary state $\mathbf{s} \in \mathcal{S}$ wherein agents are symmetric, we can directly have 
            \begin{equation*}
                \begin{split}
                    \max_{\pi_{i}} V^{\phi}_{i}(\mathbf{s}) &= \sum_{\mathcal{C}_{i} \ \mathlarger{\mathlarger{\subseteq}} \ \mathcal{N} \backslash \{i\} } \frac{|\mathcal{C}_{i}|!(|\mathcal{N}|-|\mathcal{C}_{i}|-1)!}{|\mathcal{N}|!} \cdot \max_{\pi_{i}} \Phi_{i}(\mathbf{s} | \mathcal{C}_{i}) \\
                    &= \frac{1}{|\mathcal{N}|!} \sum_{m \in \Pi(\mathcal{N})} \max_{\pi_{i}} \Phi_{i}(\mathbf{s} | \mathcal{C}_{i}^{m}) \\
                    &= \frac{1}{|\mathcal{N}|!} \sum_{m' \in \Pi(\mathcal{N})} \max_{\pi_{j}} \Phi_{j}(\mathbf{s} | \mathcal{C}_{j}^{m'}) \\
                    &= \sum_{\mathcal{C}_{j} \ \mathlarger{\mathlarger{\subseteq}} \ \mathcal{N} \backslash \{j\} } \frac{|\mathcal{C}_{j}|!(|\mathcal{N}|-|\mathcal{C}_{j}|-1)!}{|\mathcal{N}|!} \cdot \max_{\pi_{j}} \Phi_{j}(\mathbf{s} | \mathcal{C}_{j}) \\
                    &= \max_{\pi_{j}} V^{\phi}_{j}(\mathbf{s}).
                \end{split}
            \end{equation*}
            The proof of (iv) completes.
        \end{proof}
        
    \subsection{Mathematical Proofs and Derivations for Shapley Q-Learning}
    \label{subsec:math_proofs_and_derivations_shapley_q_learning}
        
        \subsubsection{Derivation of Shapley-Bellman optimality equation.}
        \label{subsubsec:deriviation_shapley-q_optimality_appendix}
        
            First, according to Bellman's principle of optimality \citep{bellman1952theory,sutton2018reinforcement} we can write out Bellman optimality equation for the optimal global Q-value such that
            \begin{equation}
                Q^{\pi^{*}}(\mathbf{s}, \mathbf{a}) = \sum_{\mathbf{s}'} Pr(\mathbf{s}' | \mathbf{s}, \mathbf{a}) \left[ R + \gamma \max_{\mathbf{a}} Q^{\pi^{*}}(\mathbf{s}', \mathbf{a}) \right].
            \label{eq:joint_bellman_optimality_equation}
            \end{equation}
            
            For convenience, we only consider the finite state space and action space here. By the efficiency property (i.e. (ii) in Proposition \ref{prop:shapley_value_properties}), we can get the approximation of the optimal global Q-value w.r.t. optimal actions such that
            \begin{equation}
                \max_{\mathbf{a}} Q^{\pi^{*}}(\mathbf{s}', \mathbf{a}) = \sum_{i \in \mathcal{N}} \max_{a_{i}} Q_{i}^{\phi^{*}}(\mathbf{s}', a_{i}).
                \label{eq:optimal_shapley_q_condition}
            \end{equation}
            
            Suppose that for all $\mathbf{s} \in \mathcal{S}$ and $a_{i} \in \mathcal{A}_{i}$, for each agent $i$ there exists bounded $\mathit{w}_{i}(\mathbf{s}, a_{i}) > 0$ and $b_{i}(\mathbf{s}) \geq 0$ that can project $Q^{\pi^{*}}(\mathbf{s}, \mathbf{a})$ onto the space of $Q^{\phi^{*}}_{i}(\mathbf{s}, a_{i})$ such that
            \begin{equation}
                Q^{\phi^{*}}_{i}(\mathbf{s}, a_{i}) = w_{i}(\mathbf{s}, a_{i}) \ Q^{\pi^{*}}(\mathbf{s}, \mathbf{a}) - b_{i}(\mathbf{s}).
            \label{eq:global_q_projection}
            \end{equation}
            
            If we denote $\mathbf{w}(\mathbf{s}, \mathbf{a}) = [w_{i}(\mathbf{s}, a_{i})]^{\top} \in \mathbb{R}^{\scriptscriptstyle|\mathcal{N}|}_{> 0}$, $\mathbf{b}(\mathbf{s}) = [b_{i}(\mathbf{s})]^{\top} \in \mathbb{R}^{\scriptscriptstyle|\mathcal{N}|}_{\geq 0}$ and $\mathbf{Q}^{\phi^{*}}(\mathbf{s}, \mathbf{a}) = [Q^{\phi^{*}}_{i}(\mathbf{s}, a_{i})]^{\top} \in \mathbb{R}^{\scriptscriptstyle|\mathcal{N}|}_{\geq 0}$, given Eq.\ref{eq:global_q_projection} we can write that
            \begin{equation}
                \mathbf{Q}^{\phi^{*}}(\mathbf{s}, \mathbf{a}) = \mathbf{w}(\mathbf{s}, \mathbf{a}) \ Q^{\pi^{*}}(\mathbf{s}, \mathbf{a}) - \mathbf{b}(\mathbf{s}).
            \label{eq:global_q_projection_joint}
            \end{equation}
            
            Besides, we suppose that $\sum_{i \in \mathcal{N}} w_{i}(\mathbf{s}, a_{i})^{-1} b_{i}(\mathbf{s}) = 0$. 
            
            Combined with Eq.\ref{eq:optimal_shapley_q_condition} and \ref{eq:global_q_projection_joint}, we can rewrite Eq.\ref{eq:joint_bellman_optimality_equation} to the equation as follows:
            \begin{equation}
                \begin{split}
                    \mathbf{Q}^{\phi^{*}}(\mathbf{s}, \mathbf{a}) = \mathbf{w}(\mathbf{s}, \mathbf{a}) \sum_{\mathbf{s}'} Pr(\mathbf{s}' | \mathbf{s}, \mathbf{a}) \left[
                    R \ + \ 
                    \gamma \sum_{i \in \mathcal{N}} \max_{a_{i}} Q_{i}^{\phi^{*}}(\mathbf{s}', a_{i}) \right] - \mathbf{b}(\mathbf{s}).
                \end{split}
            \end{equation}
            
            From Eq.\ref{eq:global_q_projection}, we know that $w_{i}(\mathbf{s}, a_{i}) > 0$. Therefore, we can rewrite Eq.\ref{eq:global_q_projection} to the following equation such that
            \begin{equation}
                w_{i}(\mathbf{s}, a_{i})^{-1} \ \left( Q^{\phi^{*}}_{i}(\mathbf{s}, a_{i}) + b_{i}(\mathbf{s}) \right) = Q^{\pi^{*}}(\mathbf{s}, \mathbf{a}).
            \label{eq:projection_global_q_optimality_trans}
            \end{equation}
            
            If we sum up Eq.\ref{eq:projection_global_q_optimality_trans} for all agents, we can obtain that
            \begin{equation}
                \sum_{i \in \mathcal{N}} w_{i}(\mathbf{s}, a_{i})^{-1} \ \left( Q^{\phi^{*}}_{i}(\mathbf{s}, a_{i}) + b_{i}(\mathbf{s}) \right) = |\mathcal{N}| \ Q^{\pi^{*}}(\mathbf{s}, \mathbf{a}).
            \end{equation}

            Since $\sum_{i \in \mathcal{N}} w_{i}(\mathbf{s}, a_{i})^{-1} b_{i}(\mathbf{s}) = 0$, we can get the following equation such that
            \begin{equation}
                \sum_{i \in \mathcal{N}} \frac{1}{|\mathcal{N}| \ w_{i}(\mathbf{s}, a_{i})} \ Q^{\phi^{*}}_{i}(\mathbf{s}, a_{i}) = \ Q^{\pi^{*}}(\mathbf{s}, \mathbf{a}).
            \label{eq:global_q_optimality_trans_1}
            \end{equation}
            Inserting Eq.\ref{eq:optimal_shapley_q_condition} into Eq.\ref{eq:global_q_optimality_trans_1}, we can get the following equation such that
            \begin{equation}
                \max_{\mathbf{a}} \sum_{i \in \mathcal{N}} \frac{1}{|\mathcal{N}| \ w_{i}(\mathbf{s}, a_{i})} \ Q^{\phi^{*}}_{i}(\mathbf{s}, a_{i}) = \sum_{i \in \mathcal{N}} \max_{a_{i}} Q_{i}^{\phi^{*}}(\mathbf{s}, a_{i}).
            \end{equation}
            Since $\mathbf{a} = \mathlarger{\mathlarger{\times}}_{\scriptscriptstyle i \in \mathcal{N}} a_{i}$, we can get that
            \begin{equation}
                \sum_{i \in \mathcal{N}} \max_{a_{i}} \frac{1}{|\mathcal{N}| \ w_{i}(\mathbf{s}, a_{i})} \ Q^{\phi^{*}}_{i}(\mathbf{s}, a_{i}) = \sum_{i \in \mathcal{N}} \max_{a_{i}} Q_{i}^{\phi^{*}}(\mathbf{s}, a_{i}).
            \end{equation}
            It is apparent that $\forall \mathbf{s} \in \mathcal{S}$ and $\mathit{a}_{i}^{*} = \arg\max_{a_{i}} Q^{\phi^{*}}_{i}(\mathbf{s}, a_{i})$, we have a solution $w_{i}(\mathbf{s}, a_{i}^{*}) = 1 / |\mathcal{N}|$. \footnote{Note that it exists other solutions rather than the one that we deduce between $\max_{a_{i}} \frac{1}{|\mathcal{N}| \ w_{i}(\mathbf{s}, a_{i})} \ Q^{\phi^{*}}_{i}(\mathbf{s}, a_{i})$ and $\max_{a_{i}} Q_{i}^{\phi^{*}}(\mathbf{s}, a_{i})$. Nevertheless, the result obtained in this paper is the one that exactly matches and explains the finding in the previous works \citep{wang2020towards}. As for the reason why the solution is the most likely to be achieved in empirical results is deserved to be studied in the future work.}
        
        \subsubsection{Proof of Theorem 1}
        \label{subsubsec:proof_of_theorem_1_appendix}
        \begin{lemma}[ \citet{dales2003introduction} ]
            \label{lem:banach_algebra}
            A set of real matrices $\mathcal{M}$ with a sub-multiplicative norm is a Banach Algebra and a non-empty complete metric space where the metric is induced by the sub-multiplicative norm. A sub-multiplicative norm $|| \cdot ||$ is a norm satisfying the following inequality such that
            \begin{equation*}
                \forall \mathbf{A}, \mathbf{B} \in \mathcal{M}: ||\mathbf{A} \mathbf{B}|| \leq ||\mathbf{A}|| \ ||\mathbf{B}||.
            \end{equation*}
        \end{lemma}
        
        \begin{lemma}
            \label{lem:matrices_norm_1_metric_space}
            For a set of real matrices $\mathcal{M}$, given an arbitrary matrix $\mathbf{A} = [a_{ij}] \in \mathbb{R}^{m \times n}$, $||\mathbf{A}||_{1} = \max_{1 \leq j \leq n} \sum_{1 \leq i \leq m} |a_{ij}|$ is a sub-multiplicative norm.
        \end{lemma}
        
        \begin{proof}
            The complete proof is as follows.
            
            First, we select two arbitrary matrices belonging to $\mathcal{M}$, i.e. $\mathbf{A}=[a_{ik}] \in \mathbb{R}^{m\times r}$ and $\mathbf{B}=[b_{kj}] \in \mathbb{R}^{r\times n}$. Then, we start proving that $||\cdot||_{1}$ is a sub-multiplicative norm as follows:
            \begin{align*}
                ||\mathbf{A} \mathbf{B}||_{1} &= \bigg|\bigg| \bigg[ \sum_{1 \leq k \leq r} a_{ik} b_{kj} \bigg] \bigg| \bigg|_{1} \\
                &= \max_{1 \leq j \leq n} \sum_{1 \leq i \leq m} \bigg| \sum_{1 \leq k \leq r} a_{ik} b_{kj} \bigg| \\
                &\quad (\text{By triangle inequality, we can obtain the following inequality.}) \\
                &\leq \max_{1 \leq j \leq n} \sum_{1 \leq i \leq m} \sum_{1 \leq k \leq r} \big| a_{ik} b_{kj} \big| \\
                &= \max_{1 \leq j \leq n} \sum_{1 \leq i \leq m} \sum_{1 \leq k \leq r} \big| a_{ik} \big| \ \big| b_{kj} \big| \\
                &= \max_{1 \leq j \leq n} \sum_{1 \leq k \leq r} \big| b_{kj} \big| \sum_{1 \leq i \leq m} \big| a_{ik} \big| \\
                &\leq \big|\big| \mathbf{B} \big| \big|_{1} \max_{1 \leq k \leq r} \sum_{1 \leq i \leq m} \big| a_{ik} \big| \\
                &= \big|\big| \mathbf{B} \big| \big|_{1} \big|\big| \mathbf{A} \big| \big|_{1} \\
                &= \big|\big| \mathbf{A} \big| \big|_{1} \big|\big| \mathbf{B} \big| \big|_{1}.
            \end{align*}
            Therefore, we have proven that given an arbitrary real matrix $\mathbf{A} = [a_{ij}] \in \mathbb{R}^{m \times n}$, $||\mathbf{A}||_{1} = \max_{1 \leq j \leq n} \sum_{1 \leq i \leq m} |a_{ij}|$ is a sub-multiplicative norm.
        \end{proof}
        
        \begin{lemma}
            \label{lemm:shapley_q_contraction_mapping}
            For all $\mathbf{s} \in \mathcal{S}$ and $\mathbf{a} \in \mathcal{A}$, Shapley-Bellman operator is a contraction mapping in a non-empty complete metric space when $\max_{\mathbf{s}} \big\{ \sum_{i \in \mathcal{N}} \max_{a_{i}} w_{i}(\mathbf{s}, a_{i}) \big\} < \frac{1}{\gamma}$.
        \end{lemma}
        
        \begin{proof}
            The complete proof is as follows.
            
            To ease life, we firstly define some variables that will be used for proof such that 
            \begin{align*}
                &\mathbf{Q}^{\phi} = \times_{i \in \mathcal{N}} Q_{i}^{\phi} \in \mathbb{R}^{|\mathcal{N}|\times|\mathcal{S}||\mathcal{A}|}, \\
                &\mathbf{w} \in \mathbb{R}^{|\mathcal{N}|\times|\mathcal{S}||\mathcal{A}|}, \\
                &Pr \in \mathbb{R}^{|\mathcal{S}||\mathcal{A}|\times|\mathcal{S}|}, \\
                &\mathbf{1} = [1,1,...,1]^{\top},
            \end{align*}
            where $\mathcal{A}=\mathlarger{\mathlarger{\times}}_{i \in \mathcal{N}} \mathcal{A}_{i}$. Then, for an arbitrary matrix $\mathbf{A} \in \mathbb{R}^{m \times n}$, we define the $||\cdot||_{1}$ for the induced matrix norm such that
            \begin{equation*}
                ||\mathbf{A}||_{1} = \max_{1 \leq j \leq n} \sum_{1 \leq i \leq m} |a_{ij}|,
            \end{equation*}
            where $\mathit{a}_{ij}$ is an arbitrary element in $\mathbf{A}$. By Lemma \ref{lem:matrices_norm_1_metric_space}, $||\cdot||_{1}$ defined here is a sub-multiplicative norm. By Lemma \ref{lem:banach_algebra}, the set of real matrices $\mathbb{R}^{|\mathcal{N}|\times |\mathcal{S}| |\mathcal{A}|}$ with the norm $||\cdot||_{1}$ is a Banach algebra and a non-empty complete metric space with the metric induced by $||\cdot||_{1}$.
            
            To show that the operator $\mathlarger{\Upsilon}$ is a contraction mapping in the supremum norm, we just need to show that for any $\mathbf{Q}^{\phi}_{1} = \times_{i \in \mathcal{N}} \big( Q_{i}^{\phi} \big)_{1} \in \mathbb{R}^{|\mathcal{N}|\times |\mathcal{S}| |\mathcal{A}|}$ and $\mathbf{Q}^{\phi}_{2} = \times_{i \in \mathcal{N}} \big( Q_{i}^{\phi} \big)_{2} \in \mathbb{R}^{|\mathcal{N}|\times |\mathcal{S}| |\mathcal{A}|}$, we have $|| \mathlarger{\Upsilon} \mathbf{Q}^{\phi}_{1}  - \mathlarger{\Upsilon} \mathbf{Q}^{\phi}_{2} ||_{1} \leq \delta ||\mathbf{Q}^{\phi}_{1} - \mathbf{Q}^{\phi}_{2}||_{1}$, where $\delta \in (0, 1)$.
            \begin{align*}
                &\quad || \mathlarger{\Upsilon} \mathbf{Q}^{\phi}_{1}  - \mathlarger{\Upsilon} \mathbf{Q}^{\phi}_{2} ||_{1} \\
                &= \max_{\mathbf{s}, \mathbf{a}} \mathbf{1}^{\top} \bigg| \mathbf{w}(\mathbf{s}, \mathbf{a}) \ \sum_{\mathbf{s}' \in \mathcal{S}} Pr(\mathbf{s}'|\mathbf{s}, \mathbf{a}) \Big[ R(\mathbf{s}, \mathbf{a}) + \gamma \sum_{i \in \mathcal{N}} \max_{a_{i}} \big( Q_{i}^{\phi} \big)_{1}(\mathbf{s}', a_{i}) \Big] - \mathbf{b}(\mathbf{s})\\
                &- \mathbf{w}(\mathbf{s}, \mathbf{a}) \ \sum_{\mathbf{s}' \in \mathcal{S}} Pr(\mathbf{s}'|\mathbf{s}, \mathbf{a}) \Big[ R(\mathbf{s}, \mathbf{a}) + \gamma \sum_{i \in \mathcal{N}} \max_{a_{i}} \big( Q_{i}^{\phi} \big)_{2}(\mathbf{s}', a_{i}) \Big] + \mathbf{b}(\mathbf{s}) \bigg|\\
                &= \gamma \max_{\mathbf{s}, \mathbf{a}} \mathbf{1}^{\top} \bigg| \mathbf{w}(\mathbf{s}, \mathbf{a}) \ \sum_{\mathbf{s}' \in \mathcal{S}} Pr(\mathbf{s}'|\mathbf{s}, \mathbf{a}) \Big[ \sum_{i \in \mathcal{N}} \max_{a_{i}} \big( Q_{i}^{\phi} \big)_{1}(\mathbf{s}', a_{i}) - \sum_{i \in \mathcal{N}} \max_{a_{i}} \big( Q_{i}^{\phi} \big)_{2}(\mathbf{s}', a_{i}) \Big] \bigg| \\
                &\leq \gamma \max_{\mathbf{s}, \mathbf{a}} \mathbf{1}^{\top} \bigg| \mathbf{w}(\mathbf{s}, \mathbf{a}) \bigg| \max_{\mathbf{s}, \mathbf{a}} \bigg| \sum_{\mathbf{s}' \in \mathcal{S}} Pr(\mathbf{s}'|\mathbf{s}, \mathbf{a}) \Big[ \sum_{i \in \mathcal{N}} \max_{a_{i}} \big( Q_{i}^{\phi} \big)_{1}(\mathbf{s}', a_{i}) - \sum_{i \in \mathcal{N}} \max_{a_{i}} \big( Q_{i}^{\phi} \big)_{2}(\mathbf{s}', a_{i}) \Big] \bigg| \\
                &\quad \left(\text{If we write $\delta = \gamma \max_{\mathbf{s}, \mathbf{a}} \mathbf{1}^{\top} \big| \mathbf{w}(\mathbf{s}, \mathbf{a}) \big|$, we can have the following equation.} \right) \\
                &= \delta \max_{\mathbf{s}, \mathbf{a}} \bigg| \sum_{\mathbf{s}' \in \mathcal{S}} Pr(\mathbf{s}'|\mathbf{s}, \mathbf{a}) \Big[ \sum_{i \in \mathcal{N}} \max_{a_{i}} \big( Q_{i}^{\phi} \big)_{1}(\mathbf{s}', a_{i}) - \sum_{i \in \mathcal{N}} \max_{a_{i}} \big( Q_{i}^{\phi} \big)_{2}(\mathbf{s}', a_{i}) \Big] \bigg| \\
                &\leq \delta \max_{\mathbf{s}, \mathbf{a}} \sum_{\mathbf{s}' \in \mathcal{S}} Pr(\mathbf{s}'|\mathbf{s}, \mathbf{a}) \bigg| \sum_{i \in \mathcal{N}} \max_{a_{i}} \big( Q_{i}^{\phi} \big)_{1}(\mathbf{s}', a_{i}) - \sum_{i \in \mathcal{N}} \max_{a_{i}} \big( Q_{i}^{\phi} \big)_{2}(\mathbf{s}', a_{i}) \bigg| \\
                &= \delta \bigg| \sum_{i \in \mathcal{N}} \Big[ \max_{a_{i}} \big( Q_{i}^{\phi} \big)_{1}(\mathbf{s}', a_{i}) - \max_{a_{i}} \big( Q_{i}^{\phi} \big)_{2}(\mathbf{s}', a_{i}) \Big] \bigg| \\
                &\quad \left( \text{By triangle inequality, we can obtain the following inequality.} \right) \\
                &\leq \delta \sum_{i \in \mathcal{N}} \bigg| \max_{a_{i}} \big( Q_{i}^{\phi} \big)_{1}(\mathbf{s}', a_{i}) - \max_{a_{i}} \big( Q_{i}^{\phi} \big)_{2}(\mathbf{s}', a_{i}) \bigg| \\
                &\leq \delta \sum_{i \in \mathcal{N}} \max_{a_{i}} \bigg| \big( Q_{i}^{\phi} \big)_{1}(\mathbf{s}', a_{i}) - \big( Q_{i}^{\phi} \big)_{2}(\mathbf{s}', a_{i}) \bigg| \\
                &\quad \left( \text{Since $\mathbf{a} = \mathlarger{\mathlarger{\times}}_{\scriptscriptstyle{i \in \mathcal{N}}} a_{i}$, we have the following equation.} \right) \\
                &= \delta \max_{\mathbf{a}} \sum_{i \in \mathcal{N}} \bigg| \big( Q_{i}^{\phi} \big)_{1}(\mathbf{s}', a_{i}) - \big( Q_{i}^{\phi} \big)_{2}(\mathbf{s}', a_{i}) \bigg| \\
                &\leq \delta \max_{\mathbf{z}, \mathbf{a}} \sum_{i \in \mathcal{N}} \bigg| \big( Q_{i}^{\phi} \big)_{1}(\mathbf{z}, a_{i}) - \big( Q_{i}^{\phi} \big)_{2}(\mathbf{z}, a_{i}) \bigg| 
                = \delta || \mathbf{Q}^{\phi}_{1}  - \mathbf{Q}^{\phi}_{2} ||_{1}.
            \end{align*}
            
            Now, we need to discuss the condition to $\delta \in (0, 1)$. Apparently, $\delta > 0$, so we just need to discuss the condition to guarantee that $\delta < 1$. We now have the following discussions such that
            \begin{align*}
                &\quad \ \ \delta = \gamma \max_{\mathbf{s}, \mathbf{a}} \mathbf{1}^{\top} \big| \mathbf{w}(\mathbf{s}, \mathbf{a}) \big| < 1 \ (\text{Since $w_{i}(\mathbf{s}, a_{i}) > 0$.})\\
                &\Rightarrow \gamma \max_{\mathbf{s}, \mathbf{a}} \sum_{i \in \mathcal{N}} w_{i}(\mathbf{s}, a_{i}) < 1 \\
                &(\text{When $\gamma \neq 0$, we can have the following inequality.}) \\
                &\Rightarrow \max_{\mathbf{s}, \mathbf{a}} \sum_{i \in \mathcal{N}} w_{i}(\mathbf{s}, a_{i}) < \frac{1}{\gamma} \\
                &(\text{Since $\mathbf{a} = \mathlarger{\mathlarger{\times}}_{\scriptscriptstyle i \in \mathcal{N}} a_{i}$, we have the following equation.}) \\
                &\Rightarrow \max_{\mathbf{s}} \Big\{ \sum_{i \in \mathcal{N}} \max_{a_{i}} w_{i}(\mathbf{s}, a_{i}) \Big\} < \frac{1}{\gamma}.
            \end{align*}
            Therefore, we show that Shapley-Bellman operator $\mathlarger{\Upsilon}$ is a contraction mapping in the non-empty complete metric space generated by $\mathbb{R}^{|\mathcal{N}| \times |\mathcal{S}| |\mathcal{A}|}$ with the metric induced by $||\cdot||_{1}$, when $\max_{\mathbf{s}} \Big\{ \sum_{i \in \mathcal{N}} \max_{a_{i}} w_{i}(\mathbf{s}, a_{i}) \Big\} < \frac{1}{\gamma}$. Finally, it is apparent that $w_{i}(\mathbf{s}, a_{i}) = 1 / |\mathcal{N}|$ when $\mathit{a}_{i} = \arg\max_{a_{i}} Q^{\phi}_{i}(\mathbf{s}, a_{i})$ satisfies the above condition.
        \end{proof}
    
        \begin{corollary}
        \label{coro:shapley_q_fixed_point}
            According to Banach fixed-point theorem \citep{banach1922operations}, Shapley-Bellman operator admits a unique fixed point. Moreover, starting by an arbitrary start point, the sequence recursively generated by Shapley-Bellman operator can finally converge to that fixed point.
        \end{corollary}
        
        \begin{proof}
            Since $\langle \mathbb{R}^{|\mathcal{N}|\times |\mathcal{S}| |\mathcal{A}|}, ||\cdot||_{1} \rangle$ is a non-empty complete metric space and Shapley-Bellman operator $\mathlarger{\Upsilon}$ is shown as a contraction mapping in Lemma \ref{lemm:shapley_q_contraction_mapping}, by Banach fixed-point theorem \citep{banach1922operations} we can directly conclude that Shapley-Bellman operator $\mathlarger{\Upsilon}$ admits a unique fixed point. Furthermore, starting by an arbitrary start point, the sequence recursively generated by Shapley-Bellman operator $\mathlarger{\Upsilon}$ can finally converge to that fixed point.
        \end{proof}
        
        \begingroup
            \def\thetheorem{\ref{thm:shapley_q_optimal}}
            \begin{theorem}
                Shapley-Bellman operator can converge to the optimal Markov Shapley Q-value and the corresponding optimal joint deterministic policy when $\max_{\mathbf{s}} \big\{ \sum_{i \in \mathcal{N}} \max_{a_{i}} w_{i}(\mathbf{s}, a_{i}) \big\} < \frac{1}{\gamma}$.
            \end{theorem}
        \endgroup
        
        \begin{proof}
            By Corollary \ref{coro:shapley_q_fixed_point}, we get that Shapley-Bellman operator admits a unique fixed point. Since Shapley-Bellman optimality equation (i.e., Eq.\ref{eq:shapley_q_optimality_equation}) is obviously a fixed point for Shapley-Bellman operator, it is not difficult to get the conclusion that the optimal Markov Shapley Q-value is achieved. Since the sum of optimal Markov Shapley Q-values is equal to the optimal global Q-value and the optimal global Q-value corresponds to the optimal joint deterministic policy, we show that the optimal joint deterministic policy is achieved. Besides, it is obvious that Shapley-Bellman optimality equation can be transformed back to the Bellman optimality equation w.r.t. the optimal global Q-value, given the efficiency property of Markov Shapley value.
        \end{proof}
        
        \subsubsection{Stochastic Approximation of Shapley-Bellman operator}
        \label{subsubsec:stochastic_approximation_of_shapley-q_operator_appendix}
        
            We now derive the stochastic approximation of Shapley-Bellman operator over the value space, i.e. a form of Q-learning derived from Shapley-Bellman operator. By sampling from $Pr(\mathbf{s}'|\mathbf{s}, \mathbf{a})$ via Monte Carlo method, the Q-learning algorithm can be expressed as follows:
            \begin{equation}
                \mathbf{Q}^{\phi}_{t+1}(\mathbf{s}, \mathbf{a}) \leftarrow \mathbf{Q}^{\phi}_{t}(\mathbf{s}, \mathbf{a}) + \alpha_{t}(\mathbf{s}, \mathbf{a}) \big[ \mathbf{w}(\mathbf{s}, \mathbf{a}) \big( R_{t} + \gamma \sum_{i \in \mathcal{N}} \max_{a_{i}} (Q_{i}^{\phi})_{t}(\mathbf{s}', a_{i}) \big) - \mathbf{b}(\mathbf{s}) - \mathbf{Q}^{\phi}_{t}(\mathbf{s}, \mathbf{a}) \big].
            \label{eq:shapley_q_learning_primal}
            \end{equation}
            
            \begin{lemma}[\citet{jaakkola1994convergence}]
                The random process $\{\Delta_{t}\}$ taking values $\mathbb{R}^{n}$ defined as $$\Delta_{t+1}(x) = (1 - \alpha_{t}(x)) \Delta_{t}(x) + \alpha_{t}(x) F_{t}(x)$$ converges to 0 w.p.1 under the following assumptions:
                \begin{itemize}
                    \item $0 \leq \alpha_{t} \leq 1$, $\sum_{t} \alpha_{t}(x) = \infty$ and $\sum_{t} \alpha_{t}^{2} \leq \infty$;
                    \item $|| \mathbb{E}[F_{t}(x) | \mathcal{F}_{t}] ||_{W} \leq \delta || \Delta_{t} ||_{W}$, with $0 \leq \delta < 1$;
                    \item $\textbf{var} [F_{t}(x) | \mathcal{F}_{t}] \leq C ( 1 + ||\Delta_{t}||_{W}^{2} )$, for $C > 0$.
                \end{itemize}
            \label{lemm:stochastic_process}
            \end{lemma}
            
            \setcounter{theorem}{3}
            \begin{theorem}
                For a finite Markov convex game, the Q-learning algorithm derived by Shapley-Bellman operator given by the update rule such that
                \begin{equation*}
                    \mathbf{Q}^{\phi}_{t+1}(\mathbf{s}, \mathbf{a}) \leftarrow \mathbf{Q}^{\phi}_{t}(\mathbf{s}, \mathbf{a}) + \alpha_{t}(\mathbf{s}, \mathbf{a}) \left[ \mathbf{w}(\mathbf{s}, \mathbf{a}) \left( R_{t} + \gamma \sum_{i \in \mathcal{N}} \max_{a_{i}} (Q_{i}^{\phi})_{t}(\mathbf{s}', a_{i}) \right) - \mathbf{b}(\mathbf{s}) - \mathbf{Q}^{\phi}_{t}(\mathbf{s}, \mathbf{a}) \right],
                \end{equation*}
                converges w.p.1 to the optimal Markov Shapley Q-value if 
                \begin{equation}
                    \sum_{t} \alpha_{t}(\mathbf{s}, \mathbf{a}) = \infty \ \ \ \ \ \ \ \ \sum_{t} \alpha^{2}_{t}(\mathbf{s}, \mathbf{a}) \leq \infty
                \label{eq:alpha_condition}
                \end{equation}
                for all $\mathbf{s} \in \mathcal{S}$ and $\mathbf{a} \in \mathcal{A}$ as well as $\max_{\mathbf{s}} \left\{ \sum_{i \in \mathcal{N}} \max_{a_{i}} w_{i}(\mathbf{s}, a_{i}) \right\} < \frac{1}{\gamma}$.
                \begin{proof}
                    The proof follows the sketch of proving the convergence of Q-learning given by \citet{melo2001convergence}. First, we rewrite Eq.\ref{eq:shapley_q_learning_primal} to 
                    $$\mathbf{Q}^{\phi}_{t}(\mathbf{s}, \mathbf{a}) = \left( 1 - \alpha_{t}(\mathbf{s}, \mathbf{a}) \right) \mathbf{Q}^{\phi}_{t}(\mathbf{s}, \mathbf{a}) + \alpha_{t}(\mathbf{s}, \mathbf{a}) \left[ \mathbf{w}(\mathbf{s}, \mathbf{a}) \left( R_{t} + \gamma \sum_{i \in \mathcal{N}} \max_{a_{i}} (Q_{i}^{\phi})_{t}(\mathbf{s}', a_{i}) \right) - \mathbf{b}(\mathbf{s}) \right].$$
                    By subtracting $\mathbf{Q}^{\phi^{*}}(\mathbf{s}, \mathbf{a})$ and letting 
                    $$\Delta_{t}(\mathbf{s}, \mathbf{a}) = \mathbf{Q}^{\phi}_{t}(\mathbf{s}, \mathbf{a}) - \mathbf{Q}^{\phi^{*}}(\mathbf{s}, \mathbf{a}),$$ we can transform Eq.\ref{eq:shapley_q_learning_primal} to
                    
                    $$\Delta_{t+1}(\mathbf{s}, \mathbf{a}) = (1 - \alpha_{t}(\mathbf{s}, \mathbf{a})) \Delta_{t}(\mathbf{s}, \mathbf{a}) + \alpha_{t}(\mathbf{s}, \mathbf{a}) F_{t}(\mathbf{s}, \mathbf{a}),$$
                    
                    where $$F_{t}(\mathbf{s}, \mathbf{a}) = \mathbf{w}(\mathbf{s}, \mathbf{a}) \left( R_{t} + \gamma \sum_{i \in \mathcal{N}} \max_{a_{i}} (Q_{i}^{\phi})_{t}(\mathbf{s}', a_{i}) \right) - \mathbf{b}(\mathbf{s}) - \mathbf{Q}^{\phi^{*}}(\mathbf{s}, \mathbf{a}).$$
                    Since $\mathbf{s}' \in \mathcal{S}$ is a random sample from Markov Chain, so we can get that
                    \begin{align*}
                        \mathbb{E}[ F_{t}(\mathbf{s}, \mathbf{a}) | \mathcal{F}_{t} ] &= \sum_{\mathbf{s}' \in \mathcal{S}} Pr(\mathbf{s}'|\mathbf{s}, \mathbf{a}) \left[ \mathbf{w}(\mathbf{s}, \mathbf{a}) \left( R_{t} + \gamma \sum_{i \in \mathcal{N}} \max_{a_{i}} (Q_{i}^{\phi})_{t}(\mathbf{s}', a_{i}) \right) - \mathbf{b}(\mathbf{s}) - \mathbf{Q}^{\phi^{*}}(\mathbf{s}, \mathbf{a}) \right] \\
                        &= \mathbf{w}(\mathbf{s}, \mathbf{a}) \sum_{\mathbf{s}' \in \mathcal{S}} Pr(\mathbf{s}'|\mathbf{s}, \mathbf{a}) \left( R_{t} + \gamma \sum_{i \in \mathcal{N}} \max_{a_{i}} (Q_{i}^{\phi})_{t}(\mathbf{s}', a_{i}) \right) - \mathbf{b}(\mathbf{s}) - \mathbf{Q}^{\phi^{*}}(\mathbf{s}, \mathbf{a}) \\
                        & \quad \left(\text{Since $\max_{\mathbf{s}} \big\{ \sum_{i \in \mathcal{N}} \max_{a_{i}} w_{i}(\mathbf{s}, a_{i}) \big\} < \frac{1}{\gamma}$.}\right) \\
                        &= \mathlarger{\Upsilon} \mathbf{Q}^{\phi}_{t}(\mathbf{s}, \mathbf{a}) - \mathlarger{\Upsilon} \mathbf{Q}^{\phi^{*}}(\mathbf{s}, \mathbf{a}).
                    \end{align*}
                    
                    By the results from Theorem \ref{lemm:shapley_q_contraction_mapping}, we can get that
                    \begin{equation*}
                        ||\mathbb{E}[ F_{t}(\mathbf{s}, \mathbf{a}) | \mathcal{F}_{t} ]||_{1} \leq \delta ||\mathbf{Q}^{\phi}_{t}(\mathbf{s}, \mathbf{a}) - \mathbf{Q}^{\phi^{*}}(\mathbf{s}, \mathbf{a})||_{1} = \delta ||\Delta_{t}(\mathbf{s}, \mathbf{a})||_{1},
                    \end{equation*}
                    where $\delta \in (0, 1)$.
                    
                    Next, we get that
                    \begin{align*}
                        \textbf{var}[F_{t}(\mathbf{s}, \mathbf{a})| \mathcal{F}_{t}] &= \mathlarger{\mathbb{E}} \big[ \big( \mathbf{w}(\mathbf{s}, \mathbf{a}) \big( R_{t} + \gamma \sum_{i \in \mathcal{N}} \max_{a_{i}} (Q_{i}^{\phi})_{t}(\mathbf{s}', a_{i}) \big) - \mathbf{b}(\mathbf{s}) - \mathbf{Q}^{\phi^{*}}(\mathbf{s}, \mathbf{a}) \\
                        &- \mathlarger{\Upsilon} \mathbf{Q}^{\phi}_{t}(\mathbf{s}, \mathbf{a}) + \mathbf{Q}^{\phi^{*}}(\mathbf{s}, \mathbf{a}) \big)^{2} \big] \\
                        &= \mathlarger{\mathbb{E}} \big[ \big( \mathbf{w}(\mathbf{s}, \mathbf{a}) \big( R_{t} + \gamma \sum_{i \in \mathcal{N}} \max_{a_{i}} (Q_{i}^{\phi})_{t}(\mathbf{s}', a_{i}) \big) - \mathbf{b}(\mathbf{s}) - \mathlarger{\Upsilon} \mathbf{Q}^{\phi}_{t}(\mathbf{s}, \mathbf{a}) \big)^{2} \big] \\
                        &= \textbf{var} \big[ \mathbf{w}(\mathbf{s}, \mathbf{a}) \big( R_{t} + \gamma \sum_{i \in \mathcal{N}} \max_{a_{i}} (Q_{i}^{\phi})_{t}(\mathbf{s}', a_{i}) \big) - \mathbf{b}(\mathbf{s}) \mathlarger{|} \mathcal{F}_{t} \big].
                    \end{align*}
                    
                    Since $R_{t}$, $\mathbf{w}(\mathbf{s}, \mathbf{a})$ and $\mathbf{b}(\mathbf{s})$ are bounded, it clearly verifies that 
                    \begin{equation*}
                        \textbf{var}[F_{t}(\mathbf{s}, \mathbf{a})| \mathcal{F}_{t}] \leq C (1 + ||\Delta_{t}(\mathbf{s}, \mathbf{a})||_{1}^{2})
                    \end{equation*}
                    for some constant $C$.
                    
                    Finally, by Lemma \ref{lemm:stochastic_process} it is easy to see that $\Delta_{t}$ converges to 0 w.p.1, i.e., $ \mathbf{Q}^{\phi}_{t}(\mathbf{s}, \mathbf{a})$ converges to $ \mathbf{Q}^{\phi^{*}}(\mathbf{s}, \mathbf{a})$ w.p.1, given the condition in Eq.\ref{eq:alpha_condition}.
                \end{proof}
            \label{thm:proof_of_shapley_q_learning}
            \end{theorem}
        
        \subsubsection{Derivation of Shapley Q-Learning}
        \label{subsubsec:derivation_of_shapley_q-learning}
        
            Similar to the operations in Section \ref{subsubsec:stochastic_approximation_of_shapley-q_operator_appendix}, by stochastic approximation in value space, i.e. sampling $\mathbf{s}'$ from $Pr(\mathbf{s}'|\mathbf{s}, \mathbf{a})$ via Monte Carlo method, Shapley-Bellman operator can be expressed as follows:
            \begin{equation}
                \mathbf{Q}^{\phi}(\mathbf{s}, \mathbf{a}) = \mathbf{w}(\mathbf{s}, \mathbf{a}) \left( R + \gamma \sum_{i \in \mathcal{N}} \max_{a_{i}} Q_{i}^{\phi}(\mathbf{s}', a_{i}) \right)  - \mathbf{b}(\mathbf{s}),
            \label{eq:stochstic_approximation_shapley_operator}
            \end{equation}
            where $\mathbf{w}(\mathbf{s}, \mathbf{a}) = [w_{i}(\mathbf{s}, a_{i})]^{\top} \in \mathbb{R}^{\scriptscriptstyle|\mathcal{N}|}_{+}$; $\mathbf{b}(\mathbf{s}) = [b_{i}(\mathbf{s})]^{\top} \in \mathbb{R}^{\scriptscriptstyle|\mathcal{N}|}_{+}$; and $\mathbf{Q}^{\phi}(\mathbf{s}, \mathbf{a}) = [Q_{i}^{\phi}(\mathbf{s}, a_{i})]^{\top} \in \mathbb{R}^{\scriptscriptstyle|\mathcal{N}|}_{+}$. Since $\mathbf{w}(\mathbf{s}, \mathbf{a}) = diag \big( \mathbf{w}(\mathbf{s}, \mathbf{a}) \big) \ \mathbf{1}$ where $diag(\cdot)$ denotes the diagonalization of a vector\footnote{It is a square diagonal matrix with the elements of vector v on the main diagonal, and the other entries of the matrix are zeros.} and $\mathbf{1}$ denotes the vector of ones, Eq.\ref{eq:stochstic_approximation_shapley_operator} can be equivalently represented as
            \begin{equation}
                 \mathbf{Q}^{\phi}(\mathbf{s}, \mathbf{a}) = diag \big( \mathbf{w}(\mathbf{s}, \mathbf{a}) \big) \ \mathbf{1} \ \left( R + \gamma \sum_{i \in \mathcal{N}} \max_{a_{i}} Q_{i}^{\phi}(\mathbf{s}', a_{i}) \right) - \mathbf{b}(\mathbf{s}).
            \label{eq:stochstic_approximation_shapley_operator_0}
            \end{equation}
            Since $w_{i}(\mathbf{s}, a_{i}) > 0, \forall i \in \mathcal{N}$, we can write the following equivalent form to Eq.\ref{eq:stochstic_approximation_shapley_operator_0} such that
            \begin{equation} 
                diag \big( \mathbf{w}(\mathbf{s}, \mathbf{a}) \big)^{-1} \mathbf{Q}^{\phi}(\mathbf{s}, \mathbf{a}) = \mathbf{1} \ \left( R + \gamma \sum_{i \in \mathcal{N}} \max_{a_{i}} Q_{i}^{\phi}(\mathbf{s}', a_{i}) \right) - diag \big( \mathbf{w}(\mathbf{s}, \mathbf{a}) \big)^{-1} \mathbf{b}(\mathbf{s}).
            \label{eq:stochstic_approximation_shapley_operator_1}
            \end{equation}
            Next, we multiply $\mathbf{1}^{\top}$ on both sides and obtain the following equation such that
            \begin{equation}
                \sum_{i \in \mathcal{N}} \frac{1}{w_{i}(\mathbf{s}, a_{i})} \ Q_{i}^{\phi}(\textbf{s}, a_{i}) = |\mathcal{N}| \ \left( R + \gamma \sum_{i \in \mathcal{N}} \max_{a_{i}} Q_{i}^{\phi}(\mathbf{s}', a_{i}) \right) - \sum_{i \in \mathcal{N}} w_{i}(\mathbf{s}, a_{i})^{-1} b_{i}(\mathbf{s}).
            \label{eq:stochstic_approximation_shapley_operator_2}
            \end{equation}
            Since the condition such that $\sum_{i \in \mathcal{N}} w_{i}(\mathbf{s}, a_{i})^{-1} b_{i}(\mathbf{s}) = 0$, by dividing $|\mathcal{N}|$ on both sides we get that
            \begin{equation}
                \sum_{i \in \mathcal{N}} \frac{1}{|\mathcal{N}|w_{i}(\mathbf{s}, a_{i})} \ Q_{i}^{\phi}(s, a_{i}) = R + \gamma \sum_{i \in \mathcal{N}} \max_{a_{i}} Q_{i}^{\phi}(s, a_{i}).
            \label{eq:stochstic_approximation_shapley_operator_3}
            \end{equation}
            Since $w_{i}(\mathbf{s}, a_{i}) = 1 / |\mathcal{N}|$ when $\mathit{a}_{i} = \arg\max_{a_{i}} Q^{\phi}_{i}(\mathbf{s}, a_{i})$, by defining $\delta_{i}(\mathbf{s}, a_{i}) = \frac{1}{|\mathcal{N}| \ w_{i}(\mathbf{s}, a_{i})}$ we can get that 
            \begin{equation}
                \delta_{i}(\mathbf{s}, a_{i}) = \begin{cases} 
                                                      1 & a_{i} = \arg\max_{a_{i}} Q^{\phi}_{i}(\mathbf{s}, a_{i}), \\
                                                      \alpha_{i}(\mathbf{s}, a_{i}) & a_{i} \neq \arg\max_{a_{i}} Q^{\phi}_{i}(\mathbf{s}, a_{i}),
                                                 \end{cases}
            \label{eq:delta_1}
            \end{equation}
            where $\alpha_{i}(\mathbf{s}, a_{i})$ is a variable that expresses $\frac{1}{|\mathcal{N}| \ w_{i}(\mathbf{s}, a_{i})}$ when $a_{i} \neq \arg\max_{a_{i}} Q^{\phi}_{i}(\mathbf{s}, a_{i})$ for the ease of implementation.
            
            Substituting Eq.\ref{eq:delta_1} into Eq.\ref{eq:stochstic_approximation_shapley_operator_3}, we can get the following equation such that
            \begin{equation}
                \sum_{i \in \mathcal{N}} \delta_{i}(\mathbf{s}, a_{i}) \ Q^{\phi}_{i}(\mathbf{s}, a_{i}) = R + \gamma \sum_{i \in \mathcal{N}} \max_{a_{i}} Q_{i}^{\phi}(\mathbf{s}', a_{i}).
            \label{eq:dual_shapley_q_operator}
            \end{equation}
            
            By rearranging Eq.\ref{eq:dual_shapley_q_operator}, we obtain the TD error of Shapley Q-learning (SHAQ) such that
            \begin{equation}
                \Delta(\mathbf{s}, \mathbf{a}, \mathbf{s}') = R + \gamma \sum_{i \in \mathcal{N}} \max_{a_{i}} Q_{i}^{\phi}(\mathbf{s}', a_{i}) - \sum_{i \in \mathcal{N}} \delta_{i}(\mathbf{s}, a_{i}) \ Q^{\phi}_{i}(\mathbf{s}, a_{i}).
            \label{eq:td_error_shapley_q_learning_appendix}
            \end{equation}
            
            Note that the TD error of SHAQ is necessary for the TD error of Eq.\ref{eq:shapley_q_learning_primal} (i.e. the stochastic learning process that we proved to converge to the optimal Markov Shapley Q-value in Theorem \ref{thm:proof_of_shapley_q_learning}). For this reason, the condition $\max_{\mathbf{s}} \big\{ \sum_{i \in \mathcal{N}} \max_{a_{i}} w_{i}(\mathbf{s}, a_{i}) \big\} < \frac{1}{\gamma}$ is necessary to be satisfied so that the convergence to the optimality is possible to hold.
    
    \subsection{Mathematical Proofs of Validity and Interpretability}
    \label{subsec:mathematical_proofs_for_regularity_and_interpretability}
    
        \begin{lemma}
        \label{lemm:markov_core_convex_set}
            Markov core is a convex set.
        \end{lemma}
        
        \begin{proof}
            Let $\big( \max_{\pi_{i}} x_{i}(\mathbf{s}) \big)_{i \in \mathcal{N}}$ and $\big( \max_{\pi_{i}} y_{i}(\mathbf{s}) \big)_{i \in \mathcal{N}}$ be two vectors in the Markov core and $\alpha \in [0, 1)$ be an arbitrary scalar. To ease life, for any $i \in \mathcal{N}$ we let $\max_{\pi_{i}} z_{i}(\mathbf{s}) = \alpha \max_{\pi_{i}} x_{i}(\mathbf{s}) + (1 - \alpha) \max_{\pi_{i}} y_{i}(\mathbf{s})$. By definition, for any coalition $\mathcal{C} \ \mathlarger{\mathlarger{\mathlarger{\subseteq}}} \ \mathcal{N}$ we have
            \begin{align*}
                \max_{\pi_{\mathcal{C}}} z(\mathbf{s} | \mathcal{C}) &= \sum_{i \in \mathcal{C}} \max_{\pi_{i}} z_{i}(\mathbf{s}) \\
                &= \sum_{i \in \mathcal{C}} \alpha \max_{\pi_{i}} x_{i}(\mathbf{s}) + (1 - \alpha) \max_{\pi_{i}} y_{i}(\mathbf{s})\\
                &= \alpha \sum_{i \in \mathcal{C}} \max_{\pi_{i}} x_{i}(\mathbf{s}) + (1 - \alpha) \sum_{i \in \mathcal{C}} \max_{\pi_{i}} y_{i}(\mathbf{s})\\
                &\geq \alpha \max_{\pi_{\mathcal{C}}} V^{\pi_{\mathcal{C}}}(\mathbf{s}) + (1 - \alpha) \max_{\pi_{\mathcal{C}}} V^{\pi_{\mathcal{C}}}(\mathbf{s})\\
                &= \max_{\pi_{\mathcal{C}}} V^{\pi_{\mathcal{C}}}(\mathbf{s}).
            \end{align*}
            Therefore, we proved that Markov core is a convex set.
        \end{proof}
        
        \begingroup
            \def\thetheorem{\ref{thm:shapley_value_core}}
            \begin{theorem}
                The optimal Markov Shapley value is a solution in the Markov core under Markov convex game (MCG) with the grand coalition.
            \end{theorem}
        \endgroup
        
        \begin{proof}
            The optimal Markov Shapley value is the affine combination of the optimal marginal contributions. We know that Markov core is a convex set by Lemma \ref{lemm:markov_core_convex_set} and the optimal marginal contribution is in the Markov core by Lemma \ref{lemm:condition_coalition_marginal_contribution}. Since the affine combination of the points in a convex set is still in this convex set, we get that the optimal Markov Shapley value is in the Markov core.
        \end{proof}
        
    \subsection{Mathematical Derivation for Implementation of Shapley Q-Learning}
    \label{subsec:mathematical_derivation_for_implementation_of_shapley_q_learning}
        
        \def\theproposition{\ref{prop:shapley_value_approximate}}
        \begin{proposition}
            Suppose any action marginal contribution can be factorised to the form such that $\Upphi_{i}(\mathbf{s}, a_{i} | \mathcal{C}_{i}) = \sigma(\mathbf{s}, \mathbf{a}_{ \scriptscriptstyle\mathcal{C}_{i} \cup \{i\} }) \ \hat{Q}_{i}(\mathbf{s}, a_{i})$. With the condition such that
            \begin{equation*}
                \mathbb{E}_{\mathcal{C}_{i} \sim Pr(\mathcal{C}_{i} | \mathcal{N} \backslash \{i\})}[\sigma(\mathbf{s}, \mathbf{a}_{ \scriptscriptstyle\mathcal{C}_{i} \cup \{i\} })] =
                \begin{cases} 
                     1 & \ \ a_{i} = \arg\max_{a_{i}} Q^{\phi}_{i}(\mathbf{s}, a_{i}), \\
                     K \in (0, 1) & \ \ a_{i} \neq \arg\max_{a_{i}} Q^{\phi}_{i}(\mathbf{s}, a_{i}),
                \end{cases}
            \end{equation*}
            we have
            \begin{equation*}
                \begin{cases} 
                     Q_{i}^{\phi}(\mathbf{s}, a_{i}) = \hat{Q}_{i}(\mathbf{s}, a_{i}) & \ \ a_{i} = \arg\max_{a_{i}} \hat{Q}_{i}(\mathbf{s}, a_{i}), \\
                     \alpha_{i}(\mathbf{s}, a_{i}) \ Q^{\phi}_{i}(\mathbf{s}, a_{i}) = \hat{\alpha}_{i}(\mathbf{s}, a_{i}) \ \hat{Q}_{i}(\mathbf{s}, a_{i}) & \ \ a_{i} \neq \arg\max_{a_{i}} \hat{Q}_{i}(\mathbf{s}, a_{i}),
                \end{cases}
            \end{equation*}
            where $\hat{\alpha}_{i}(\mathbf{s}, a_{i}) = \mathbb{E}_{\mathcal{C}_{i} \sim Pr(\mathcal{C}_{i} | \mathcal{N} \backslash \{i\})}[ \hat{\psi}_{i}(\mathbf{s}, a_{i}; \mathbf{a}_{ \scriptscriptstyle\mathcal{C}_{i} }) ]$ and $\hat{\psi}_{i}(\mathbf{s}, a_{i}; \mathbf{a}_{ \scriptscriptstyle\mathcal{C}_{i} }) := \alpha_{i}(\mathbf{s}, a_{i}) \ \sigma(\mathbf{s}, \mathbf{a}_{ \scriptscriptstyle\mathcal{C}_{i} \cup \{i\} })$.
        \end{proposition}
            
        \begin{proof}
            We suppose for any $\mathbf{s} \in \mathcal{S}$ and $\mathbf{a} \in \mathcal{A}$, we have $\Upphi_{i}(\mathbf{s}, a_{i} | \mathcal{C}_{i}) = \sigma(\mathbf{s}, \mathbf{a}_{ \scriptscriptstyle\mathcal{C}_{i} \cup \{i\} }) \ \hat{Q}_{i}(\mathbf{s}, a_{i})$ and $\mathbb{E}_{\mathcal{C}_{i}}[\sigma(\mathbf{s}, \mathbf{a}_{ \scriptscriptstyle\mathcal{C}_{i} \cup \{i\} })] = 1$ when $\mathit{a}_{i} = \arg\max_{a_{i}} Q^{\phi}_{i}(\mathbf{s}, a_{i})$. By the definition of the Markov Shapley Q-value, it is not difficult to obtain
            \begin{align*}
                Q^{\phi}_{i}(\mathbf{s}, a_{i}) &= \mathbb{E}_{\scriptscriptstyle\mathcal{C}_{i}}[ \Upphi_{i}(\mathbf{s}, a_{i} | \mathcal{C}_{i}) ] \\
                &= \mathbb{E}_{\scriptscriptstyle\mathcal{C}_{i}}[ \sigma(\mathbf{s}, \mathbf{a}_{ \scriptscriptstyle\mathcal{C}_{i} \cup \{i\} }) \ \hat{Q}_{i}(\mathbf{s}, a_{i}) ] \\
                &= \mathbb{E}_{\scriptscriptstyle\mathcal{C}_{i}}[ \sigma(\mathbf{s}, \mathbf{a}_{ \scriptscriptstyle\mathcal{C}_{i} \cup \{i\} }) ] \ \hat{Q}_{i} (\mathbf{s}, a_{i}).
            \end{align*}
            Recall that $\delta_{i}(\mathbf{s}, a_{i})$ is defined as follows:
            \begin{equation*}
                \delta_{i}(\mathbf{s}, a_{i}) = \begin{cases} 
                                                      1 & a_{i} = \arg\max_{a_{i}} Q^{\phi}_{i}(\mathbf{s}, a_{i}), \\
                                                      \alpha_{i}(\mathbf{s}, a_{i}) & a_{i} \neq \arg\max_{a_{i}} Q^{\phi}_{i}(\mathbf{s}, a_{i}).
                                                 \end{cases}
            \end{equation*}
            If $a_{i} = \arg\max_{a_{i}} Q^{\phi}_{i}(\mathbf{s}, a_{i})$, it is not difficult to get that $Q_{i}^{\phi}(\mathbf{s}, a_{i}) = \hat{Q}_{i}(\mathbf{s}, a_{i})$.
            
            If $a_{i} \neq \arg\max_{a_{i}} Q^{\phi}_{i}(\mathbf{s}, a_{i})$, we can have the following equation such that
            \begin{align*}
                \alpha_{i}(\mathbf{s}, a_{i}) \ Q^{\phi}_{i}(\mathbf{s}, a_{i}) &= \alpha_{i}(\mathbf{s}, a_{i}) \ \mathbb{E}_{\mathcal{C}_{i}}[ \sigma(\mathbf{s}, \mathbf{a}_{\scriptscriptstyle \mathcal{C}_{i} \cup \{i\}}) \ \hat{Q}_{i}(\mathbf{s}, a_{i}) ] \\
                &= \mathbb{E}_{\mathcal{C}_{i}}[ \alpha_{i}(\mathbf{s}, a_{i}) \ \sigma(\mathbf{s}, \mathbf{a}_{\scriptscriptstyle \mathcal{C}_{i} \cup \{i\}}) ] \ \hat{Q}_{i}(\mathbf{s}, a_{i}) \\
                &:= \mathbb{E}_{\mathcal{C}_{i}}[ \hat{\psi}_{i}(\mathbf{s}, a_{i}; \mathbf{a}_{\scriptscriptstyle \mathcal{C}_{i}}) ] \ \hat{Q}_{i}(\mathbf{s}, a_{i}),
            \end{align*}
            where $\alpha_{i}(\mathbf{s}, a_{i}) \ \sigma(\mathbf{s}, \mathbf{a}_{ \scriptscriptstyle\mathcal{C}_{i} \cup \{i\} })$ is defined as $\hat{\psi}_{i}(\mathbf{s}, a_{i}; \mathbf{a}_{ \scriptscriptstyle\mathcal{C}_{i} })$. Since under this situation $\hat{Q}_{i}(\mathbf{s}, a_{i})$ is always a scaled $Q^{\phi}_{i}(\mathbf{s}, a_{i})$ with the scale of $1/K$, the decisions are consistent to the original decisions.
        \end{proof}
        
        \subsubsection{Implementation of $\hat{\alpha}_{i}(\mathbf{s}, a_{i})$}
        \label{subsubsec:implementation_of_alpha_appendix}
        
            As introduced in the main part of paper, when $a_{i} \neq \arg\max_{a_{i}} \hat{Q}_{i}(\mathbf{s}, a_{i})$, $\hat{\alpha}_{i}(\mathbf{s}, a_{i})$ is implemented as follows:
            \begin{equation*}
                \hat{\alpha}_{i}(\mathbf{s}, a_{i}) = \frac{1}{M}\sum_{k = 1}^{M} \mathlarger{F}_{\mathbf{s}} \Big( \hat{Q}_{\mathcal{C}_{i}^{k}}(\tau_{\mathcal{C}_{i}^{k}}, \mathbf{a}_{\mathcal{C}_{i}^{k}}), \ \hat{Q}_{i}(\tau_{i}, a_{i}) \Big) + 1,
            \label{eq:delta_deep_representation}
            \end{equation*}
            where
            \begin{equation*}
                \hat{Q}_{\mathcal{C}_{i}^{k}}(\tau_{\mathcal{C}_{i}^{k}}, \mathbf{a}_{\mathcal{C}_{i}^{k}}) = \frac{1}{|\mathcal{C}_{i}^{k}|}\sum_{j \in \mathcal{C}_{i}^{k}} \hat{Q}_{j}(\tau_{j}, a_{j})
            \end{equation*}
            and $\mathcal{C}_{i}^{k} \sim \mathit{Pr}(\mathcal{C}_{i} | \mathcal{N} \backslash \{i\})$ that follows the distribution w.r.t. the occurrence frequency of $\mathcal{C}_{i}$; and $\mathlarger{F}_{\mathbf{s}}(\cdot, \cdot)$ is a monotonic function with an absolute activation function on the output whose weights are generated from hypernetworks w.r.t. the global state, similar to the architecture of QMIX \citep{RashidSWFFW18}. Since $\mathlarger{F}_{\mathbf{s}}(\cdot, \cdot) \geq 0$ always holds, it is not difficult to obtain that $\hat{\alpha}_{i}(\mathbf{s}, a_{i}) \geq 1$ always holds. As Eq.\ref{eq:shapley_q_approximate} shows, it is not difficult to get that $\alpha_{i}(\mathbf{s}, a_{i}) = K^{-1} \ \hat{\alpha}_{i}(\mathbf{s}, a_{i})$. Since $K \in (0, 1)$, we get that $\alpha_{i}(\mathbf{s}, a_{i}) > 1$.
            
            As introduced in the main part of paper, the following equation is satisfied such that
            \begin{equation*}
                \delta_{i}(\mathbf{s}, a_{i}) = \frac{1}{|\mathcal{N}| \ w_{i}(\mathbf{s}, a_{i})}.
            \end{equation*}
            For all $\mathbf{s} \in \mathcal{S}$ and $a_{i} \neq \arg\max_{a_{i}} \hat{Q}_{i}(\mathbf{s}, a_{i})$, $\delta_{i}(\mathbf{s}, a_{i}) = \alpha_{i}(\mathbf{s}, a_{i}) > 1$. So, we can derive that
            \begin{equation*}
                \begin{split}
                    &\quad \ \ w_{i}(\mathbf{s}, a_{i}) = \frac{1}{|\mathcal{N}| \ \alpha_{i}(\mathbf{s}, a_{i})} \\
                    &\Rightarrow \max_{a_{i}} w_{i}(\mathbf{s}, a_{i}) = \max_{a_{i}} \frac{1}{|\mathcal{N}| \ \alpha_{i}(\mathbf{s}, a_{i})} = \frac{1}{|\mathcal{N}| \ \min_{a_{i}} \alpha_{i}(\mathbf{s}, a_{i})} < \frac{1}{|\mathcal{N}|} \\
                    &\Rightarrow 0 < \sum_{i \in \mathcal{N}} \max_{a_{i}} w_{i}(\mathbf{s}, a_{i}) < 1.
                \end{split}
            \end{equation*}
            
            For all $\mathbf{s} \in \mathcal{S}$ and $a_{i} = \arg\max_{a_{i}} \hat{Q}_{i}(\mathbf{s}, a_{i})$, $\delta_{i}(\mathbf{s}, a_{i}) = \hat{\delta}_{i}(\mathbf{s}, a_{i}) = 1$. So, we can derive that
            \begin{equation*}
                \begin{split}
                    &\quad \ \ w_{i}(\mathbf{s}, a_{i}) = \frac{1}{|\mathcal{N}|} \\
                    &\Rightarrow \sum_{i \in \mathcal{N}} \max_{a_{i}} w_{i}(\mathbf{s}, a_{i}) = 1.
                \end{split}
            \end{equation*}
            
            Therefore, we can directly obtain that for all $\mathbf{s} \in \mathcal{S}$ and $\mathbf{a} \in \mathcal{A}$, $$0 < \max_{\mathbf{s}} \Big\{ \sum_{i \in \mathcal{N}} \max_{a_{i}} w_{i}(\mathbf{s}, a_{i}) \Big\} \leq 1.$$
            
            Since $\gamma \in (0, 1)$, we can get that $\frac{1}{\gamma} > 1$. As a result, we show that for all $\mathbf{s} \in \mathcal{S}$ and $\mathbf{a} \in \mathcal{A}$, $$0 < \max_{\mathbf{s}} \Big\{ \sum_{i \in \mathcal{N}} \max_{a_{i}} w_{i}(\mathbf{s}, a_{i}) \Big\} < \frac{1}{\gamma}.$$
            
            We get that our implementation of $\hat{\alpha}_{i}(\mathbf{s}, a_{i})$ satisfies the condition in Theorem \ref{thm:shapley_q_optimal}.

\section{Potential Negative Societal Impacts}
\label{sec:potential_negative_societal_impacts}

    Although this paper studies a fundamental theory of multi-agent reinforcement learning, if the proposed algorithm is applied to real-world applications in the future, there may still exist some potential negative societal impacts. First, since the theory does not consider robustness, it is possible that the proposed algorithm would be attacked or vulnerable to some extreme scenarios like most of machine learning models and algorithms. Fortunately, our theory is orthogonal to the robustness and it is possible to consider robustness as an extension in the future work. Another potential negative societal impacts could come from the implementation of models, e.g., policy and critic. Since these are implemented by neural networks that are known as black boxes, the reliability could be a problem. Nevertheless, this is irrelevant to the main purpose of this paper and can be improved by other related research tracks in the future.

\end{document}